\documentclass{lmcs} 


\usepackage[bookmarksdepth=subsubsection]{hyperref}

\usepackage{symbols}
\usepackage{tikz}
\usepackage{graphicx}
\usepackage[all]{foreign}

\usepackage{multirow}
\usepackage{rotating}

\usepackage{amsthm}
\usepackage{thmtools}

\usepackage{silence}
\newcommand{\cev}[1]{\reflectbox{\ensuremath{\vec{\reflectbox{\ensuremath{#1}}}}}}

\renewcommand*{\path}[1]{\url{#1}}

\theoremstyle{plain} 
\theoremstyle{definition}
\newtheorem{definition}{Definition}[section]
\theoremstyle{plain}
\newtheorem{proposition}{Proposition}[section]
\theoremstyle{plain}
\newtheorem{theorem}{Theorem}[section]
\theoremstyle{plain}
\newtheorem{lemma}{Lemma}[section]
\theoremstyle{plain}
\newtheorem{corollary}{Corollary}[section]
\theoremstyle{plain}
\newtheorem{remark}{Remark}[section]

\def\cf{{\em cf.}}

\begin{document}

%
\title{Synthesis of timeline-based planning strategies avoiding determinization}
\titlecomment{{\lsuper*}This paper is an extended and revised version of~\cite{gandalf24}.}
\thanks{Angelo Montanari acknowledges the support from the Interconnected Nord-Est Innovation Ecosystem (iNEST), which received funding from the European Union Next-GenerationEU (PIANO NAZIONALE DI RIPRESA E RESILIENZA (PNRR) -- MISSIONE 4 COMPONENTE 2, INVESTIMENTO 1.5 -- D.D. 1058 23/06/2022, ECS00000043) and from the MUR PNRR project FAIR - Future AI Research (PE00000013) also funded by the European Union Next-GenerationEU.
 Dario Della Monica acknowledges the partial support from the M4C2 I1.3
 ``SEcurity and RIghts In the CyberSpace -- SERICS'' (PE00000014 - CUP
 D33C22001300002), under the National Recovery and Resilience Plan (NRRP) funded
 by the European Union-NextGenerationEU and from the European Union projects ``MAPSART -- Massive Parallelism Serving Automated Reasoning Tasks'' (CUP C63C22000770006) and ``SecCo-OC -- Secure Containers Open Call'' (CUP D33C22001300002).
 Dario Della Monica and Angelo Montanari acknowledge the partial support from
 the INdAM-GNCS project ``Neurosymbolic semantic contract for rigorous reachability and certificates in cyber-physical systems (CPS)'' (CUP E53C25002010001).
 This manuscript reflects only the authors’ views and opinions, neither the European Union nor the European Commission can be considered responsible for them.
%
%
}	

\specialissue{Instructions for \LaTeX{} formatting of LMCS articles}\specialissue{Selected Papers of the 15th International Symposium on Games, Automata, Logics, and Formal Verification (GandALF 2024)}


\author[D.~Della Monica]{Dario Della Monica\lmcsorcid{0000-0001-9743-665X}}[a]
\author[A.~Montanari]{Angelo Montanari}[a]
\author[P.~Sala]{Pietro Sala\lmcsorcid{0000-0002-2612-1519}}[b]

\address{University of Udine, Italy}	
\email{dario.dellamonica@uniud.it, angelo.montanari@uniud.it}  

\address{University of Verona, Italy}	
\email{pietro.sala@univr.it}  

%
%




\begin{abstract}
%
%
  Qualitative timeline-based planning models domains as sets of independent, but
  interacting, components whose behaviors over time, the timelines, are governed
  by sets of qualitative temporal constraints (ordering relations), called
  synchronization rules.
  Its plan-existence problem has been shown to be \PSPACE-complete; in
  particular, \PSPACE-membership has been proved via reduction to the
  nonemptiness problem for nondeterministic finite automata.
  However, nondeterministic automata cannot be directly used to synthesize
  planning strategies as a costly determinization step is needed.
  In this paper, we identify a fragment of qualitative timeline-based planning
  whose plan-existence problem can be directly mapped into the nonemptiness
  problem of deterministic finite automata, which can then be used to
  synthesize strategies.
  In addition, we identify a maximal subset of Allen's relations that fits into
  such a deterministic fragment.
\end{abstract}

\maketitle


\section{Introduction}
\label{sec:intro}

\emph{Timeline-based} planning is an approach that originally emerged and
developed in the context of planning and scheduling of \emph{space}
operations~\cite{Muscettola94}. In contrast to common action-based formalisms,
such as PDDL~\cite{FoxL03}, timeline-based languages do not distinguish
among actions, states, and goals. Rather, the domain is modeled as a set of
independent, but interacting, components, the timelines, whose behavior over
time is governed by a set of temporal constraints.
It is worth pointing out that timeline-based planning was born with an
application-oriented flavor, with various successful stories, and only
relatively recently foundational work about its expressiveness and
complexity has been done.
The present paper aims at bringing back theory to practice by searching for expressive enough and computationally well-behaved fragments.

Timeline-based planning has been successfully employed by planning systems
developed at NASA~\cite{ChienRKSEMESFBST00,ChienSTCRCDLMFTHDSUBBGGDBDI04} and at
ESA~\cite{FratiniCORD11} for both short- to long-term mission planning and
on-board autonomy. More recently, timeline-based planning systems such as
PLATINUm~\cite{UmbricoCMO17} are being employed in collaborative robotics
applications~\cite{UmbricoCO23}. All these applications share a deep reliance on
\emph{temporal reasoning} and the need for a tight integration of planning with
\emph{execution}, both features of the timeline-based framework.
The latter feature is usually achieved by the use of \emph{flexible timelines},
which represent a set of possible executions of the system that differ in the
precise timing of the events, hence handling the intrinsic \emph{temporal
uncertainty} of the environment. A formal account of timeline-based planning
with uncertainty has been provided by \cite{CialdeaMayerOU16}. A lot of theoretical research followed, including
\emph{complexity}~\cite{BozzelliMMP18b,BozzelliMMP18a,GiganteMCO17} and
\emph{expressiveness}~\cite{DellaMonicaGMS18,GiganteMMO16} analyses, based on
the  formalization given in \cite{CialdeaMayerOU16}, which is the one we use here as well.

To extend reactivity and adaptability of timeline-based systems beyond
temporal uncertainty, the framework of (two-player) \emph{timeline-based games} has been
recently proposed. In timeline-based games, the system player tries to build a
set of timelines satisfying the constraints independently from the choices of
the environment player. This framework allows one to handle general
nondeterministic environments in the timeline-based setting. However, its
expressive power comes at the cost of increasing the complexity of the problem.
While the plan-existence problem for timeline-based planning is
EXPTIME-complete~\cite{GiganteMCO17}, deciding the existence of strategies for
timeline-based games is 2EXPTIME-complete~\cite{GiganteMOCR20}, and a controller
synthesis algorithm exists that runs in doubly exponential
time~\cite{AcamporaGGMP2022}.

Such a high complexity motivates the search for simpler fragments
%
%
that can nevertheless be useful in practical scenarios. One of them is the
\emph{qualitative} fragment, where temporal constraints only concern the
relative order between pairs of events and not their
%
%
distance. The qualitative fragment already proved itself to be easier for the
plan-existence problem, being PSPACE-complete~\cite{DellaMonicaGTM20}, and this
makes it a natural candidate for the search of a good fragment for the strategy-existence problem.
Unfortunately, a \emph{deterministic} arena is crucial to synthesize a non-clairvoyant strategy
in \emph{reactive synthesis} problems~(see, for instance,
\cite{PnueliR89}), and determinizing the nondeterministic automaton of exponential size built for the qualitative case in \cite{DellaMonicaGTM20} would cause an exponential blowup, thus resulting in a procedure of doubly-exponential
complexity. The strategy synthesis procedure from a deterministic arena is a well-known, standard construction in reactive synthesis~\cite{PnueliR89}. Once the deterministic arena is obtained, one solves a reachability/safety game on it via classical fixpoint algorithms.

In this paper, we identify a class of qualitative timeline-based planning
problems, called the \emph{eager fragment}, that admits a
characterization of solution plans in terms of \emph{deterministic} finite automata
(DFA) of size at most exponential, thereby enabling an exponential solution to
the strategy-existence problem (the synthesis problem for short).\footnote{As a matter of fact, the eager fragment was introduced in~\cite{gandalf24}. 
However, there was a mistake in its definition, because
it allows disjunctions (inside rules), while they must be disallowed for the result to hold.}
Intuitively, the eager fragment aims to remove
\emph{nondeterminism} from qualitative timeline-based planning problems.
We identify two sources of nondeterminism: disjunctions, which require automata
to \emph{guess} the disjuncts witnessing their satisfaction, and some particular
conjunctions of constraints, imposing partial ordering on events, which require
automata to \emph{guess} the exact (linear) order of the events in advance.

On the one hand, we prove that restricting to the eager fragment is a
\emph{sufficient} condition to directly synthesize a DFA of singly-exponential size, thus usable as an arena to play the game in an
asymptotically optimal way.
On the other hand, we give evidence of the expressive power of such a fragment
by showing that eager
synchronization rules can systematically encode all major control flow patterns of Business Process Model and Notation (BPMN), including
sequential
execution, parallel branching, exclusive choice, and iterative loops (\autoref{sec:eager-expressiveness}); moreover, 
 we demonstrate
that the eager fragment is expressive enough to capture a large subset of
Allen's relations~\cite{Allen83} (\autoref{sec:allen}).

It is not known, instead, whether both restrictions imposed by the eager fragment to remove the two source of nondeterminism (disjunctions and particular
conjunctions imposing a partial ordering on events) are \emph{necessary}, or one suffices.
Towards an answer to this question, 
 we identify a class of (non-eager) qualitative
timeline-based planning problems for which a characterization of solution plans
in terms of DFA of size at most exponential does
not exist (\autoref{sec:expAutomataLowerBound}).
Such a class relaxes one of the two restrictions (by allowing for disjunctions).
However, to show that both restrictions are necessary, a similar result should
be proved for a class of problems obtained by relaxing the other restriction.

The rest of the paper is organized as follows.
\autoref{sec:preliminaries} recalls some background knowledge on timeline-based
planning.
\autoref{sec:fragment} defines the eager fragment, that directly maps into a DFA
of singly exponential size.
\autoref{sec:wordsToPlans} gives a word encoding of timelines, and vice versa.
\autoref{sec:expAutomataLowerBound} proves that it is not possible to encode the
solution plans for (non-eager) qualitative timeline-based planning problems
using deterministic finite automata of exponential size.
\autoref{sec:dfaForPlans} builds an automaton to recognize plans, and
\autoref{sec:dfaForSolutions} shows how to construct an automaton that accepts
solution plans.
\autoref{sec:eager-expressiveness} demonstrates the practical relevance of
the eager fragment through a comprehensive case study that systematically
translates BPMN diagrams into eager
timeline-based qualitative planning problems, showing that the fragment can
capture all major control flow patterns.
\autoref{sec:allen} identifies the maximal subset of Allen's relations which is
captured by the eager fragment.
Finally, \autoref{sec:conclusion} summarizes the main contributions of the work
and discusses possible future developments.

This paper is a considerably revised and extended version of~\cite{gandalf24}.
In particular, missing proofs are added, \autoref{sec:expAutomataLowerBound}
and~\autoref{sec:eager-expressiveness} are completely new,
\autoref{sec:fragment} has been enriched with more examples that make it clear
the intuition behind eager rules, and \autoref{sec:allen} has been considerably
extended with a summarizing table and with several explanatory examples.



\section{Background}
\label{sec:preliminaries}

In this section, we recall the basic notions of timeline-based
planning and of its qualitative variant.
As usual, $\mathbb N$ is the set of natural numbers, and  \natpos stands for $\mathbb N \setminus \{ 0 \}$.

\subsection{Timeline-based planning}

The key notion is that of \emph{state variable}.
%

\begin{definition}[State variable]
  \label{def:statevar}
  A \emph{state variable} is a tuple $x=(V_x,T_x,D_x)$, where:
  \begin{itemize}
  \item $V_x$ is the \emph{finite domain} of the variable;
  \item $T_x:V_x\to2^{V_x}$ is the \emph{value transition function}, which maps
        each value $v\in V_x$ to the set of values that can (immediately) follow it;
  \item $D_x:V_x\to\natpos\times(\natpos\cup\set{+\infty})$ is a function that maps
        each $v\in V_x$ to the pair $(d^{x=v}_{min},d^{x=v}_{max})$ of
        minimum and maximum durations allowed for intervals where $x=v$.
  \end{itemize}
\end{definition}

A \emph{timeline} describes how the value of a state variable $x$ evolves over time. It consists of a finite sequence of \emph{tokens}, each denoting a value $v$ and (the duration
of) a time interval $d$.

\begin{definition}[Token and timeline]
  \label{def:timeline}
  A \emph{token} for $x$ is a tuple $\tau=(x,v,d)$, where $x$ is a state
  variable, $v\in V_x$ is the value held by the variable, and $d\in\natpos$ is
  the \emph{duration} of the token, with $D_x(v) = (d^{x=v}_{min},d^{x=v}_{max})$ 
  and $d^{x=v}_{min} \le d \le d^{x=v}_{max}$.
  A \emph{timeline} for a state variable $x$ is a finite sequence
  $\timeline=\seq{\tau_1,\ldots,\tau_k}$ of tokens for $x$, for some
  $k\in\N$, such that, for any $1\le i < k$, if $\tau_i = (x,v_i,d_i)$,
  then $v_{i+1} \in T_x(v_i)$.
\end{definition}

For every timeline $\timeline = \seq{\tau_1,\ldots,\tau_k}$ and
token $\tau_i=(x,v_i,d_i)$ in $\timeline$, we define the functions
$\starttime(\timeline,i) = \sum_{j=1}^{i-1} d_j$ and
$\endtime(\timeline,i) = \starttime(\timeline,i) + d_i$. 
We call the \emph{horizon} of $\timeline$ the end time of the last token in
$\timeline$, that is, $\endtime(\timeline, k)$.  We write
$\starttime(\tau_i)$ and $\endtime(\tau_i)$ to indicate
$\starttime(\timeline,i)$ and $\endtime(\timeline,i)$, respectively, when
there is no ambiguity.

The overall behavior of state variables is subject to
a set of temporal constraints known as \emph{synchronization rules} (or
simply \emph{rules}). We start by defining their basic building blocks.
Let $\toknames$ be a finite set of \emph{token names}. \emph{Atoms} are
formulas of the following form:
\begin{equation*}
\begin{split}
  \mathit{atom} &\coloneqq \mathit{term} \before_{l,u} \mathit{term} \mid \mathit{term} \before*_{l,u} \mathit{term} \\
  \mathit{term} &\coloneqq \tokstart(a) \| \tokend(a) \| t
\end{split}
\end{equation*}
%
%
where $a\in\toknames$, $l,t\in\N$, and $u\in\N\cup\set{+\infty}$. 
Terms $\tokstart(a)$ and $\tokend(a)$ respectively denote the start and the end
of the token associated with the token name $a$.
As an example, atom $\tokstart(a)\before_{l,u}\tokend(b)$ (resp.,
$\tokstart(a)\before*_{l,u}\tokend(b)$) relates tokens $a$ and $b$ by stating
that the end of $b$ cannot precede (resp., must succeed) the beginning of $a$,
and the distance between these two endpoints must be at least $l$ and at most
$u$.
An atom $\mathit{term}_1\before_{l,u}\mathit{term}_2$, with $l=0$, $u=+\infty$,
and $\mathit{term}_1,\mathit{term}_2 \notin \N$, is \emph{qualitative} (the
subscript is usually omitted in this case).
We sometimes use the abbreviation $\mathit{term}_1 = \mathit{term}_2$ for
$\mathit{term}_1 \before \mathit{term}_2 \wedge \mathit{term}_2 \before
\mathit{term}_1$.
%


An \emph{existential statement} $\E$ is a constraint of the form:
\[ \exists a_1[x_1=v_1]a_2[x_2=v_2] \ldots a_n[x_n=v_n]. \ \rulebody \] where
$x_1, \ldots, x_n$ are state variables, $v_1, \ldots, v_n$ are values, with $v_i\in V_{x_i}$, for $i = 1, \ldots, n$, $a_1,\ldots,a_n$ are token names from \toknames, and
$\rulebody$ is a finite conjunction of atoms, called a \emph{clause},
involving only tokens $a_1,\ldots,a_n$, plus, possibly, the \emph{trigger token}
(usually denoted by $a_0$) of the \emph{synchronization
  rule} in which the existential statement is embedded, as shown
below.%
\footnote{Without loss of generality, we assume that
  \begin{enumerate*}[label={\it (\roman*)}]
  \item if a token $a$ appears in the
    quantification prefix $\exists a_1[x_1=v_1]a_2[x_2=v_2] \ldots a_n[x_n=v_n]$
    of \E, then at least one among $\tokstart(a)$ and $\tokend(a)$ occurs in one
    of its atoms, and
  \item trivial atoms, i.e., atoms of the form $\tokstart(a) \before
    \tokstart(a)$, $\tokend(a) \before \tokend(a)$, $\tokstart(a) \before
    \tokend(a)$, or $\tokstart(a) \before* \tokend(a)$, never occur in
    existential statements, even though they clearly hold by the definition
    of token.
  \end{enumerate*}
}
Intuitively, an existential statement asks for the existence of tokens
$a_1,a_2,\dots,a_n$ whose state variables take the corresponding values
$v_1,v_2,\dots,v_n$ and are such that their start and end times satisfy the
atoms in $\rulebody$.

\emph{Synchronization rules} have one of
the following forms:
\begin{align*}
    a_0[x_0=v_0]\implies \E_1\lor\E_2\lor\ldots\lor\E_k \\ 
    \top \implies \E_1\lor\E_2\lor\ldots\lor\E_k
\end{align*}
where $a_0 \in \toknames$, $x_0$ is a state variable, $v_0 \in V_{x_0}$,
and $\E_i$ is an existential statement, for each $1\le i\le k$. 
In the former case, $a_0[x_0=v_0]$ is called \emph{trigger} and $a_0$ is the \emph{trigger token}, and the rule
is considered \emph{satisfied} if \emph{for all} the tokens for $x_0$ with
value $v_0$, at least one of the existential
statements is satisfied.
In the latter case, the rule is said to be \emph{triggerless}, and it
states the truth of the body without any precondition.%
\footnote{%
  Without loss of generality, if $a_0$ is the trigger token of a non-triggerless
  rule, then both $\tokstart(a_0)$ and $\tokend(a_0)$ occur in the 
  existential statements of the rule.
  In particular, as an exception, we allow trivial atoms over trigger tokens (e.g.,
  $\tokstart(a_0) \before* \tokend(a_0)$).
}
We refer the reader to~\cite{CialdeaMayerOU16} for a formal account of the
semantics of the rules.
A synchronization rule is \emph{disjunction-free} if it contains only one
existential statement, that is, $k=1$.

A \emph{timeline-based planning problem} consists of a set of state variables
and a set of rules that represent the problem domain and the goal.

\begin{definition}[Timeline-based planning problem]
  \label{def:planningProblem}
  A \emph{timeline-based planning problem} is defined as a pair $P = (\SV, S)$,
  where $\SV$ is a set of state variables and $S$ is a set of synchronization
  rules involving state variables in $\SV$.
  Moreover, $P$ is \emph{disjunction-free} if so are all of its synchronization
  rules.
\end{definition}

A \emph{solution plan} for a given timeline-based planning problem is a set of
timelines, one for each state variable, that satisfies all the synchronization
rules.

\begin{definition}[Plan and solution plan]
  \label{def:plan}
  A \emph{plan} over a set of state variables $\SV$ is a finite set of timelines
  with the same horizon, one for each state variable $x\in\SV$.
  A \emph{solution plan} (or, simply, \emph{solution}) for a timeline-based
  planning problem $P=(\SV,S)$ is a plan over \SV such that all the rules in $S$
  are satisfied.
\end{definition}

The problem of determining whether a solution plan exists for a given
timeline-based planning problem is \EXPSPACE-complete \cite{GiganteMCO17}.
%

\begin{definition}[Qualitative timeline-based planning]
\label{def:qual:tbp}
    A timeline-based planning problem $P=(\SV,S)$ is said to be \emph{qualitative} if the following conditions hold: 
    \begin{enumerate}
    \item $D_x(v)=(1,+\infty)$, for all state variables $x\in\SV$ and $v\in
      V_x$.
    \item synchronization rules in $S$ involve \emph{qualitative} atoms only.
    \end{enumerate}
\end{definition}


%
%
%

In the rest of the paper, we focus on qualitative 
timeline-based planning.
Its complexity is shown to be PSPACE-complete in~\cite{DellaMonicaGTM20}, where
a reduction to the nonemptiness problem for non-deterministic finite automata
(NFA) is given.

\section{A well-behaved fragment}
\label{sec:fragment}


In this section, we introduce a meaningful fragment of qualitative
timeline-based planning for which it is possible to construct a DFA of singly exponential size.
The fragment is characterized by conditions on the admissible patterns of
synchronization rules (\emph{eager rules}).
The distinctive feature of eager rules is that their satisfaction with respect
to a given plan can be checked using an eager/greedy strategy, that is, when a
relevant event (start/end of a token involved in some atom) occurs, we are
guaranteed that the starting/ending point of such a token is useful for rule
satisfaction. With non-eager rules, instead, a relevant event may occur that is
not useful for rule satisfaction: some analogous event in the future will be.


As a preliminary step, we define a sort of reflexive
and transitive closure of a clause.
By slightly abusing the notation, we identify a clause \rulebody
%
%
with the finite set of atoms occurring in it. Let $t$, $t_1$, $t_2$, and $t_3$ be
terms of the form $\tokstart(a)$ or $\tokend(a)$, with $a \in \toknames$.
The \emph{closure} of $\rulebody$, denoted by $\hat \rulebody$, is defined as
the smallest set of atoms including $\rulebody$ such that:
\begin{enumerate*}[label={\it (\roman*)}]
\item if term $t$ occurs in $\rulebody$, then atom $t \before t$ belongs to
  $\hat \rulebody$,
\item if both terms $\tokstart(a)$ and $\tokend(a)$ occur in $\rulebody$ for
  some token name $a$, then atom $\tokstart(a) \before* \tokend(a)$ belongs to
  $\hat \rulebody$,
\item if atom $t_1 \before* t_2$ belongs to $\hat
  \rulebody$, then atom $t_1 \before t_2$ belongs to
  $\hat \rulebody$ as well,
\item if atoms $t_1 \before t_2$ and $t_2
  \before t_3$ belong to $\hat \rulebody$, then atom $t_1\before t_3$
  belongs to $\hat \rulebody$ as well,
\item if atoms $t_1 \before* t_2$ and $t_2
  \before t_3$ belong to $\hat \rulebody$, then atom $t_1 \before* t_3$ 
  belongs to $\hat \rulebody$ as well,
\item if atoms $t_1 \before t_2$  and $t_2
  \before* t_3$ belong to $\hat \rulebody$, then atom $t_1 \before* t_3$ belongs to $\hat \rulebody$ as well.
\end{enumerate*}%
\footnote{%
  Without loss of generality, we assume that $\hat\rulebody$ is consistent, \ie
  it admits at least a solution. This check can be done in polynomial time,
  since it is an instance of linear programming.
}
We write $t_1 \equiv t_2 \in \hat \rulebody$ as an abbreviation to mean
that both atoms $t_1 \before t_2$ and $t_2 \before t_1$ belong to $\hat
\rulebody$, and $t_1 \equiv t_2 \not\in \hat \rulebody$ for its negation.
Clearly, the set (of terms occurring in) $\hat \rulebody$ and the relation
$\before$ define a preorder.

Intuitively, the closure $\hat{\rulebody}$ captures all the temporal ordering relationships that are logically implied by the explicit constraints in $\rulebody$. It includes not only the directly stated ordering constraints, but also their transitive consequences and the implicit relationships that follow from the nature of tokens (such as the fact that every token has a start time before its end time). As an example, if $\rulebody$ contains atoms $\tokstart(a) \before \tokstart(b)$ and $\tokstart(b) \before \tokend(c)$, then $\hat{\rulebody}$ will also contain the atom $\tokstart(a) \before \tokend(c)$, which follows by transitivity, even though this relationship was not explicitly stated. Similarly, if both $\tokstart(a)$ and $\tokend(a)$ appear in $\rulebody$, then $\hat{\rulebody}$ automatically includes $\tokstart(a) \before* \tokend(a)$, reflecting the fundamental property that tokens have positive duration. 
On the other hand, the absence of a relationship between token endpoints
indicate that either the opposite relationship holds 
or the relationship between such endpoints is not implied by the existing constraints.
For instance, if $\tokend(a) \before \tokstart(b)$ does not belong to the
closure, then it can be that $\tokstart(b) \before* \tokend(a)$ is in the
closure (and thus $b$ must begin before the end of $a$) or, simply, that the end
of $a$ is not constrained to occur before the beginning of $b$.
This closure operation ensures that we have a complete picture of all temporal
relationships that must hold in any plan satisfying the rule, which is essential
for determining whether a rule can be checked eagerly or not.
Notice that trivial atoms over non-trigger tokens are allowed in the closure of
a clause (while they are disallowed in the clause itself).

Let us  now define the fundamental notion of \emph{eager rule}.

\begin{definition}[Ambiguous token, eager rule, eager planning problem]
  \label{def:eager:rule}
  Let $\Rule$ be a synchronization rule, $\rulebody$ one of the clauses of
  \Rule, and $a$ a token name occurring in \rulebody.
  Then,
  \begin{enumerate}[label={\it (\Alph*)},ref={\it \Alph*}]
  \item \label{item:left-ambiguous}
    $a$ is \emph{left-ambiguous} if all of the following are verified:
    \begin{enumerate}[label={\it (\ref{item:left-ambiguous}\arabic*)},ref={\it
        \ref{item:left-ambiguous}\arabic*}]
    \item \label{item:left-ambiguous-i}
      if $a_0$ is the trigger token of \Rule, then $\tokstart(a) \equiv
      \tokstart(a_0) \not\in \hat\rulebody$ and $\tokstart(a) \equiv
      \tokend(a_0) \not\in \hat\rulebody$, and
    \item \label{item:left-ambiguous-ii}
      there is a term $t \not\in \{ \tokstart(a),
      \tokend(a) \}$
%
%
      such that $\tokstart(a) \before t \in \hat \rulebody$ and $\tokend(a)
      \before t \not\in \hat \rulebody$;
    \end{enumerate}
  \item \label{item:right-ambiguous}
    $a$ is \emph{right-ambiguous} if
%
%
    there is a term $t \not\in \{ \tokstart(a),
    \tokend(a) \}$ such that $t \before \tokend(a) \in \hat \rulebody$ and $t
    \before \tokstart(a) \not\in \hat \rulebody$;

  \item \label{item:ambiguous} $a$ is \emph{ambiguous} if it is not the trigger
    token of \Rule, and it is both left- and right-ambiguous.
  \end{enumerate}

  \noindent Rule \Rule is \emph{unambiguous} if none of the token names
  occurring in it is ambiguous.

  \noindent Rule \Rule is \emph{eager} if it is unambiguous and
  disjunction-free.

  \noindent
  Finally, a qualitative timeline-based planning problem $P = (\SV, S)$ is
  \emph{eager} if $S$ only contains eager rules.
\end{definition}
Clearly, deciding whether a rule is disjunction-free is linear in its size.
Moreover, deciding whether a disjunction-free rule is also eager amounts to
computing the closure of its unique clause (which requires time at most
polynomial in the size of the rule), and then verifying that none of its the
token names is ambiguous (which ultimately reduces to a series of membership
tests, and thus requires time linear in the size of the closure).
Therefore, deciding whether a rule is eager can be done in time at most
polynomial in the size of the rule.

From now on, we focus on eager qualitative timeline-based planning
problems. For the sake of brevity, we sometimes refer to them simply as
\emph{planning problems}.
Intuitively, eager rules remove the \emph{nondeterminism} (a sort of ambiguity)
that comes from two factors: disjunctions, which require to \emph{guess} a
disjunct, and some patterns of conjunctions of atoms, which require to
\emph{guess} the right occurrences of events relevant for satisfaction.
We claim that restricting to eager rules (thus, removing these two sources of
nondeterminism) suffices to obtain a singly exponential DFA, whose construction
will be illustrated in the next sections. We give here a short intuitive account
of the rationale behind the conditions of \autoref{def:eager:rule}.

\medskip

Consider the following synchronization rule:
\begin{equation*}
  a_0[x_0=v_0] \implies \exists a_1[x_1=v_1]. \ (\tokstart(a_0) = \tokstart(a_1)
                \wedge \tokend(a_0) \before \tokend(a_1)).
\end{equation*}
According to \autoref{def:eager:rule}, it is an eager rule because $a_0$ is not
ambiguous (trivially, as $a_0$ is the trigger token), and neither is $a_1$,
since $\tokstart(a_1) \equiv \tokstart(a_0) \in \hat\rulebody$, which implies
that $a_1$ is not left-ambiguous (see \autoref{item:left-ambiguous-i} in
\autoref{def:eager:rule}).
In particular, having $\tokstart(a_0) = \tokstart(a_1)$ is crucial for any DFA
$\autom$ recognizing solution plans, because, when $\autom$ reads the event
$\tokstart(a_0)$, it can \emph{eagerly} and \emph{deterministically} go to a
state representing the fact that both $\tokstart(a_0)$ and $\tokstart(a_1)$ have
happened.  Moreover, if later it reads the event $\tokend(a_1)$, but it has not
read $\tokend(a_0)$ yet, then it transitions to a rejecting state, that is, a
state from which it cannot accept any plan; if, instead, it reads the event
$\tokend(a_1)$ only after reading $\tokend(a_0)$ (or at the same time), then it
transitions to an accepting state.

\medskip

As a more subtle example, consider the following rule, obtained from the previous one by replacing $=$ by $\before$ in the clause:
\begin{equation}\label{eq:example-eager}
  a_0[x_0=v_0] \implies \exists a_1[x_1=v_1]. \ (\tokstart(a_0) \before
                \tokstart(a_1) \wedge \tokend(a_0) \before \tokend(a_1)).
\end{equation}
This rule is eager as well, because, once again, trigger token $a_0$ is
trivially not ambiguous and token name $a_1$ is not left-ambiguous
(\autoref{item:left-ambiguous-ii} in \autoref{def:eager:rule} is falsified).
It is not immediate, though, to see that this rule can always be handled
\emph{eagerly} by a DFA.
Indeed, consider the plan (partially) depicted in \autoref{fig:eager-example}.
The picture shows a token of value $v_0$ for timeline $x_0$ that starts at time
1 and ends a time 4.
Such a token matches the trigger token $a_0$, thus triggering the rule.
\autoref{fig:eager-example} also depicts two tokens matching value ($v_1$) and
timeline ($x_1$) associated with token name $a_1$: the first one starting at
time 2 and the second one starting at time 5.
The plan satisfies the rule from Equation~\eqref{eq:example-eager} above, thanks
to the token starting at time 5.
Now, consider a DFA that tries to \emph{eagerly} verify the rule.
Initially, the DFA idles until the beginning of a token that matches the trigger
token $a_0$.
When, at time 1, the token for $x_0$ starts, the DFA transitions into a state
where $\tokstart(a_0)$ has happened, and the automaton waits for the beginning
of a token for $x_1$ of value $v_1$ or the end of the token for $x_0$ (the two
events can also happen at the same time).
At time 2, a token starts that matches token name $a_1$; even if this token is
not useful to certify the fulfillment of the rule (since it ends before the
ending of the token for $x_0$), a DFA that behaves eagerly will match this event
with $\tokstart(a_1)$, and will thus transition into a state where it waits for
the token for $x_0$ to end before (or at the same time of) the token for $x_1$.
Since the token for $x_1$ ends at time 3, strictly before the token for $x_0$,
it seems that the eager choice of the DFA will lead to a rejection.
Fortunately, since $\tokstart(a_1)$ is not constraint to happen before some
other event (there is no atom $\tokstart(a_1) \before t$ in $\hat \rulebody$,
with $t \in \{ \tokstart(a_0), \tokend(a_0) \}$ -- see
\autoref{item:left-ambiguous-ii} in \autoref{def:eager:rule}), the DFA can be
instructed to
ignore the ending, at time 3, of the token for $x_1$, thus, as a matter of fact,
\emph{adjusting} the first match of $\tokstart(a_1)$: it will be (tacitly)
re-matched with some future token start.
At time 4, the token for $x_0$ ends, and the DFA transitions into a state where
it waits for the end of a token for $x_1$ of value $v_1$, which will happen at
time 6.
Tacitly, term $\tokstart(a_1)$ is necessarily re-matched with the beginning of
the token ending at time 6.

\begin{figure}[t]
  \centering

  \begin{tikzpicture}[x=1cm,y=1cm]

    \foreach \t in {0,1,2,3,4,5,6}
    {
      \draw[gray!35] (\t,1.7) -- (\t,-1.15);
      \node[gray, font=\footnotesize] at (\t,1.9) {\t};
    }

    \node[left] at (-0.5,1) {$x_0$};

    \draw[thick] (1,1) -- (4,1);
    \draw[thick] (1,0.9) -- (1,1.1);
    \draw[thick] (4,0.9) -- (4,1.1);
    \node[above] at (2.5,1) {$v_0$};

    \node[left] at (-0.5,0) {$x_1$};

    \draw[thick] (2,0) -- (3,0);
    \draw[thick] (2,-0.1) -- (2,0.1);
    \draw[thick] (3,-0.1) -- (3,0.1);
    \node[above] at (2.5,0) {$v_1$};

    \draw[thick] (3,0) -- (5,0);
    \draw[thick] (3,-0.1) -- (3,0.1);
    \draw[thick] (5,-0.1) -- (5,0.1);
    \node[above] at (4,0) {$v_2$};

    \draw[thick] (5,0) -- (6,0);
    \draw[thick] (5,-0.1) -- (5,0.1);
    \draw[thick] (6,-0.1) -- (6,0.1);
    \node[above] at (5.5,0) {$v_1$};

    \node[left] at (-0.5,-1) {$x_2$};

    \draw[thick] (2,-1) -- (3,-1);
    \draw[thick] (2,-.9) -- (2,-1.1);
    \draw[thick] (3,-.9) -- (3,-1.1);
    \node[above] at (2.5,-1) {$v_1$};

    \draw[thick] (3,-1) -- (5,-1);
    \draw[thick] (3,-0.9) -- (3,-1.1);
    \draw[thick] (5,-0.9) -- (5,-1.1);
    \node[above] at (4,-1) {$v_1$};

  \end{tikzpicture}

  \caption{A (partial) plan satisfying the rule from
    Equation~\eqref{eq:example-eager}, which is eager, and the one from
    Equation~\eqref{eq:example-2-eager}, which is not.
    The token starting at time 1 triggers both rules.
    In both cases, the satisfaction of the rule is witnessed by the second token
    of value $v_1$ in the relevant timeline (not the first one).
    However, a DFA that behaves eagerly can be instructed to identify the
    witnessing token for the eager rule, but not for the non-eager one.
    Note that the token for $x_1$ starting at time 3 has value $v_2 \neq v_1$;
    thus, it does not match token name $a_1$, and cannot be used to fulfill the
    rule from Equation~\eqref{eq:example-eager}.
  }
  \label{fig:eager-example}
\end{figure}


\medskip

Let us provide now an example of a non-eager rule, thus not amenable to being
checked in an eager/greedy fashion:
\begin{equation}\label{eq:example-2-eager}
  a_0[x_0=v_0] \implies \exists a_3[x_2=v_1]. \ (\tokstart(a_0) \before
                \tokstart(a_3) \wedge \tokstart(a_3) \before
                \tokend(a_0) \wedge \tokend(a_0) \before \tokend(a_3)).
\end{equation}
This rule is \emph{not} eager, because token name $a_3$ is ambiguous (as it is
not the trigger token, and it is
both left-and right-ambiguous), and thus cannot be handled by a DFA that acts
eagerly like the one described above.
As a matter of fact, such a DFA would not accept the plan (partially) depicted
in~\autoref{fig:eager-example}, even though it satisfies the rule.
In particular, due to $a_3$ being left-ambiguous (there is $t \in \{
\tokstart(a_0), \tokend(a_0) \}$ such that $\tokstart(a_3) \before t \in \hat
\rulebody$ and $\tokend(a_3) \before t \not\in \hat \rulebody$ -- see
\autoref{item:left-ambiguous-ii} in \autoref{def:eager:rule}), it is not
possible
to
instruct a DFA, as before, to somehow match $\tokstart(a_3)$ with the beginning
of the token for $x_2$ starting at time 2 and $\tokend(a_3)$ with the end of a
different token for $x_2$, the one ending at time 5.
In other words, the eager choice of the DFA to match $\tokstart(a_3)$ with the
beginning of the token for $x_2$ starting at time 2 is final and cannot be
adjusted, and the DFA is not able to deterministically establish that the token
witnessing the rule is the second one for $x_2$ (rather than the first one): it
is something that must be \emph{guessed nondeterministically}.

In what follows, we give a reduction from the plan-existence problem for the
eager fragment of qualitative timeline-based planning to the nonemptiness
problem of DFAs of \emph{singly exponential} size with respect to the original
problem.
The approach is inspired by those in~\cite{DellaMonicaGTM20,DellaMonicaGMS18}
for non-eager timeline-based planning, where an NFA of exponential size is built
for any given timeline-based planning problem.
It is important to note that there is a bijective correspondence between the set
of solutions of a planning problem and the language accepted by the automaton
built from that problem.
However, the reductions presented in~\cite{DellaMonicaGTM20,DellaMonicaGMS18}
produce nondeterministic automata, which cannot be used as arenas to solve
timeline-based games without a preliminary determinization step that would cause a
second exponential blowup.

In the following, we first show how to encode timelines and plans as finite words, and vice versa
(\autoref{sec:wordsToPlans}).
Using such an encoding, we show that it is not possible to characterize the
solution plans for non-eager qualitative timeline-based planning problems using DFA of exponential size
(\autoref{sec:expAutomataLowerBound}).
Then, given an eager qualitative timeline-based planning problem $P = (\SV, S)$,
we show how to build a DFA whose language encodes the set of solution plans for
$P$.  The DFA consists of the intersection of two DFAs: one checks that the
input word correctly encodes a (candidate) plan over \SV that fulfills the
constraint on the alternation of token values expressed by functions $T_x$, for
$x \in \SV$ (\autoref{sec:dfaForPlans}); the other one verifies that the
encoded plan is indeed a solution plan for $P$, i.e., synchronization rules in
$S$ are fulfilled (\autoref{sec:dfaForSolutions}).




\section{From plans to finite words and vice versa}
\label{sec:wordsToPlans}

In this section, as a first step towards the construction of the DFA
corresponding to an eager qualitative timeline-based planning problem, we show
how to encode timelines and plans as \emph{words} that can be recognized by an
automaton, and \viceversa.

Let $P = (\SV, S)$ be an eager qualitative timeline-based planning
problem, and let $V = \cup_{x\in \SV}V_x$.
We define the \emph{initial alphabet} $\Sigma_\SV^I$ as $(\{ - \} \times
V)^\SV$, that is the set of functions from $\SV$ to $(\{ - \} \times V)$.%
\footnote{%
  The symbol $\{ - \}$ is a technicality that allows us to consider pairs
  instead of single values in $V$, to be uniform with symbols of the non-initial
  alphabet.
}
Similarly, we define the \emph{non-initial alphabet} $\Sigma_\SV^N$ as $((V
\times V) \cup \{ \circlearrowleft \} )^\SV$, where the pairs $(v,v') \in
V \times V$ are supposed to encode the value $v$ of the token that has just
ended and the value $v'$ of the token that has just started, and
$\circlearrowleft$ represents the fact that the value for the state
variable has not changed.
The \emph{input alphabet} (or, simply, \emph{alphabet}) associated with
$\SV$, denoted by $\Sigma_\SV$, is the union $\Sigma_\SV^I \cup
\Sigma_\SV^N$.
%
%
%
%
Observe that the size of the alphabet $\Sigma_\SV$ is at most exponential
in the size of $\SV$, precisely $\abs{\Sigma_\SV} = \abs{\Sigma_\SV^I}
+ \abs{\Sigma_\SV^N} = \abs{V}^{\abs{\SV}} + (\abs{V}^2 + 1)^{\abs{\SV}}$.
%



We now show how to encode the basic structure%
\footnote{%
  With ``basic structure'' we refer to the fact that, in this section, we
  neither take into account the transition functions $T_x$ of state
  variables nor their domains $V_x$ (\cf~\autoref{def:statevar}), which will be
  dealt with in \autoref{sec:dfaForPlans}.
  Recall that functions $D_x$ are irrelevant as we only consider qualitative
  planning problems (\cf~\autoref{def:qual:tbp}).
}
underlying each plan over \SV as a word in $\Sigma_\SV^I \cdot
(\Sigma_\SV^{N})^* \cup \{ \varepsilon \}$, where $\varepsilon$ is the
empty word (corresponding to the empty plan), $(\Sigma_\SV^{N})^*$ is the
Kleene's closure of $\Sigma_\SV^{N}$, and $\cdot$ denotes the concatenation
operation.
Intuitively, let $\sigma$ be the symbol at position $i$ of a word $\lambda \in
\Sigma_\SV^I \cdot (\Sigma_\SV^{N})^* \cup \{ \varepsilon \}$.
Then, if $\sigma(x) = (v, v')$ for some $x \in \SV$, then at time $i$ a new
token begins in the timeline for $x$ with value $v'$; instead, if $\sigma(x) =
\circlearrowleft$, then no change happens at time $i$ in the timeline for $x$,
meaning that no token ends at that time point in the timeline for $x$.
The value $v$ of the token ending at time $i$ will come in handy later in the
construction of the automata.

We remark that not all words in $\Sigma_\SV^I \cdot (\Sigma_\SV^{N})^* \cup
\{ \varepsilon \}$ correspond to plans over \SV: for a word to correctly
encode a plan, the information carried by the word about the value of
a starting token and the one associated to the end of the same token must
coincide.
Formally, given a word $\lambda=\seq{\sigma_0,\ldots,\sigma_{|\lambda|-1}} \in
\Sigma_\SV^I \cdot (\Sigma_\SV^{N})^* \cup \{ \varepsilon \}$ and a state
variable $x \in \SV$, let $\mathit{changes}(x) =
(i^x_0,i^x_1,\ldots,i^x_{k^x-1})$, for some $k^x \in \N$, be the increasing
sequence of positions where $x$ changes, \ie $i \in \mathit{changes}(x)$ if and
only if $\sigma_i(x) \neq \circlearrowleft$, for all $i \in \{ 0, \ldots,
|\lambda|-1 \}$.  We denote by $\cev{\sigma}^x_{i}$ and
$\vec{\sigma}^x_{i}$ the first and the second component of $\sigma_i(x)$,
respectively, for all $x \in \SV$ and $i \in \mathit{changes}(x)$.
We omit superscripts $^x$ when there is no risk of ambiguity.
\begin{definition}[Words weakly-encoding plans]
  \label{def:weak:encode}
  Let $\lambda \in \Sigma_\SV^I \cdot (\Sigma_\SV^{N})^* \cup \{ \varepsilon \}$
  and let $\mathit{changes}(x) = (i_0,i_1,\ldots,i_{k-1})$.  We say that
  \emph{$\lambda$ weakly-encodes a plan over \SV} if $\vec{\sigma}^x_{i_{h-1}}
  = \cev{\sigma}^x_{i_{h}}$ for all $x \in \SV$ and $h \in \{ 1, \ldots, k-1 \}$.
  If this is the case, then the \emph{plan induced by $\lambda$} is the set $\{
  \timeline_x \mid x \in \SV \}$, where
  $\timeline_x=\seq{(x,\vec{\sigma}^x_{i_0}, i_1-i_0),
    (x,\vec{\sigma}^x_{i_1},i_2-i_1), \ldots, (x,\vec{\sigma}^x_{i_{k-1}},
    i_k-i_{k-1})}$ and $i_k = |\lambda|$, for all $x \in \SV$.
\end{definition}
\noindent
Intuitively, if a word weakly-encodes a plan, then it captures the
dynamics of a state variable, but it ignores its domain and transition
function, which will be taken care of in the next sections.
%
%
%
A converse correspondence from plans to words can be
defined accordingly.

Before concluding the section, we introduce another couple of notions that will
come in handy later.
We denote by \eventsOnesigma the set of events (beginning/ending of a token)
occurring at a given time, encoded in the alphabet symbol $\sigma$.
Formally, \eventsOnesigma is the smallest set such that:
\begin{itemize}
\item if $\sigma(x) = (v,v') $ for some $x$, then
    $\{ \mathit{end}(x,v),\mathit{start}(x,v') \} \subseteq
    \eventsOnesigma$, and
\item if $\sigma(x) = (-,v')$ for some $x$, then $\mathit{start}(x,v') \in
  \eventsOnesigma$.
\end{itemize}
Finally, we say that $\sigma$ \emph{triggers} a rule \Rule if
$\mathit{start}(x_0,v_0) \in \eventsOnesigma$ and $a_0[x_0=v_0]$ is the trigger
of \Rule.


\section{Instances for which no exponential automata exist}
\label{sec:expAutomataLowerBound}


We let $\firstnatn = \{ 0, 1, \ldots, n \}$ and $\firstnatNOzeron = \{
1, \ldots, n \}$, for all $n \in \mathbb N$.
For a planning problem $P$, we denote by $|P|$ its size, that is, the length of
its representation\footnote{
  The length of the representation of a qualitative planning problem $(\SV, S)$
  is the sum of the length of the representation of \SV, which is at most
  polynomial in the number of variables in it and the size of the domains $V_x$,
  for $x \in \SV$, and the one of the representation of $S$, which is at most
  polynomial in the number of synchronization rules and the size of the longest
  rule in it.
}, and by $\mathcal L(P)$ the language of words
encoding solution plans for $P$.

In this section, we show that a characterization of the solution plans for
(non-eager) qualitative timeline-based planning problems using deterministic
finite automata of exponential size does not exist.
More precisely, we define a schema $\{P_n\}_{n \in \natpos}$ of (non-eager)
qualitative timeline-based planning problems and we show that, for $n$ large
enough, the smallest automata accepting $\mathcal L(P_n)$ has size more than
exponential in the size of $P_n$ (\autoref{thm:exp-lower-bound} below).

For all $n \in \natpos$, we let $P_n = (\SV_n, S_n)$, where

\begin{itemize}
\item $\SV_n = \{ x_i \}_{i \in \firstnatn}$, with
  \begin{itemize}
  \item $x_i = (V_{x_i}, T_{x_i}, D_{x_i})$ for all $i \in \firstnatn$,
  \item $V_{x_0} = \{ v_0, \bar{v}_0 \}$,
  \item $V_{x_i} = \{ v_i, v'_i, \bar{v}_i \}$ for all $i \in \firstnatNOzeron$,
  \item $T_{x_i}(v) = V_{x_i}$ for all $i \in \firstnatn$ and $v \in V_{x_i}$,
  \item $D_{x_i}(v) = (1,+\infty)$ for all $i \in \firstnatn$ and $v \in
    V_{x_i}$.
%
  \end{itemize}

\item $S_n$ is a singleton containing the following synchronization rule

  \smallskip

\noindent{\centering

$
\begin{array}{@{\hspace{-6.7pt}}r@{\hspace{2mm}}p{6mm}@{\hspace{0mm}}l}
  a_0[x_0=v_0] & \ensuremath{\implies}
  & \exists a_1[x_1=v_1]a'_1[x_1=v'_1].  \tokstart(a_0) \before
  \tokstart(a_1) \wedge \tokstart(a_1) \before \tokend(a_0) \wedge \tokend(a_0)
  \before \tokstart(a'_1) \\
         & \multicolumn{1}{r}{\lor} & \exists
           a_2[x_2=v_2]a'_2[x_2=v'_2]. \tokstart(a_0) \before
           \tokstart(a_2) \wedge \tokstart(a_2) \before \tokend(a_0) \wedge
           \tokend(a_0) \before \tokstart(a'_2) \\
         & \multicolumn{1}{r}{\lor} & \multicolumn{1}{l}{\ldots} \\
         & \multicolumn{1}{r}{\lor} & \exists
           a_n[x_n=v_n]a'_n[x_n=v'_n]. \tokstart(a_0) \before
           \tokstart(a_n) \wedge \tokstart(a_n) \before \tokend(a_0) \wedge
           \tokend(a_0) \before \tokstart(a'_n) \\
\end{array}
$

}

\end{itemize}

In what follows, we identify, for any given $n \in \natpos$, a set of plans
(equivalently, words encoding plans) over $\SV_n$ that contains more than $2^n$
plans that are pairwise distinguished by some extension, in the sense clarified
later.
The thesis (\cf~\autoref{thm:exp-lower-bound}) then follows from Myhill-Nerode
theorem~\cite{hopcroft2006introduction}.
Note also that the above synchronization rule is not eager, as it is not
disjunction-free (it is unambiguous, though -- \cf~\autoref{def:eager:rule}).
Therefore, while restricting to disjunction-free and unambiguous (thus, eager)
rules is a sufficient condition towards the construction of the desired
deterministic automata, it is unclear whether both restrictions are necessary or
if restricting only to disjunction-free rule -- while still allowing ambiguous
ones -- is already a sufficient condition.

Now, consider plans over $\SV_n$, for $n \in \natpos$, such that:
\begin{itemize}
\item the timeline associated with $x_0$ is a sequence of $v_0$-valued tokens of
  duration two, followed by a $\bar{v}_0$-valued final token of duration one;
\item each other timeline, associated with variable $x_j$ for some $j \in
  \firstnatNOzeron$,
  \begin{itemize}
  \item begins with a $\bar{v}_j$-valued token whose duration is one and
  \item continues with a sequence of tokens of duration two and whose value is
    either $v_j$ or $\bar{v}_j$.
  \end{itemize}
\end{itemize}
Formally, for every $n \in \natpos$, let $W_n$ be the set of words of odd length
encoding plans over $\SV_n$ (that is, every word $\lambda \in W_n$ has the form
$\sigma_0 \sigma_1 \ldots \sigma_h$ for some even positive natural number $h$)
such that:

\begin{itemize}
\item the timeline associated with state variable $x_0$ evolves as follows:
  \begin{itemize}
  \item $\sigma_0(x_0) = (-,v_0)$ (a $v_0$-valued token starts at the
    beginning of the timeline),

  \item $\sigma_{i}(x_0) = \circlearrowleft$, for all odd $i \in \firstnath$ (at
    odd time instants no token starts/ends),

  \item $\sigma_{i}(x_0) = (v_0,v_0)$, for all even $i \in \{ 2, \ldots, h-2 \}$
    (at even time instants, except for $h$, a new $v_0$-valued token starts),

  \item $\sigma_{h}(x_0) = (v_0,\bar{v}_0)$ (the last token has value
    $\bar{v}_0$), and

  \end{itemize}

\item other timelines, associated with state variables $x_j$, for $j \in
  \firstnatNOzeron$, evolve as follows:

  \begin{itemize}
  \item $\sigma_0(x_j) = (-,\bar{v}_j)$ (a $\bar{v}_j$-valued token starts at
    the beginning of the timeline),

  \item $\sigma_{i}(x_j) = \circlearrowleft$, for all even $i \in
    \firstnatNOzeroh$ (at even time instants, except for zero, no token
    starts/ends),

  \item $\sigma_{i}(x_j) = (v_{\mathit{old}},v_{\mathit{new}})$ for some
    $v_{\mathit{new}} \in \{ v_{j}, \bar{v}_{j} \}$ and some $v_{\mathit{old}}
    \in V_{x_j}$, and for all odd $i \in \firstnath$ (at odd time instants a new
    token starts whose value can be either $v_j$ or $\bar{v}_j$).

  \end{itemize}

\end{itemize}

\begin{figure}
  \centering
\begin{tikzpicture}
    \def\n{3}  
    \def\ysep{.6}  
    \def\xstart{0}  
    \def\recth{0.4} 
    \def\rectl{2} 
    \def\unit{.75}  
    \def\last{5}  

    \foreach \i in {0,1,2,3,4,5,6} {
        \node[above] at (\xstart + \i*\unit, 0.5) {$\sigma_{\i}$};
        \draw[thick,lightgray,dotted] (\xstart + \i*\unit, 0.4) -- ++(0, -3.8);
    }
    \node[above] at (\xstart + 8*\unit, 0.5) {$\dots$};
    \node[above] at (\xstart + 11*\unit, 0.5) {$\sigma_h$};
    \draw[thick,lightgray,dotted] (\xstart + 11*\unit, 0.4) -- ++(0, -3.8);

    \draw (\xstart + 13*\unit, .9) -- ++(0, -4.3);

    \node[above] at (\xstart + 14*\unit, 0.5) {$\sigma$};
    \draw[thick,lightgray,dotted] (\xstart + 14*\unit, 0.4) -- ++(0, -3.8);

    \node[left] at (\xstart, 0) {$x_0$};
    \foreach \i in {0,1,2} {
        \draw[thick] (\xstart + \rectl*\i*\unit, -\recth/2) rectangle ++(\rectl*\unit, \recth) node[midway] {$v_0$};
    }
    \node at (\xstart + 8*\unit, -\recth/2) {$\dots$};
    \draw[thick] (\xstart + 9*\unit, -\recth/2) rectangle ++(\rectl*\unit, \recth) node[midway] {$v_0$};
    \draw[thick] (\xstart + 11*\unit, -\recth/2) rectangle ++(\unit, \recth) node[midway] {$\overline v_0$};

    \draw[thick] (\xstart + 14*\unit, -\recth/2) rectangle ++(\rectl*\unit, \recth) node[midway] {$\overline v_0$};

    \foreach \j in {1,...,\n} {
        \node[left] at (\xstart, -\j*\ysep) {$x_\j$};
        \draw[thick] (\xstart, -\j*\ysep-\recth/2) rectangle ++(\unit, \recth) node[midway] {$\overline{v}_\j$};
        \foreach \i in {0,1,2} {
            \draw[thick] (\xstart + \unit + \rectl*\i*\unit, -\j*\ysep-\recth/2) rectangle ++(\rectl*\unit, \recth) node[midway] {$v_\j / \overline{v}_\j$};

        }
        \node at (\xstart + 8*\unit, -\j*\ysep-\recth/2) {$\dots$};
        \draw[thick] (\xstart + 10*\unit, -\j*\ysep-\recth/2) rectangle ++(\rectl*\unit, \recth) node[midway] {$v_\j / \overline{v}_\j$};

        \draw[thick] (\xstart + 14*\unit, -\j*\ysep-\recth/2) rectangle ++(\rectl*\unit, \recth) node[midway] {$v'_\j / \overline{v}_\j$};

    }

    \node[left] at (\xstart, -\last*\ysep+\ysep) {$\dots$};
    \node[above] at (\xstart + 8*\unit, -\last*\ysep+\ysep-0.25) {$\dots$};

    \node[left] at (\xstart, -\last*\ysep) {$x_{n}$};
    \draw[thick] (\xstart, -\last*\ysep-\recth/2) rectangle ++(\unit, \recth) node[midway] {$\overline{v}_{n}$};
    \foreach \i in {0,1,2} {
      \draw[thick] (\xstart + \unit + \rectl*\i*\unit, -\last*\ysep-\recth/2) rectangle ++(\rectl*\unit, \recth) node[midway] {$v_{n} / \overline{v}_{n}$};
    }
    \node at (\xstart + 8*\unit, -\last*\ysep-\recth/2) {$\dots$};
    \draw[thick] (\xstart + 10*\unit, -\last*\ysep-\recth/2) rectangle ++(\rectl*\unit, \recth) node[midway] {$v_{n} / \overline{v}_{n}$};

    \draw[thick] (\xstart + 14*\unit, -\last*\ysep-\recth/2) rectangle ++(\rectl*\unit, \recth) node[midway] {$v'_{n} / \overline{v}_{n}$};


\end{tikzpicture}
\caption{Graphical representation of sets $W_n$, with $n \in \mathbb N$.
  $W_n$ includes words $\lambda=\sigma_0\sigma_1 \ldots \sigma_h$, for all even
  positive natural numbers $h$, encoding plans over $\SV_n$ and such that:
  the timeline associated with variable $x_0$ features $h/2$ $v_0$-valued
  tokens and a final $\overline v_0$-valued token;
  timelines associated with variables other than $x_0$ feature an initial
  $\overline v_j$-valued token, followed by $h/2$ tokens labeled with either
  $v_j$ or $\overline v_j$, for $j \in \firstnatNOzeron$.
}
\label{fig:planCounterexample}
\end{figure}

A graphical account of words in $W_n$, for $n \in \mathbb N$, is given in
\autoref{fig:planCounterexample}, where $h$ is any even positive natural number.
It is easy to see that none of the words in $W_n$ represents a solution plan for
$P_n$.
However, each of them can be completed into a solution plan by adding only one
more token in each timeline.
Formally, let $\Lambda$ be the set of all functions $\sigma : \SV \rightarrow
\bigcup_{i \in \firstnatn} V_{x_i}$ such that $\sigma(x_0) = \bar v_0$ and
$\sigma(x_i) \in \{ v'_i, \bar v_i \}$, for all $i \in \firstnatNOzeron$ (see
\autoref{fig:planCounterexample}, right-hand side).
We have that
\begin{equation}
  \parbox{.9\textwidth}{\emph{
      \\[-2mm]
      \noindent
      the concatenation of a plan $\sigma_0 \sigma_1 \ldots \sigma_h$
      in $W_n$ with a function $\sigma \in \Lambda$ is a solution plan for $P_n$
      if and only if for all odd natural numbers $i \in \firstnath$ there is $j
      \in \firstnatNOzeron$ such that $\sigma_i(x_j) = v_j$ and $\sigma(x_j) =
      v'_j$.
  }}
\label{text:completition}
\end{equation}
Two plans $\sigma_0 \sigma_1 \ldots \sigma_h$ and $\sigma'_0 \sigma'_1 \ldots
\sigma'_{h'}$ in $W_n$ are \emph{distinguished} by $\Lambda$ if and only if
there is $\sigma \in \Lambda$ such that exactly one between $\sigma_0 \sigma_1
\ldots \sigma_h \sigma$ and $\sigma'_0 \sigma'_1 \ldots \sigma'_{h'} \sigma$ is
a solution plan for $P_n$.
In what follows, we show that, for $n$ large enough, $W_n$ contains more than
$2^n$ plans that are pairwise distinguished by $\Lambda$.
The thesis then follows from Myhill-Nerode
theorem~\cite{hopcroft2006introduction}.

To this end, notice that timeline associated with $x_0$ is the same for all
plans in $W_n$ of a given length, and thus a plan in $W_n$ is univocally
identified by its length and the sequences of $v_j/\bar{v}_j$ tokens in all the
other timelines.
Therefore, there is a bijection $\biject$ between plans $\sigma_0 \sigma_1
\ldots \sigma_h$ in $W_n$ of length $h+1$ and sequences $\mu_1 \mu_2 \ldots
\mu_{\frac{h}{2}}$, of length $\frac{h}{2}$, of sets of natural numbers, defined
as follows: $\biject(\sigma_0 \sigma_1 \ldots \sigma_h) = \mu_1 \mu_2 \ldots
\mu_{\frac{h}{2}}$, where $\mu_i \subseteq \firstnatNOzeron$ for all $i \in
\firstnatNOzerohdivtwo$ and $j \in \mu_i$ if and only if in the timeline
associated with $x_j$ the $i$th token after the initial $\bar{v}_j$-valued token
has value $v_j$.
(Observe that it is essential for the existence of bijection $\biject$ to have
state variables with disjoint domains, i.e., $V_{x_i} \cap V_{x_j} = \emptyset$
for all $i \neq j$.)
Analogously, all functions in $\Lambda$ have the same value at $x_0$; therefore,
there is another bijection, which we call $\biject$ abusing the notation,
between $\Lambda$ and the powerset of \firstnatNOzeron (namely, $2^{\{ 1, 2,
  \ldots, n \}}$), which maps every $\sigma \in \Lambda$ to set $\biject(\sigma)
= \mu \subseteq \firstnatNOzeron$, with $j \in \mu$ if and only if $\sigma(x_j)
= v'_j$ (and thus $j \not\in \mu$ if and only if $\sigma(x_j) = \bar v_j$).
Now, given a plan $\sigma_0 \sigma_1 \ldots \sigma_h$ in $W_n$ and $\sigma \in
\Lambda$ such that $\biject(\sigma_0 \sigma_1 \ldots \sigma_h) = \mu_1 \mu_2
\ldots \mu_{\frac{h}{2}}$ and $\biject(\sigma)=\mu$, we can rephrase
statement~\eqref{text:completition} above as:
\begin{equation}
  \parbox{.9\textwidth}{\emph{
      \\[-2mm]
      \noindent
      the concatenation of $\mu_1 \mu_2 \ldots \mu_{\frac{h}{2}}$ with $\mu$
      represents a solution plan for $P_n$ if and only if $\mu_i \cap \mu \neq
      \emptyset$ for all $i \in \firstnatNOzerohdivtwo$.
  }}
\label{text:completitionbis}
\end{equation}
Clearly, the order of the elements in sequence $\mu_1 \mu_2 \ldots
\mu_{\frac{h}{2}}$, as well as the presence of repetitions therein, is
irrelevant.
To make this formal, for a plan $\vec{\sigma} = \sigma_0 \sigma_1 \ldots
\sigma_h$ in $W_n$, let $\mathit{set}(\vec{\sigma})$ be the support of the
sequence of sets $\biject(\vec{\sigma})$, i.e., $\mathit{set}(\vec{\sigma}) = \{
\mu \mid \mu \in \biject(\vec{\sigma}) \}$ is the set of elements occurring in
$\biject(\vec{\sigma})$, unsorted and devoid of repetitions.
Then, if two plans $\vec{\sigma},\vec{\sigma'} \in W_n$ are such that
$\mathit{set}(\vec{\sigma}) = \mathit{set}(\vec{\sigma'})$, then they are not
distinguished by $\Lambda$.
The converse implication does not hold in general, but it does if we restrict
ourselves to plans $\vec{\sigma}$ where at every odd time instant there are
exactly $\floor{\frac{n}{2}}$ different timelines where a $v_j$-valued token
starts (and, consequently, exactly $\ceil{\frac{n}{2}}$ different timelines
where a $\bar v_j$-valued token starts instead -- see
\autoref{fig:planCounterexample}).
Formally, we have the following result.
\begin{proposition}
  Let $\vec{\sigma},\vec{\sigma'} \in W_n$ be such that
  \begin{itemize}
  \item $|\mu| = \floor{\frac{n}{2}}$ for all $\mu \in \biject(\vec{\sigma})
    \cup \biject(\vec{\sigma'})$, and
  \item $\mathit{set}(\vec{\sigma}) \neq \mathit{set}(\vec{\sigma'})$.
  \end{itemize}
  Then, there is $\sigma \in \Lambda$ that distinguishes them.
\end{proposition}
\begin{proof}
  Due to $\mathit{set}(\vec{\sigma}) \neq \mathit{set}(\vec{\sigma'})$, we have
  that:
  \begin{itemize}
  \item there is $\mu \in \biject(\vec{\sigma})$ such that $\mu \neq \mu'$ for
    all $\mu' \in \biject(\vec{\sigma'})$, or
  \item there is $\mu' \in \biject(\vec{\sigma'})$ such that $\mu \neq \mu'$ for
    all $\mu \in \biject(\vec{\sigma})$.
  \end{itemize}
  Assume, without loss of generality, that the former holds (the other case is
  dealt with analogously).
  Since $\mu$ has the same cardinality as every set $\mu' \in
  \biject(\vec{\sigma'})$, there must be an element $e_{\mu'} \in \mu' \setminus
  \mu$ for each $\mu' \in \biject(\vec{\sigma'})$.
  Let us collect these elements in a set $\mu'' = \{ e_{\mu'} \mid \mu' \in
  \biject(\vec{\sigma'}) \}$.
  Clearly, $\mu'' \cap \mu = \emptyset$; thus, according
  to~\eqref{text:completitionbis}, the concatenation of $\biject(\vec{\sigma})$
  with $\mu''$ does not represent a solution plan for $P_n$.
  In addition, for all $\mu' \in
  \biject(\vec{\sigma'})$ it holds that $\mu'' \cap \mu' \neq \emptyset$; thus,
  according to~\eqref{text:completitionbis}, the concatenation of
  $\biject(\vec{\sigma'})$ with $\mu''$ represents a solution plan for $P_n$.
%
%
  The thesis follows, since $\mu'' = \biject(\sigma)$, for some $\sigma \in
  \Lambda$.
\end{proof}

Using the above proposition, showing that, for $n$ large enough, there are more
than $2^n$ plans in $W_n$ that are pairwise distinguished by $\Lambda$ amounts
to showing that for any given set of cardinality $n$ there are more than $2^n$
distinct sets of subsets of cardinality $\floor{\frac{n}{2}}$, as formally
stated in the next proposition (note that
$\mathit{set}(\vec{\sigma})$ is a set of subsets of \firstnatNOzeron, for all
$\vec{\sigma} \in W_n$).

\begin{proposition}
  For every $n \in \mathbb N$, with $n \geq 4$, the cardinality of the powerset
  of the set $\{ S \subseteq \firstnatNOzeron \mid |S| = \floor{\frac{n}{2}} \}$
  is greater than $2^n$.
\end{proposition}
\begin{proof}
  The cardinality of $\{ S \subseteq \firstnatNOzeron \mid |S| =
  \floor{\frac{n}{2}} \}$ is $\binom{n}{\floor{\frac{n}{2}}}$.
  Thus, the cardinality of its powerset is $2^{\binom{n}{\floor{\frac{n}{2}}}}$,
  which is strictly greater than $2^n$ for all $n \geq 4$.
\end{proof}

\begin{theorem} \label{thm:exp-lower-bound}
  There exists a set $\mathcal P$ of (non-eager) qualitative timeline-based
  planning problems of unbounded size such that for all problems $P \in \mathcal
  P$ it holds that the smallest automaton recognizing $\mathcal L(P)$ has size
  greater than $2^{p(|P|)}$, for any polynomial $p$.
\end{theorem}



\section{DFA accepting plans}
\label{sec:dfaForPlans}

Given an eager qualitative timeline-based planning problem $P = (\SV, S)$, we
show how to build a DFA $\mathcal T_{\SV}$, of size at most exponential in the
size of \SV (and thus the one of $P$), accepting words that
\emph{encode plans over \SV}, i.e., words that weakly-encode plans
(\cf~\autoref{def:weak:encode}) and, additionally, are compatible with the
domain functions $V_x$ and comply with the constraints on the alternation of
token values expressed by functions $T_x$, for $x \in \SV$.
In the next section, we show how to obtain a DFA, once again of size at most
exponential in the size of $P$, that accepts exactly the \emph{solution
  plans} for $P$.

For every planning problem $P = (\SV, S)$, the DFA $\mathcal T_{\SV}$ is the
tuple $\langle Q_{\SV}, \Sigma_\SV, \delta_{\SV}, q^0_{\SV}, \allowbreak F_{\SV}
\rangle$, whose components are defined as follows.
\begin{itemize}

\item $Q_{\SV}$ is the finite set of \emph{states}.
  Intuitively, a state of $\mathcal{T}_{\SV}$ keeps track of the token values of
  the timelines at the current and the previous step of the run.
  Therefore, a state is a function mapping each state variable $x$ into a pair
  $(v,v')$, where $v'$ (resp., $v$) denotes the token value of timeline $x$ at
  the current (resp., previous) step.
  To formally define $Q_{\SV}$, we exploit the alphabet
  $\Sigma_\SV$ (\cf~\autoref{sec:wordsToPlans}): states are the
  alphabet symbols except for those functions $\sigma \in
  \Sigma_\SV$ that assign $\circlearrowleft$ to at least one state variable $x
  \in \SV$.
  For technical reasons, we also need a fresh \emph{initial state} $q^0_{\SV}$
  and a fresh \emph{rejecting sink state} \sinkT.
  Formally, $Q_{\SV} = \left(\Sigma_\SV \setminus \bar{Q}_{\SV} \right) \cup \{
  q^0_{\SV}, \sinkT \}$, where $\bar{Q}_{\SV} = \{ \sigma \in \Sigma_\SV \mid
  \sigma(x) = \circlearrowleft \text{ for some } x \in \SV \}$.
  Clearly, the size of $Q_{\SV}$ is at most as the size of $\Sigma_\SV$, which
  is in turn at most exponential in the size of $P$.

%
%

\item $\Sigma_\SV$ is the \emph{input alphabet}, defined as in
  \autoref{sec:wordsToPlans}.
\item $\delta_{\SV} : Q_{\SV} \times \Sigma_\SV \rightarrow Q_{\SV}$ is
  the \emph{transition function}.
  Towards a definition of $\delta_{\SV}$, we say that an alphabet symbol $\sigma
  \in \Sigma_\SV$ is \emph{compatible} with a state $\sigma_1 \in Q_{\SV}$ (we
  use for states the same symbols as for the alphabet, i.e., $\sigma, \sigma_1,
  \sigma_2, \ldots$, to stress the fact that states are closely related to
  alphabet symbols) if one of the following holds:
  \begin{enumerate*}[label={\it (\roman*)}]
  \item $\sigma_1 = q^0_{\SV}$ is the initial state and $\sigma \in
    \Sigma_\SV^I$ is an initial symbol such that for each $x \in \SV$ it holds
    that $\sigma(x) = (-,v)$ with $v \in V_x$;
  \item $\sigma_1 \in \Sigma_\SV \setminus \bar{Q}_{\SV}$ and
    $\sigma \in \Sigma_\SV^N$ is a non-initial alphabet symbol such that for
    each $x \in \SV$ either $\sigma(x) = \circlearrowleft$ or
    $\sigma_1(x) = (v,v')$ and
    $\sigma(x) = (v',
    v'')$ with $v'' \in T_x(v') \cap V_x$.
  \end{enumerate*}
%
%
%
%
%
  Now, $\delta_{\SV} : Q_{\SV} \times \Sigma_\SV \to Q_{\SV}$ is defined as
  follows.
  For all $\sigma_1 \in Q_{\SV}$ and $\sigma \in \Sigma_\SV$, if $\sigma$ is not
  compatible with $\sigma_1$ or $\sigma_1$ is the sink state (i.e., $\sigma_1 =
  \sinkT$), then $\delta(\sigma_1, \sigma) = \sinkT$; otherwise
  \begin{itemize}
  \item if $\sigma_1$ is the initial state (i.e., $\sigma_1 = q^0_{\SV}$), then
    $\delta(\sigma_1, \sigma) = \sigma$; in other words, in this case the
    automaton transitions to the state represented by the letter that is being
    read;
  \item if $\sigma_1 \in \Sigma_\SV \setminus \bar{Q}_{\SV}$, then
    $\delta(\sigma_1, \sigma) = \sigma_2$, where $\sigma_2(x) = \sigma_1(x)$ if
    $\sigma(x) = \circlearrowleft$, and $\sigma_2(x) = \sigma(x)$ otherwise, for
    all $x \in \SV$; intuitively, the automaton transitions into a state that
    updates the information about tokens whose value is changed, according to
    the letter that is being read.
  \end{itemize}
  We point out that, in both cases, the automaton transitions to the next state in a deterministic fashion.
\item $F_{\SV} = Q_{\SV} \setminus \{ \sinkT \}$ is the set of \emph{final
  states}.
\end{itemize}

Correctness of the DFA $\mathcal T_{\SV}$ is proved by the next
lemma.

\begin{lemma}
  \label{lem:automatonPlans}
  Let $P = (\SV, S)$ be an eager qualitative timeline-based planning problem.
  Then, words accepted by $\mathcal T_{\SV}$ are exactly those encoding plans
  over \SV.
  Moreover the size of $\mathcal T_{\SV}$ is at most exponential in the size of
  $P$.
\end{lemma}


\begin{proof}
  Let $W$ denote the set of words that weakly-encode plans over $\SV$ (in the
  sense of Definition~\ref{def:weak:encode}) and additionally satisfy the domain
  and transition constraints, that is, for each word
  $\lambda=\seq{\sigma_0,\ldots,\sigma_{|\lambda|-1}} \in W$, each $x \in \SV$,
  and each position $i \in \mathit{changes}(x)$, we have $\vec{\sigma}_i^x \in
  V_x$ and, if $i$ is not the first position in $\mathit{changes}(x)$, then
  $\vec{\sigma}_i^x \in T_x(\cev{\sigma}_i^x)$. We establish that
  $\mathcal{L}(\mathcal{T}_{\SV}) = W$, which precisely corresponds to the set
  of words encoding plans over $\SV$.
  We prove the equivalence by showing both containments
  $\mathcal{L}(\mathcal{T}_{\SV}) \subseteq W$ and $W \subseteq
  \mathcal{L}(\mathcal{T}_{\SV})$ using contradiction arguments based on the
  minimum index where violations occur.

  To show $\mathcal{L}(\mathcal{T}_{\SV}) \subseteq W$, let us suppose, towards
  a contradiction, that there exists a word $\lambda = \seq{\sigma_0, \sigma_1,
    \ldots, \sigma_{n-1}} \in \mathcal{L}(\mathcal{T}_{\SV})$ such that $\lambda
  \notin W$.
  Since $\mathcal{T}_{\SV}$ accepts $\lambda$, there is an accepting run
  $q^0_{\SV} \xrightarrow{\sigma_0} q_1 \xrightarrow{\sigma_1} \cdots
  \xrightarrow{\sigma_{n-1}} q_n$ with $q_n \in F_{\SV}$. Let $i$ be the minimum
  index such that the constraints defining $W$ are violated at position $i$. We
  consider the possible violations.
  If $i = 0$ and either $\sigma_0 \notin \Sigma_\SV^I$ or $\sigma_0(x) = (-,v)$
  with $v \notin V_x$ for some $x \in \SV$, then symbol $\sigma_0$ is not
  compatible with the initial state $q^0_{\SV}$, and thus the transition
  $\delta_{\SV}(q^0_{\SV}, \sigma_0)$ leads to the sink state $\sinkT$,
  contradicting the fact that $\lambda$ is accepted.
  If $i > 0$, then $q_i$ is not the initial state, and in order for symbol
  $\sigma_i$ to be compatible with $q_i$, it must be $\sigma_i \in \Sigma_\SV^N$
  ($\sigma_i$ is a non-initial alphabet symbol).
  If the violation occurs because the weak encoding condition fails (i.e.,
  $\vec{\sigma}^x_{i_{h-1}} \neq \cev{\sigma}^x_{i_h}$ for some consecutive
  changes), or because domain constraints are violated ($\vec{\sigma}_i^x \notin
  V_x$), or because transition constraints fail ($\vec{\sigma}_i^x \notin
  T_x(\cev{\sigma}_i^x)$), then again symbol $\sigma_i$ is not compatible with
  state $q_i$, and transition $\delta_{\SV}(q_i, \sigma_i)$ leads to the sink
  state $\sinkT$, again contradicting acceptance of $\lambda$.
  Since we reach a contradiction in all cases, we conclude
  $\mathcal{L}(\mathcal{T}_{\SV}) \subseteq W$.

  To show $W \subseteq \mathcal{L}(\mathcal{T}_{\SV})$, let us suppose, towards
  a contradiction, that there exists a word $\lambda = \seq{\sigma_0, \sigma_1,
    \ldots, \sigma_{n-1}} \in W$ such that $\lambda \notin
  \mathcal{L}(\mathcal{T}_{\SV})$.
  Consider the unique run of $\mathcal{T}_{\SV}$ on $\lambda$: $q^0_{\SV}
  \xrightarrow{\sigma_0} q_1 \xrightarrow{\sigma_1} \cdots
  \xrightarrow{\sigma_{n-1}} q_n$.
  Since $\lambda$ is not accepted, from the final states being $Q_{\SV}
  \setminus \{\sinkT\}$ (i.e., all the states but the sink one) we have that the
  run eventually hits the sink state $\sinkT$.
  Let $i$ be the minimum index such that $q_i = \sinkT$.
  We can
  assume that $i>0$ since the fresh initial state $q^0_{\SV}$ is distinct from
  the sink one.
  Then, we have $\delta_{\SV}(q_{i}, \sigma_i) = \sinkT$.
  This happens exactly when $\sigma_i$ is not compatible with state $q_{i}$.
  However, since $\lambda \in W$, all the constraints that define compatibility
  are satisfied: the weak encoding condition and the domain/transition
  constraints ensure compatibility.
  This contradicts the incompatibility detected by the automaton, and thus we
  can conclude $W \subseteq \mathcal{L}(\mathcal{T}_{\SV})$.

  Finally, for the bound on the size of the automaton, we have that the size of
  $\mathcal{T}_{\SV}$ is at most polynomial in $|Q_{\SV}|$ and $|\Sigma_\SV|$.
  From the construction, we have
  $|Q_{\SV}| \leq |\Sigma_{\SV}| = |\Sigma_{\SV}^I| + |\Sigma_{\SV}^N| =
  |V|^{|\SV|} + (|V|^2 + 1)^{|\SV|}$, where $V = \cup_{x\in \SV}V_x$.
  Since each domain $V_x$ has size polynomial in the size of $P$, the
  total size $|V| = |\bigcup_{x \in \SV} V_x|$ is polynomial in $|P|$.
  Then, $|\Sigma_\SV|$, and thus $|Q_{\SV}|$ as well, is at most exponential in
  $|P|$.
\end{proof}



\section{DFA accepting solution plans}
\label{sec:dfaForSolutions}

%

In this section, we go through the construction of an automaton recognizing
solution plans for a planning problem.
Let $P = (\SV, S)$ be an eager qualitative timeline-based planning problem,
and let $V = \cup_{x\in \SV}V_x$. We show how to build a DFA
$\automAP$, whose size is at most exponential in the size of $P$, that
accepts exactly those words encoding solutions plans for $P$ when
restricted to words encoding plans over \SV. In different terms, if a word
encodes a plan over \SV, then it is accepted by \automAP if and only if it
encodes a solution plan for $P$. However, \automAP may also accept words
that do not encode a plan over \SV. Therefore, we need the intersection of
such a DFA \automAP with DFA $\mathcal T_{\SV}$ from the previous section.

Since, as already noticed, the ordering relations imposed by conjunctions of
atoms in synchronization rules induce preorders, in the following we use labeled
directed acyclic graphs (labeled DAGs) to represent them.
Intuitively, each conjunction of atoms $\rulebody$ identifies a DAG whose
vertices are the equivalence classes of terms $\tokstart(a)/\tokend(a)$
occurring in it.
Formally, let $\Rule$ be a synchronization rule, and let $\E$ and $\rulebody$
be, respectively, the existential statement and the conjunction of atoms of
\Rule.
%
%
For each term $t$ occurring in \rulebody we denote by $[t]_{\equiv_{\rulebody}}$
its \emph{equivalence class with respect to $\rulebody$}, defined as the set $\{
u \mid t \equiv u \in \hat \rulebody \}$.
We omit the subscript $_\rulebody$ when it is clear from the context.
Moreover, if $t = \tokstart(a)$ (resp., $t = \tokend(a)$) and $a[x=v]$ either
occurs in (the quantifier prefix of) \E or is the trigger of \Rule, then
$\termToEventt = \mathit{start}(x,v)$ (resp., $\termToEventt =
\mathit{end}(x,v)$); with an abuse of notation, we lift function \termToEventFun
to the domain of sets of terms, i.e., $\termToEventT = \{ \termToEventt \mid t
\in T \}$, for all sets of terms $T$.
The \emph{(labeled) DAG associated with \Rule}, denoted $G_\Rule$, is defined as
the tuple $(V_\Rule,A_\Rule,A^<_\Rule,\ell_\Rule)$, where (we omit the subscript
$_\Rule$ when clear from the context):
\begin{itemize}
\item $V = \{ [t]_\equiv \mid t $ is a term occurring in \rulebody $ \}$,
\item $A = \{ ([t]_\equiv, [u]_\equiv) \in V \times V \mid [t]_\equiv \neq
  [u]_\equiv$ and $t \before u \in \hat \rulebody \}$,
\item $A^< = \{ ([t]_\equiv, [u]_\equiv) \in A \mid t \before* u \in \hat
  \rulebody \}$, and
\item $\ell$ is the \emph{labeling function} assigning to each vertex
  $[t]_\equiv \in V$ its \emph{label} as follows: $\ell([t]_\equiv) =
  \termToEvent{[t]_\equiv}$.
%
\end{itemize}
Notice that labeling function $\ell$ is crucial to establish a correspondence
between vertices of the DAG and sets of events \eventsOnesigma associated with
alphabet symbols $\sigma$.
It is convenient to further lift function $\termToEventFun$ to deal with sets of
vertices (i.e., sets of sets of terms): $\termToEvent{V'} = \bigcup_{[t]_\equiv
  \in V'}\termToEvent{[t]_\equiv}$ for all $V' \subseteq V$.
%
%
When convenient, we may use dashed and solid arrows to denote arcs in $A$ and
$A^<$, respectively, thus writing $[t]_\equiv \dashrightarrow [u]_\equiv$ and
$[t]_\equiv \rightarrow [u]_\equiv$ for $([t]_\equiv,[u]_\equiv) \in A$ and
$([t]_\equiv,[u]_\equiv) \in A_<$, respectively.
We likewise use $\not\dashrightarrow$ and $\not\rightarrow$ to denote the
absence of the corresponding arcs.
\autoref{fig:dag-begin-strict-nonstrict} shows such a difference.


\begin{figure}[t]
\centering
\begin{tikzpicture}




\node [draw,rectangle,rounded corners,align=center] at (0,0) (n1) {\scriptsize
  $n_1 {=}
  \left\{ \hspace{-2mm} \begin{array}{l}
    \tokstart(a_0), \\
    \tokstart(a_1)
  \end{array} \hspace{-2mm} \right\}
$};
\path (n1) ++(3,0) node[draw,rectangle,rounded corners] (n2) {\scriptsize
  $n_2 {=} \{ \tokend(a_0) \}$};
\path (n2) ++(3,0) node[draw,rectangle,rounded corners] (n3) {\scriptsize
  $n_3 {=} \{ \tokend(a_1) \}$};

\draw [->,>=stealth, distance = 0.05cm] (n1.east) -- (n2.west);

\draw [dashed,->,>=stealth] (n2.east) -- (n3.west);

\node[font=\scriptsize, above] at (n1.north) {(non-strict)};




\path [align=center]
(n1) ++ (0,-2) node[draw,rectangle,rounded corners] (n1prime) {\scriptsize
  $n'_1 {=}
  \left\{ \hspace{-2mm} \begin{array}{l}
    \tokstart(a_0), \\
    \tokstart(a_1)
  \end{array} \hspace{-2mm} \right\}
$};
\path (n1prime) ++(3,0) node[draw,rectangle,rounded corners] (n2prime) {
  \scriptsize
  $n'_2 {=} \tokend(a_0)$};
\path (n2prime) ++(3,0) node[draw,rectangle,rounded corners] (n3prime) {
  \scriptsize
  $n'_3 {=} \tokend(a_1)$};

\draw [->,>=stealth, distance = 0.05cm] (n1prime.east) -- (n2prime.west);

\draw [->,>=stealth, distance = 0.05cm] (n2prime.east) -- (n3prime.west);

\node[font=\scriptsize, above] at (n1prime.north) {(strict)};

\node[scale=.8] at (10,-1) {
  $\begin{array}{l}
     \ell(n_1) = \ell(n'_1) =
     \left\{ \hspace{-2mm} \begin{array}{l}
       \mathit{start}(x_0,v_0), \\
       \mathit{start}(x_1,v_1)
     \end{array} \hspace{-2mm} \right\} \\[5mm]
     \ell(n_2) = \ell(n'_2) = \{ \mathit{end}(x_0,v_0) \} \\[2.5mm]
     \ell(n_3) = \ell(n'_3) = \{ \mathit{end}(x_1,v_1) \}
  \end{array}$
};

\end{tikzpicture}

\caption{
  The picture shows two DAGs on the left-hand side.
  The one on the top is associated with (the existential statement of) the rule
  $a_0[x_0=v_0] \implies \exists a_1[x_1=v_1].
  (\tokstart(a_0) \before \tokstart(a_1) \wedge \tokstart(a_1) \before
  \tokstart(a_0) \wedge \tokend(a_0) \before \tokend(a_1))$.
%
%
%
  It forces token $a_0$ to either be a prefix of or coincide with token $a_1$.
  The one on the bottom is associated with the existential statement obtained
  replacing $\tokend(a_0) \before \tokend(a_1)$ with $\tokend(a_0) \before*
  \tokend(a_1)$, that forces $a_0$ to be a (strict) prefix of $a_1$.
  The labeling function is the same for both DAGs, and it is shown on the
  right-hand side.}
\label{fig:dag-begin-strict-nonstrict}
\end{figure}

\subsection{Viewpoints}
\label{sec:viewpoints}
%
%

In order to obtain the desired automaton construction, it is useful to introduce
an auxiliary structure,
%
%
called viewpoint.
%
%
A \emph{viewpoint} \viewpointV for \Rule is a pair $(G, K)$, where $G =
(V,A,A^<,\ell)$ is the (labeled) DAG associated with \Rule and $K \subseteq V$ is a
\emph{downward
  closed} subset of vertices of $G$, that is, $n \in K$ implies $n' \in K$ for
all $n,n' \in V$ with $n' \dashrightarrow n$.
%
The number of different viewpoints for \Rule is $2^{|V|}$, hence at most
exponential in the size of $P$.
%
%
If $K = \emptyset$, then $\viewpointV = (G,K)$ is \emph{initial}; analogously,
if $K$ is the entire set of vertices of $G$, then $\viewpointV$ is \emph{final}.
If \viewpointV is a viewpoint for \Rule, then we use \rulefunV to refer to
\Rule.
%
%
Intuitively, the downward closed component of a viewpoint for a synchronization
rule is meant to collect the events that are relevant to witness the
satisfaction of the rule, among those that are encoded in the letters of the
input word that have been read so far.
In other words, a viewpoint retains information about relevant tokens
encountered so far along the plan.
In what follows, we describe how viewpoints evolve.
%

According to the formalization provided below, states of automata \automAP are
sets of viewpoints containing at least one
viewpoint for each rule of $P$ (besides a fresh rejecting sink state \sinkAP).
Therefore, to define automata runs, we first show how viewpoints evolve upon
reading an alphabet symbol.
To this end, we introduce the following notions.
Let $\viewpointV = (G,K)$ be a viewpoint, with $G = (V,A,A^<,\ell)$.
We set $\nextV = K'$, where $K'$ is the largest downward closed subset of $V$
for which there is no pair of vertices $v,v' \in K'\setminus K$ with $v
\rightarrow v'$.
Notice that it is always the case that $\nextV \supseteq K$.
Moreover, given an alphabet symbol $\sigma \in \Sigma_\SV$, we define
$\consumedVsigma = K'$, where $K'$ is the largest downward closed subset of
vertices of $\nextV$ such that $\termToEvent{K' \setminus K} \subseteq
\eventsOnesigma$.
Once again, it holds that $\consumedVsigma \supseteq K$.
Intuitively, a viewpoint $\viewpointV = (G,K)$
evolves by suitably extending $K$; to this end, $\nextV$ identifies the only
vertices that can possibly be added to $K$ (independently from the alphabet
symbol read), that is, vertices of $V \setminus K$ reachable from $K$ with no
arcs in $A_{<}$ (solid arrows) among them.
The exact extension, however, depends on the actual symbol $\sigma$ read by the
automaton: $K$ cannot be extended with events that are not included in $\sigma$.
Therefore, \consumedVsigma identifies precisely how viewpoint \viewpointV
evolves upon reading $\sigma$.

There are events that \emph{must} be consumed by a viewpoint \viewpointV when
certain symbols (i.e., symbols that encode, among others, such events) are read
by the automaton.
We aim to collect such events in a set called the \emph{waiting list} of
\viewpointV, defined in the following.
The idea is that an automaton must reject if an event in the waiting list of
\viewpointV is read (through the symbol $\sigma$) and the viewpoint is not
\emph{ready} to consume it (i.e., the event is not featured in \consumedVsigma).
Let the \emph{past} of $\viewpointV$, denoted \pastV, be the set
$\bigcup_{[t]_{\equiv} \in K} [t]_{\equiv}$, and the \emph{future} of
$\viewpointV$, denoted \futureV, be the set $\bigcup_{[t]_{\equiv} \in V
  \setminus K} [t]_{\equiv}$.
Then, one would expect the waiting list of \viewpointV to include those events
$\tokend(a)$ (token ending) that are in the future of \viewpointV, but whose
corresponding  $\tokstart(a)$ (token beginning) is in the past of \viewpointV,
that is, the
beginning of a token has been read (and consumed), but its ending has not yet.
This intuition is quite correct, but there is an exception.
Let \rulebody be the unique clause occurring in the rule corresponding to
\viewpointV, namely \rulefunV, and let $\tokstart(a) \in \pastV$ and $\tokend(a)
\in \futureV$.
Then, $\tokend(a)$ is \emph{not} included in the waiting list of \viewpointV if
$a$ is not the trigger token of \rulefunV and there is no term $t \neq
\tokstart(a)$ such that $\tokstart(a) \before t \in \hat \rulebody$ and
$\tokend(a) \before t \not\in \hat \rulebody$.
The idea is that if there is no term $t$ that is required to occur after (or
together with) $\tokstart(a)$ and possibly before than $\tokend(a)$, then it is
safe for the automaton to be indulgent, and allow the viewpoint to ignore the
event $\tokend(a)$ that matches with $\tokstart(a)$ and wait instead for some
future $\tokend(a)$ event (that refers to the end of a different token than the
one $\tokstart(a)$ refers to), thus having non-matching $\tokstart(a)$ and
$\tokend(a)$.
Intuitively, the reason why it is safe for the automaton to be indulgent is
that, under these conditions, the positions of the non-matching event
$\tokstart(a)$ and the event $\tokstart(a)$ that matches the future $\tokend(a)$
coincide relative to the other events.
This is formally captured by defining the \emph{waiting list} of $\viewpointV$,
denoted \waitingV, as the set

\bigskip

\noindent
{

  \centering

  $\left\{\begin{array}{ll}
   \tokend(a) \mid
   & \tokstart(a) \in \pastV, \tokend(a) \in \futureV \text{,
     and at least one of the following holds:} \\[-3mm]
   & \parbox[t]{.8\linewidth}{
     \begin{itemize}
     \item $a$ is the trigger token of \rulefunV,
     \item there is a term $t \not\in \{ \tokstart(a), \tokend(a) \}$ such that
       $\tokstart(a) \before t \in \hat \rulebody$ and $\tokend(a) \before t
       \not\in \hat \rulebody$
     \end{itemize} }
  \end{array} \right\}$



}

\bigskip

%
%
\noindent
where \rulebody is the unique clause occurring in \rulefunV.
We finally say that $\viewpointV$ is \emph{compatible} with symbol $\sigma$ if
$\termToEvent{\waitingV} \cap \eventsOnesigma \subseteq
\termToEvent{\consumedVsigma \setminus K}$.
Observe that this last notion of compatibility ensures that a viewpoint is
allowed to evolve upon reading a symbol only if no token ending event in the
waiting list is overlooked (in other words, events read by the automaton that
belong to the waiting list \emph{must} be consumed).

We can now define the \emph{evolution of a viewpoint} $\viewpointV = (G,K)$ upon
reading an alphabet symbol $\sigma \in \Sigma_\SV$, denoted \evolviewVsigma, as
the viewpoint $(G, \consumedVsigma)$, if $\viewpointV$ is compatible with
$\sigma$; \evolviewVsigma is undefined otherwise.
%

\subsection{States, initial state, and final states of \texorpdfstring{\automAP}{AP}}

We have already mentioned that \emph{states} of \automAP are sets of viewpoints
containing at least one viewpoint for each rule $\Rule \in S$ (recall that $S$
is the set of rules in planning problem $P$), besides a fresh rejecting sink
state \sinkAP.
However, since it is crucial for us to bound the size of \automAP to be at most
exponential in the one of $P$, we impose the \emph{linearity condition},
formalized in what follows.
To this end, it is useful to assume, without loss of generality, that different
rules of $P$ involves disjoint sets of token names.
%
%
For each rule $\Rule \in S$, let \viewpointsetR be the set of all viewpoints
for \Rule, and let $\viewpointsetP = \bigcup_{\Rule \in S} \viewpointsetR$.
We define an ordering relation $\preceq$ between viewpoints: for all
$\viewpointV, \viewpointVprime \in \viewpointsetP$, with $\viewpointV = (G,K)$
and $\viewpointVprime = (G', K')$, it holds that $\viewpointV \preceq
\viewpointVprime$ if and only if
\begin{enumerate*}[label={\it (\roman*)}]
\item $G = G'$, that is, $\viewpointV, \viewpointVprime \in \viewpointsetR$ for
  some $\Rule \in S$, and
\item $K \subseteq K'$.
\end{enumerate*}
Intuitively, $\viewpointV \preceq \viewpointVprime$ captures the fact that
\viewpointVprime has gone further than \viewpointV in matching input symbols
with the beginning/end of a token, in order to satisfy a rule.
%
%
%
%


At this point, we can formalize the linearity condition, crucial to constrain
the size of \automAP (\autoref{lem:automatonSolutions}).

\begin{definition}[Linearity condition]
  \label{def:linearity}
  A set of viewpoints $\viewpointset$ satisfies the \emph{linearity condition}
  if for all viewpoints $\viewpointV, \viewpointVprime \in \viewpointset$ and
  rules $\Rule \in S$, if $\viewpointV, \viewpointVprime \in \viewpointsetR$,
  then $\viewpointV \preceq \viewpointVprime$ or $\viewpointVprime \preceq
  \viewpointV$ holds.
\end{definition}
\noindent Intuitively, we impose all viewpoints for the same rule in a state of \automAP to be linearly ordered.

We are now ready to formally characterize the set of \emph{states} of \automAP,
consisting of the sets $\viewpointset \subseteq \viewpointsetP$ of viewpoints
that contain at least one viewpoint for each rule $\Rule \in S$ and that satisfy
the linearity condition, and including, in addition, a fresh \emph{rejecting
  sink state} \sinkAP.
We denote it by $Q_P$; formally, $Q_P = \{ \sinkAP \} \cup \{ \viewpointset
\subseteq \viewpointsetP \mid \viewpointset \cap \viewpointsetR \neq \emptyset
\text{ for all } \Rule \in S \text{ and } \viewpointset \text{ satisfies the
  linearity condition} \}$.
The \emph{initial state} $q^0_P$ of \automAP is the set $\{
\viewpointV^0_{\Rule} \mid \Rule \in S \}$, where $\viewpointV^0_\Rule$ is the
initial viewpoint of rule \Rule.
Towards a definition of the set $F_P$ of \emph{final states} of \automAP, we
introduce the notion of \emph{enabled viewpoints}. A viewpoint $\viewpointV =
(G,K)$ for rule $\Rule \in S$ is \emph{enabled} if either \Rule is triggerless
or \Rule has trigger token $a_0$ and $[\tokstart(a_0)]_{\equiv} \in K$.
A state $q$ of \automAP is \emph{final} if it is not the rejecting sink state
\sinkAP and every enabled viewpoint therein is final.

\subsection{Transition function of \texorpdfstring{\automAP}{AP}}

The last step of our construction is the definition of the \emph{transition
  function} $\delta_P$ for automaton \automAP.
To this end, we first introduce the notion of alphabet symbol \emph{enabling} a
viewpoint \viewpointV, along with the one of state of \automAP \emph{compatible}
with an alphabet symbol.
Let \viewpointV be a viewpoint for a non-triggerless rule \Rule with trigger
token $a_0$ and $\sigma \in \Sigma_\SV$ an alphabet symbol.
We say that $\sigma$ \emph{enables} $\viewpointV = (G,K)$ if
$[\tokstart(a_0)]_\equiv \in \consumedVsigma \setminus K$.
Moreover, we say that a state $q \in Q_P \setminus \{ \sinkAP \}$ is
\emph{compatible} with $\sigma$ if
\begin{enumerate*}[label={\it (\roman*)}]
\item all viewpoints in $q$ are compatible with $\sigma$ and
\item for all non-triggerless rules $\Rule \in S$, if $\sigma$ \emph{triggers}
  \Rule (i.e., $\mathit{start}(x_0,v_0) \in \eventsOnesigma$ with $a_0[x_0 =
  v_0]$ trigger token of \Rule), then there is a viewpoint $\viewpointV \in q$
  for \Rule such that $\sigma$ enables \viewpointV.
\end{enumerate*}

We are now ready to define the \emph{transition function} $\delta_P$ of
\automAP.
For all $q \in Q_P$ and alphabet symbol $\sigma \in
\Sigma_\SV$: \begin{itemize}
\item if $q = \sinkAP$ or $q$ is not compatible with $\sigma$, then
  $\delta(q,\sigma) = \sinkAP$;
\item otherwise, $\delta(q,\sigma) = q'$, where
  $q'$ is the smallest set such that for all $\viewpointV = (G,K) \in q$
  \begin{itemize}
  \item $\evolviewVsigma \in q'$ and
  \item if $\sigma$ enables \viewpointV, then $(G,K') \in q'$, where $K' =
    \consumedVsigma \setminus \{ [t]_\equiv \mid \tokstart(a_0) \before t \in
    \hat \rulebody \}$ and $a_0$ is the trigger token of \rulefunV.
  \end{itemize}
\end{itemize}

As an example, \autoref{fig:automaton-nonstrict} and
\autoref{fig:automaton-strict} depict the automata $\automAP$ for two
single-rule planning problems whose associated DAGs are shown in
\autoref{fig:dag-begin-strict-nonstrict} (non-strict and strict case,
respectively).
Each state $q_i$ is annotated with the set of downward closed subsets of DAG
vertices representing the viewpoint component~$K$ (the graph of the viewpoints
is always the same, i.e., the one shown in
\autoref{fig:dag-begin-strict-nonstrict}).
Transition labels encode conditions on the alphabet symbol~$\sigma$ being
read, expressed as pairs of constraints on events for timelines $x_0$ and
$x_1$; $v_i$ denotes the presence of the relevant event, $\neg v_i$ its
absence, $*$ no constraints, and $\circlearrowleft$ that the corresponding
timeline preserves its current value.
The strict case (\autoref{fig:automaton-strict}) requires six non-sink states,
compared to four non-sink states for the non-strict case
(\autoref{fig:automaton-nonstrict}), reflecting the finer tracking needed to
enforce the strict ordering constraint $\tokend(a_0) \before*
\tokend(a_1)$.

\begin{figure}[t]
\centering
\includegraphics[width=0.9\textwidth]{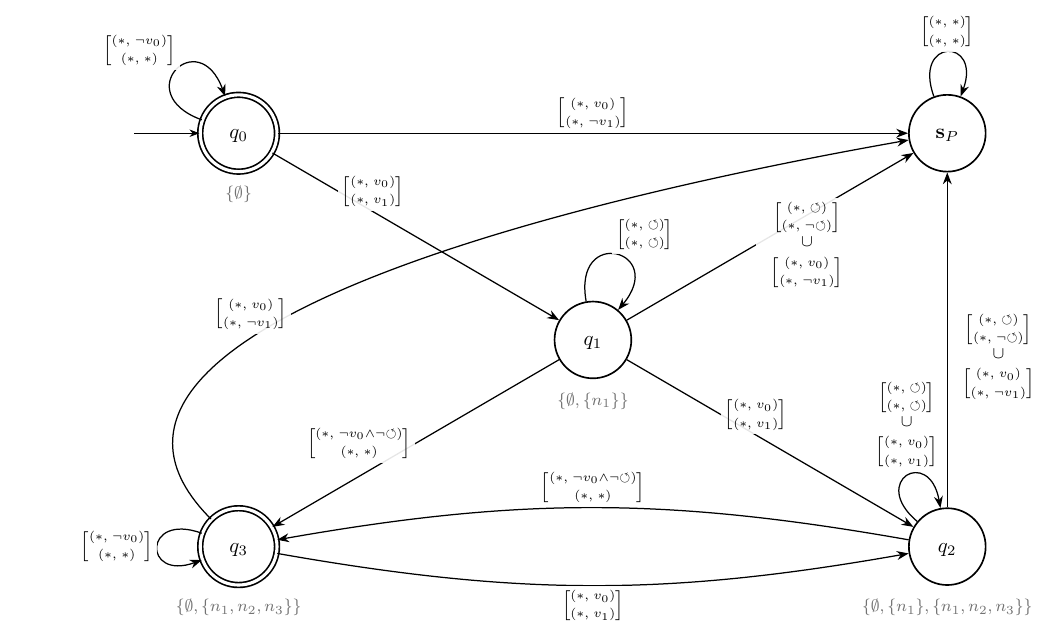}
\caption{The 
automaton $\automAP$ for the non-strict rule whose DAG is
depicted in \autoref{fig:dag-begin-strict-nonstrict} (non-strict), with four
non-sink states.}
\label{fig:automaton-nonstrict}
\end{figure}

\begin{figure}[t]
\centering
\includegraphics[width=0.9\textwidth]{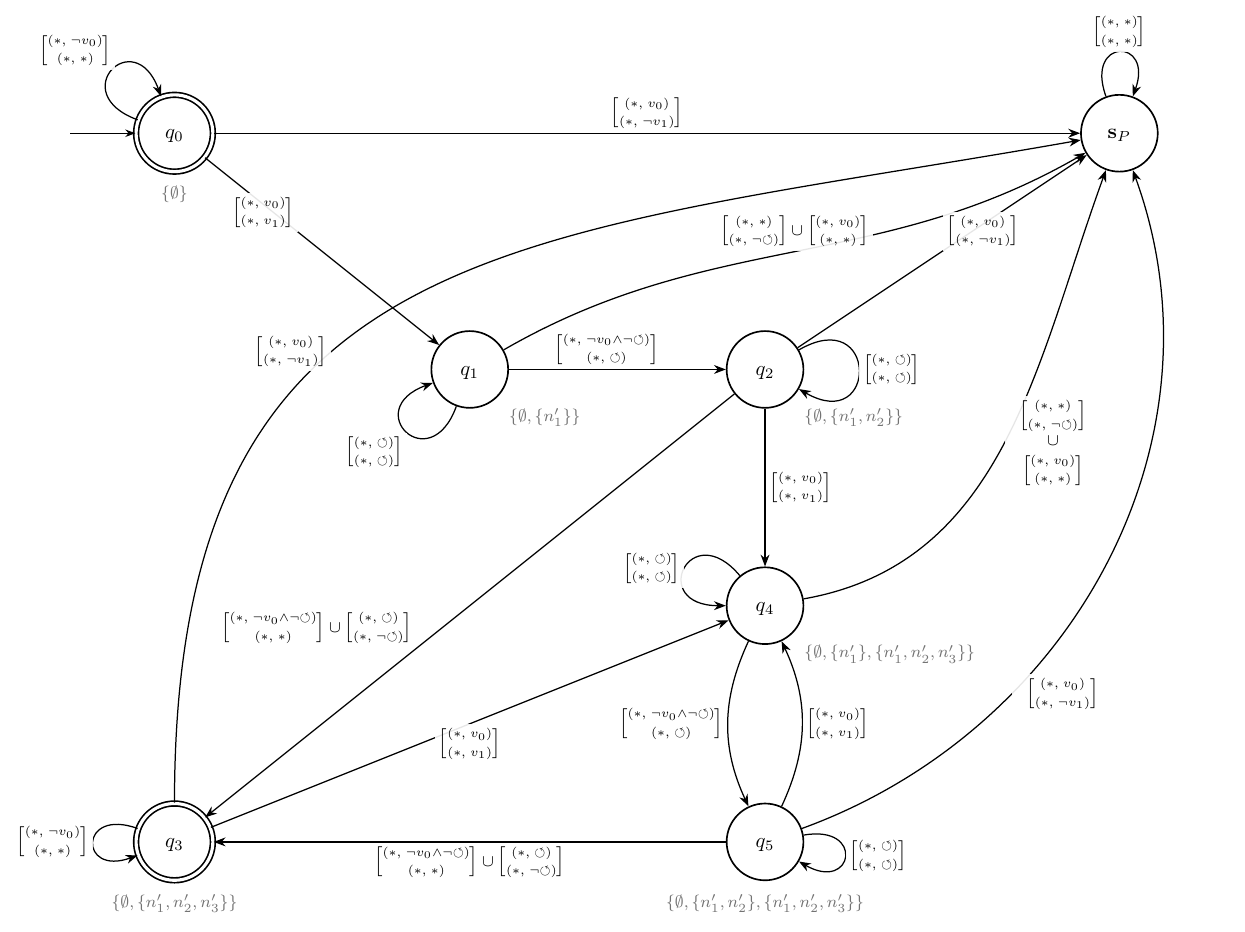}
\caption{The automaton $\automAP$ for the strict rule whose DAG is
depicted in \autoref{fig:dag-begin-strict-nonstrict} (strict), with six
non-sink states.}
\label{fig:automaton-strict}
\end{figure}

The proof of the following lemma is quite technical and notationally heavy.
Therefore, its proof is given in \autoref{app:proofSolutionAutomaton}.
\begin{restatable}{lemma}{lemAutomatonSolutions}
  \label{lem:automatonSolutions}
  Let $P = (\SV, S)$ be an eager qualitative timeline-based planning problem.
  Each finite word over $\Sigma_\SV$ that encodes a plan over \SV is
  accepted by \automAP if and only if it encodes a solution plan for $P$.
  Moreover, the size of $\automAP$ is at most exponential in the size of $P$.
\end{restatable}

Our final result follows from \autoref{lem:automatonPlans} and
\autoref{lem:automatonSolutions}.

\begin{theorem}\label{thm:automaton}
  Let $P = (\SV, S)$ be an eager qualitative timeline-based planning problem.
  Then, the words accepted by the intersection automaton of $\automAP$ and
  $\mathcal T_{\SV}$ are exactly those encoding solution plans for $P$.
  Moreover, the size of the intersection automaton of $\automAP$ and $\mathcal
  T_{\SV}$ is at most \emph{exponential} in the size of $P$.
\end{theorem}

\autoref{thm:automaton} enables the classical strategy synthesis procedure from a DFA arena \cite{PnueliR89}. Once the DFA is obtained, one solves a reachability/safety game on it via classical fixpoint algorithms. 





\section{Expressiveness of the eager fragment: A BPMN case study}\label{sec:eager-expressiveness}


\begin{figure}[t]
\centering
\includegraphics[width=0.9\textwidth]{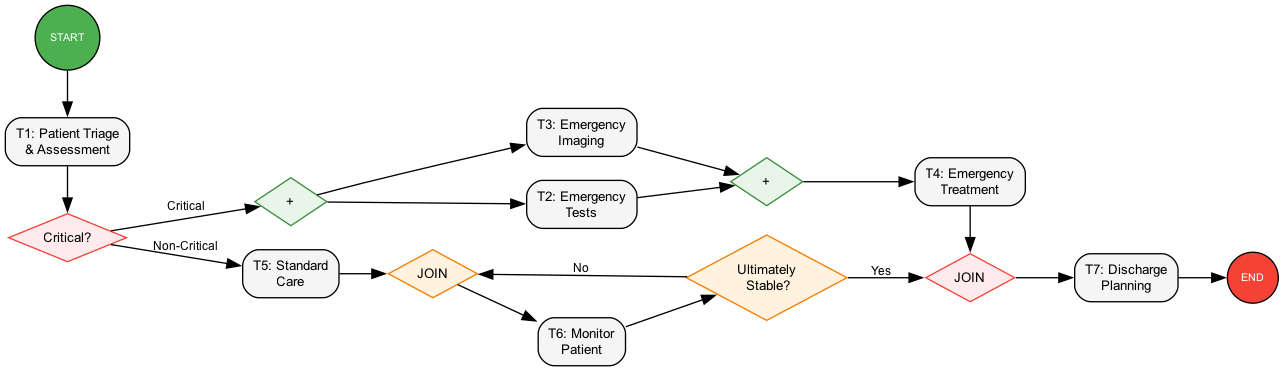}
\caption{Emergency Department BPMN Process Diagram}
\label{fig:example}
\end{figure}

Having established the theoretical foundations of the eager fragment and demonstrated its decidability properties through deterministic finite automata, we now turn to examining its practical expressiveness. To illustrate the modeling capabilities of the eager fragment, we present a comprehensive case study involving the translation of Business Process Model and Notation (BPMN) diagrams into timeline-based planning problems. BPMN provides a rich set of control flow constructs including sequential flows, parallel execution, exclusive choice (XOR), and loops, making it an ideal test-bed for demonstrating that the eager fragment can capture complex real-world process semantics. Through this translation, we show how various BPMN constructs can be systematically encoded using eager synchronization rules, thereby demonstrating the expressiveness of the fragment for representing sophisticated temporal and control flow constraints.

More precisely, the section is articulated in four stages. First, we introduce BPMN through a concrete example that illustrates the key modeling constructs and their intended semantics. Second, we introduce the concept of parse trees that capture the hierarchical decomposition of the well-structured class of BPMN diagrams into Single Entry Single Exit (SESE) blocks, which are the building blocks that will be compiled into the eager fragment. Third, we present the general translation methodology that systematically maps each type of BPMN construct to corresponding eager timeline synchronization rules. Finally, we apply this translation framework to our working example, demonstrating how the resulting timeline-based planning problem preserves the original process semantics while remaining within the eager fragment.

To illustrate our translation approach, we present a simplified Emergency Department process that demonstrates all major BPMN control flow patterns.
The process shown in Figure~\ref{fig:example} models patient flow through an emergency department. Upon arrival, every patient undergoes an initial triage and assessment phase (T1) where medical staff evaluate the patient condition and assign a priority level. Based on this assessment, patients are classified through an XOR gateway that routes them along different treatment paths according to the severity of their condition. Critical patients follow an intensive care pathway where emergency tests (T2) and emergency imaging (T3) are performed concurrently to rapidly diagnose the condition, after which emergency treatment (T4) is administered. Non-critical patients instead follow a standard care pathway where they receive routine medical care (T5) and are subsequently monitored (T6) in a loop structure that continues until their condition stabilizes. Regardless of which treatment path is followed, all patients eventually proceed to discharge planning (T7) where arrangements are made for their release or transfer to appropriate care facilities.

To systematically translate BPMN diagrams into timeline-based planning problems, we employ a structural decomposition approach based on parse trees and Single Entry Single Exit (SESE) blocks. A SESE block is a process fragment that has exactly one entry point and one exit point, meaning that control flow can only enter the block through a single incoming edge and can only leave through a single outgoing edge. This structural constraint ensures that SESE blocks are compositional units that can be translated independently as we will see.

The hierarchical structure of well-formed BPMN diagrams naturally decomposes into a parse tree where each node represents a 
SESE block of a specific type. Figure~\ref{fig:sesetree} shows the complete SESE decomposition of our Emergency Department process. 
We distinguish five fundamental block types that capture the essential BPMN control flow patterns. Task blocks 
($b_6$, $b_8$, $b_{11}$, $b_{12}$, $b_{13}$, $b_{15}$, $b_{16}$) represent atomic activities that cannot be further decomposed. 
Flow blocks ($b_1$, $b_2$, $b_4$, $b_7$) capture sequential execution where exactly two sub-blocks are executed in strict temporal order, with a ``before''
 child that must complete prior to the ``after'' child beginning execution. Parallel blocks ($b_5$) model concurrent execution where exactly two branches are activated simultaneously and execution continues only after both branches complete. Loop blocks ($b_9$) represent iterative structures where a body block may be executed multiple times based on continuation conditions. XOR blocks ($b_3$) implement exclusive choice where exactly one of two alternative branches is selected for execution based on process state or external conditions. When more than two sequential steps, parallel branches, or choice alternatives are required, they can be obtained by chaining multiple blocks of the same type, just as we demonstrate with the flow blocks $b_1$ and $b_2$   in our decomposition of Figure~\ref{fig:sesetree}.

\begin{figure}[t]
\centering
\includegraphics[width=0.8\textwidth]{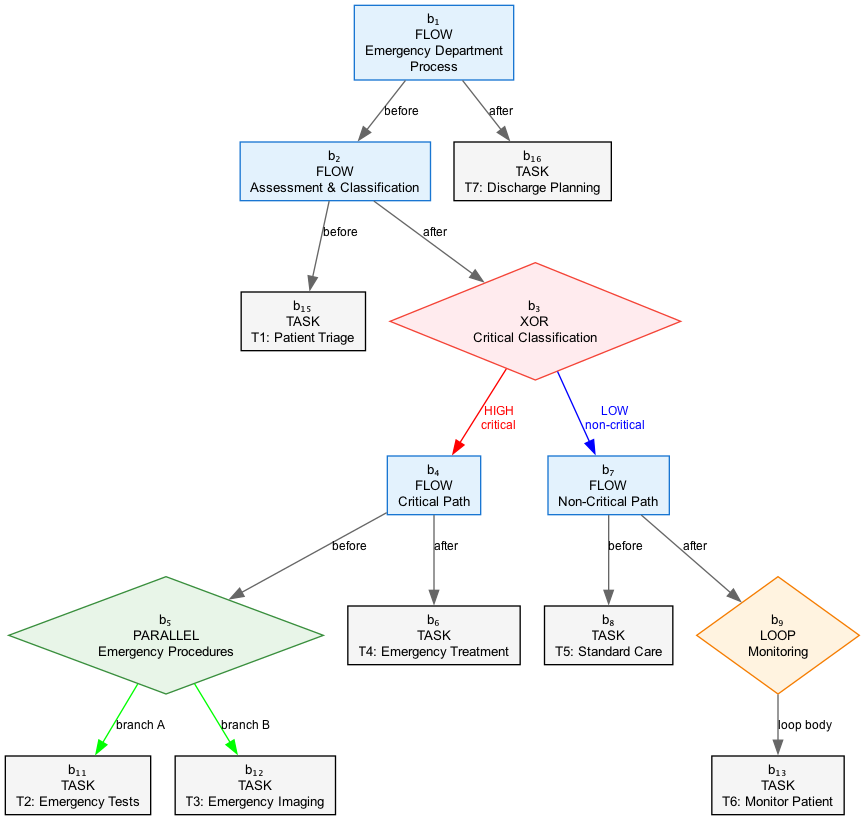}
\caption{\label{fig:sesetree}SESE Block Decomposition Tree for Emergency Department Process}
\end{figure}

The parse tree structure in Figure~\ref{fig:sesetree} reveals the hierarchical organization of the process through its binary decomposition. The root block $b_1$ represents the entire process as a flow block containing two sequential phases: the assessment and classification phase ($b_2$) executed before the final discharge planning ($b_{16}$). The assessment phase $b_2$ further decomposes into initial patient triage ($b_{15}$) followed by critical classification ($b_3$). The XOR block $b_3$ implements the branching logic that routes patients to either the critical path ($b_4$) or non-critical path ($b_7$) based on their condition severity. The critical path $b_4$ sequences emergency procedures ($b_5$) before emergency treatment ($b_6$), where the emergency procedures block coordinates concurrent execution of emergency tests ($b_{11}$) and imaging ($b_{12}$). The non-critical path $b_7$ sequences standard care ($b_8$) before the monitoring loop ($b_9$) containing the patient monitoring task ($b_{13}$). This decomposition enables systematic translation where each SESE block type maps to specific eager timeline synchronization rules that preserve the intended control flow semantics.

%

Let us provide now the general translation rules, then we will apply them to our working example.
The construction of BPMN semantics into timeline-based planning is not a novel concept. In \cite{DBLP:conf/time/CombiOS19}, the authors present a comprehensive approach that utilizes the full quantitative fragment of timeline-based planning, which allows for disjunctions and encompasses all possible relations among temporal points within a clause. However, this comprehensive expressiveness comes at a significant computational cost, as the resulting planning problems are EXPSPACE-complete. In contrast, the encoding proposed in this work adopts a more restrictive but computationally efficient approach. By limiting ourselves to the eager fragment of qualitative timeline-based planning, we achieve a more tractable complexity class of PSPACE-complete, while still maintaining sufficient expressiveness to capture all major BPMN control flow patterns as we demonstrate in this section.

To systematically translate BPMN diagrams into timeline-based planning problems, we define a translation function $\mathcal{R}$ that maps each SESE block to a set of eager timeline synchronization rules. This function decomposes into two complementary components:
$$\mathcal{R}(b) = \mathcal{R}_{\text{forward}}(b) \cup \mathcal{R}_{\text{backward}}(b).$$
This decomposition into forward and backward rules serves just for conceptual clarity
with the dichotomy determined by which block serves as the trigger 
in each synchronization rule:

\begin{itemize}
\item rules in $\mathcal{R}_{\text{forward}}(b)$ are triggered by parent blocks (rules of the form $a_0[x_{\text{parent}} = \top] \implies \ldots$) and capture the natural control flow semantics where parent blocks orchestrate the activation patterns of their children, implementing the specific behavioral constraints of each block type.

\item rules in $\mathcal{R}_{\text{backward}}(b)$ are triggered by child blocks (rules of the form $a_0[x_{\text{child}} = \top] \implies \ldots$) and establish the essential parent-child dependency relationships, ensuring that child blocks can only execute when their parent blocks permit such activation.

\end{itemize}

\noindent 
Given a SESE tree decomposition with blocks $B = \{b_1, b_2, \ldots, b_n\}$, the complete timeline-based planning problem encoding the original BPMN diagram is defined as:

$$P = (\mathcal{SV}, \mathcal{S})$$

\noindent  where $\mathcal{SV}$ is the set of all state variables introduced for the blocks, and $\mathcal{S}$ is the complete set of synchronization rules:

$$\mathcal{S} = \bigcup_{b \in B} \mathcal{R}(b) \cup \{\exists t[x_{b_{\text{root}}} = \top]. \top\} = \bigcup_{b \in B} \left(\mathcal{R}_{\text{forward}}(b) \cup \mathcal{R}_{\text{backward}}(b)\right) \cup \{\exists t[x_{b_{\text{root}}} = \top]. \top\}$$

\noindent The singleton triggerless rule $\exists t[x_{b_{\text{root}}} = \top]. \top$ added to $\mathcal{S}$ will be discussed below as it constitutes a technical requirement for proper process initiation.

\noindent Let us define the state variables first. For each SESE block $b$ in the decomposition, we introduce state variables that capture the activation status and control flow decisions. The specific variables as well as their domain depend on the block type.
For each block $b$ in the SESE tree, we introduce a state variable
$$x_b \text{ with domain } V_{x_b} = \{\top, \bot\}$$
where $\top$ indicates that the block is active (executing) and $\bot$ indicates that the block is inactive. This applies uniformly to all block types: TASK, FLOW, PARALLEL, LOOP, and XOR blocks.
The following transition functions for block state  avoids consecutive disabled intervals that make no sense in process execution:
\begin{align*}
T_{x_b}(\top) &= V_{x_b} = \{\top, \bot\} \\
T_{x_b}(\bot) &= \{\top\}
\end{align*}

\noindent  For XOR blocks, which implement exclusive choice semantics, we require an additional state variable to track which branch has been selected. Therefore, for each XOR block $b$, we introduce a second state variable
$$x_{b\_\text{dec}} \text{ with domain } V_{x_{b\_\text{dec}}} = \{\bot, \top_{\text{high}}, \top_{\text{low}}\}$$
This decision variable determines which branch is selected, where $\top_{\text{high}}$ selects the high-priority branch, $\top_{\text{low}}$ selects the low-priority branch, and $\bot$ indicates no decision has been made.
The transition functions for all state variables mentioned above allow transitions between any values in their respective domains, that is, $T_{x_b}(v) = V_{x_b}$ for all values $v$ in the domain.

Finally for FLOW blocks, which implement sequential execution semantics, we require an additional control flow timeline to ensure proper sequential execution and mutual exclusion between phases. Therefore, for each FLOW block $b$, we introduce a second state variable
$$x_{b\_\text{flow}} \text{ with domain } V_{x_{b\_\text{flow}}} = \{\bot, \top_{\text{before}}, \top_{\text{after}}\}$$
This control variable enforces the sequential progression from the before phase to the after phase, with automatic mutual exclusion between the two phases.
For such a reason this is the only variable that has a constrained  transition function which enforces sequential execution:
\begin{align*}
T_{x_{b\_\text{flow}}}(\bot) &= \{\top_{\text{before}}\} \\
T_{x_{b\_\text{flow}}}(\top_{\text{before}}) &= \{ \top_{\text{after}}\} \\
T_{x_{b\_\text{flow}}}(\top_{\text{after}}) &= \{\bot, \top_{\text{before}}\}
\end{align*}

This transition structure ensures that the flow begins in the disabled state, 
must transition to the before phase first, can transition from before to after, and prohibits direct transitions from disabled to after. 

\noindent Given a SESE tree with blocks $B = \{b_1, b_2, \ldots, b_n\}$, the complete set of state variables is:
$$\mathcal{SV} = \{x_b \mid b \in B\} \cup 
\begin{array}{c}\{x_{b\_\text{dec}} \mid b \in B \text{ and } b \text{ is an XOR block}\} \\\cup \\
    \{x_{b\_\text{flow}} \mid b \in B \text{ and } b \text{ is a FLOW block}\}\end{array}$$

We must address a special case in our encoding: the root region of the SESE decomposition. Unlike other blocks that can transition between enabled and disabled states, the root region has a unique semantics. The root region state variable $x_b$ (where $b$ denotes the root block) has only the $\top$ (enabled) value in its domain $V_{x_b} = \{\top\}$, and its transition function is defined as $T_{x_b}(\top) = \emptyset$, meaning that once the root region becomes active, it remains active throughout the entire process execution. This represents a semantic constraint ensuring that each plan corresponds to exactly one complete process computation.

As a goal specification, we represent the requirement for process initiation as a triggerless rule: $\exists t[x_b = \top]. \top$. This triggerless rule mandates that the root region must be activated, which will cascadingly trigger all due synchronization rules according to the hierarchical structure of the SESE decomposition. This design choice ensures that every valid plan represents a complete execution of the BPMN process from start to finish, rather than partial or incomplete executions.

\noindent We now define the functions $\mathcal{R}_{\text{forward}}$ and $\mathcal{R}_{\text{backward}}$ 
responsible of generating all the synchronization rules for each block type. It is crucial to note that all synchronization rules involved in our BPMN encoding are eager, which ensures the deterministic automaton construction presented in this paper. The eager nature of these rules typically manifests in three common patterns: rules where the trigger token equals another token (trigger $=$ another token), rules where the trigger token is started by or starts another token (corresponding to specific rows in Table~\ref{tab:allen-eagerness}   of Section~\ref{sec:allen}), and rules where the trigger token is ended by a non-trigger token, which amounts to non-trigger token ending the trigger token.
When rules of these kinds appear for the first time in the encoding below, we reference the corresponding 
row in Table~\ref{tab:allen-eagerness}   of Section~\ref{sec:allen}.
The fact that all such rows  have ``Overall ambiguous'' value equal to ``no'' establishes the eagerness of the encoding.
 A notable aspect of our encoding is the use of reflexive versions of Allen relations. For instance, when we say that token $a$ starts token $b$, we include the case where $a = b$, 
 meaning that the loose end constraint is $e(a) \leq e(b)$ rather than the strict inequality $e(a) < e(b)$ required by the traditional Allen relation. 
 As we will point out in Section~\ref{sec:allen}, reflexivity for relations such as starts, ends, and during (where both endpoints have non-strict equality) 
 does not affect the eagerness status: if a rule is eager with strict relations, it remains eager with reflexive relations, and if it is not eager with strict relations, 
 it remains non-eager with reflexive relations.

\begin{figure}[t]
\centering
\begin{tikzpicture}[x=0.8cm,y=1.2cm,
    customfont/.style={font=\fontsize{7}{9}\selectfont},
    customsmall/.style={font=\fontsize{6}{8}\selectfont}]

    \foreach \t in {0,1,2,3,4,5,6,7,8,9,10,11}
    {
        \draw[gray!35] (\t,8.5) -- (\t,1);
        \node[gray, customsmall] at (\t,8.7) {\t};
    }

    \node[left, customfont] at (-1.5, 8) {$x_{b_{1}}$ (Root)};
    \node[left, customfont] at (-1.5, 7.5) {$x_{b_{1}\_flow}$ (Root Flow)};
    \node[left, customfont] at (-1.5, 7) {$x_{b_{2}}$ (Assessment)};
    \node[left, customfont] at (-1.5, 6.5) {$x_{b_{2}\_flow}$ (Assessment Flow)};
    \node[left, customfont] at (-1.5, 6) {$x_{b_{15}}$ (T1: Triage)};
    \node[left, customfont] at (-1.5, 5.5) {$x_{b_{3}}$ (XOR Critical)};
    \node[left, customfont] at (-1.5, 5) {$x_{b_{3}\_dec}$ (XOR Decision)};
    \node[left, customfont] at (-1.5, 4.5) {$x_{b_{7}}$ (Non-Critical Path)};
    \node[left, customfont] at (-1.5, 4) {$x_{b_{7}\_flow}$ (Non-Critical Flow)};
    \node[left, customfont] at (-1.5, 3.5) {$x_{b_{8}}$ (T5: Standard Care)};
    \node[left, customfont] at (-1.5, 3) {$x_{b_{9}}$ (Loop Block)};
    \node[left, customfont] at (-1.5, 2.5) {$x_{b_{13}}$ (T6: Monitor)};
    \node[left, customfont] at (-1.5, 2) {$x_{b_{16}}$ (T7: Discharge)};

    \draw[thick, blue] (0,8) -- (11,8);
    \draw[thick, blue] (0,7.9) -- (0,8.1);
    \draw[thick, blue] (11,7.9) -- (11,8.1);
    \node[above, customfont] at (5.5,8) {$\top$ (Active)};

    \draw[thick, blue!70] (0,7.5) -- (9,7.5);
    \draw[thick, blue!70] (0,7.4) -- (0,7.6);
    \draw[thick, blue!70] (9,7.4) -- (9,7.6);
    \node[above, customfont] at (4.5,7.5) {$\top_{before}$ (Assessment Phase)};
    
    \draw[thick, blue!70] (9,7.5) -- (11,7.5);
    \draw[thick, blue!70] (9,7.4) -- (9,7.6);
    \draw[thick, blue!70] (11,7.4) -- (11,7.6);
    \node[above, customfont] at (10,7.5) {$\top_{after}$ (Discharge Phase)};

    \draw[thick, blue] (0,7) -- (9,7);
    \draw[thick, blue] (0,6.9) -- (0,7.1);
    \draw[thick, blue] (9,6.9) -- (9,7.1);
    \node[above, customfont] at (4.5,7) {$\top$ (Active)};

    \draw[thick, blue!70] (0,6.5) -- (2,6.5);
    \draw[thick, blue!70] (0,6.4) -- (0,6.6);
    \draw[thick, blue!70] (2,6.4) -- (2,6.6);
    \node[above, customfont] at (1,6.5) {$\top_{before}$ (Triage)};
    
    \draw[thick, blue!70] (2,6.5) -- (9,6.5);
    \draw[thick, blue!70] (2,6.4) -- (2,6.6);
    \draw[thick, blue!70] (9,6.4) -- (9,6.6);
    \node[above, customfont] at (5.5,6.5) {$\top_{after}$ (Classification)};

    \draw[thick, gray] (0,6) -- (2,6);
    \draw[thick, gray] (0,5.9) -- (0,6.1);
    \draw[thick, gray] (2,5.9) -- (2,6.1);
    \node[above, customfont] at (1,6) {$\top$ (Executing)};

    \draw[thick, red] (2,5.5) -- (9,5.5);
    \draw[thick, red] (2,5.4) -- (2,5.6);
    \draw[thick, red] (9,5.4) -- (9,5.6);
    \node[above, customfont] at (5.5,5.5) {$\top$ (Active)};

    \draw[thick, red!70] (2,5) -- (9,5);
    \draw[thick, red!70] (2,4.9) -- (2,5.1);
    \draw[thick, red!70] (9,4.9) -- (9,5.1);
    \node[above, customfont] at (5.5,5) {$\top_{low}$ (Non-Critical Selected)};

    \draw[thick, blue] (2,4.5) -- (9,4.5);
    \draw[thick, blue] (2,4.4) -- (2,4.6);
    \draw[thick, blue] (9,4.4) -- (9,4.6);
    \node[above, customfont] at (5.5,4.5) {$\top$ (Active)};

    \draw[thick, blue!70] (2,4) -- (4,4);
    \draw[thick, blue!70] (2,3.9) -- (2,4.1);
    \draw[thick, blue!70] (4,3.9) -- (4,4.1);
    \node[above, customfont] at (2.5,4) {$\top_{before}$ (Standard Care)};
    
    \draw[thick, blue!70] (4,4) -- (9,4);
    \draw[thick, blue!70] (4,3.9) -- (4,4.1);
    \draw[thick, blue!70] (9,3.9) -- (9,4.1);
    \node[above, customfont] at (7,4) {$\top_{after}$ (Monitoring Phase)};

    \draw[thick, gray] (2,3.5) -- (4,3.5);
    \draw[thick, gray] (2,3.4) -- (2,3.6);
    \draw[thick, gray] (4,3.4) -- (4,3.6);
    \node[above, customfont] at (3,3.5) {$\top$ (Standard Care)};

    \draw[thick, orange] (4,3) -- (9,3);
    \draw[thick, orange] (4,2.9) -- (4,3.1);
    \draw[thick, orange] (9,2.9) -- (9,3.1);
    \node[above, customfont] at (6.5,3) {$\top$ (Loop Active)};

    \draw[thick, gray] (4,2.5) -- (7,2.5);
    \draw[thick, gray] (4,2.4) -- (4,2.6);
    \draw[thick, gray] (7,2.4) -- (7,2.6);
    \node[above, customfont] at (5.2,2.5) {$\top$ (Monitor Iter 1)};

    \draw[thick, gray] (7,2.5) -- (9,2.5);
    \draw[thick, gray] (7,2.4) -- (7,2.6);
    \draw[thick, gray] (9,2.4) -- (9,2.6);
    \node[above, customfont] at (8.7,2.5) {$\top$ (Monitor Iter 2)};

    \draw[thick, gray] (9,2) -- (11,2);
    \draw[thick, gray] (9,1.9) -- (9,2.1);
    \draw[thick, gray] (11,1.9) -- (11,2.1);
    \node[above, customfont] at (10.5,2) {$\top$ (Discharge)};

    \draw[dashed, gray, thick] (0,1.5) -- (0,8.5);
    \node[below, customsmall] at (0,1.2) {Triage→Classification};
    
    \draw[dashed, gray, thick] (2,1.5) -- (2,8.5);
    \node[below, customsmall] at (2,1.0) {Non-Critical Start};
    
    \draw[dashed, gray, thick] (7,1.5) -- (7,8.5);
    \node[below, customsmall] at (7,1.2) {Monitor Iter 1→2};

    \draw[dashed, gray, thick] (9,1.5) -- (9,8.5);
    \node[below, customsmall] at (9,1.0) {Monitoring→Discharge};

    \draw[->, thick, gray] (2,6.2) -- (2,5.3);
    \draw[->, thick, gray] (2,5.2) -- (2,4.3);
    \draw[->, thick, gray] (9,4.2) -- (9,1.8);

\end{tikzpicture}

\caption{Timeline execution for the non-critical path in the Emergency Department BPMN process with two iterations of the monitoring loop. The diagram shows the execution flow when the XOR decision selects the low branch (non-critical patient). The monitoring loop ($b_9$) executes twice with consecutive monitoring task iterations: first iteration from time 4-7 (length 3), second iteration from time 7-9 (length 2). The loop block covers exactly the concatenation of the two monitoring executions with no gaps, with the patient becoming stable after the second monitoring cycle, allowing progression to discharge planning.}
\label{fig:non-critical-loop-timeline}
\end{figure}
\begin{figure}[t]
\centering
\begin{tikzpicture}[x=0.8cm,y=1.2cm,
    customfont/.style={font=\fontsize{7}{9}\selectfont},
    customsmall/.style={font=\fontsize{6}{8}\selectfont}]

    \foreach \t in {0,1,2,3,4,5,6,7,8,9,10,11}
    {
        \draw[gray!35] (\t,8.5) -- (\t,1);
        \node[gray, customsmall] at (\t,8.7) {\t};
    }

    \node[left, customfont] at (-1.5, 8) {$x_{b_{1}}$ (Root)};
    \node[left, customfont] at (-1.5, 7.5) {$x_{b_{1}\_flow}$ (Root Flow)};
    \node[left, customfont] at (-1.5, 7) {$x_{b_{2}}$ (Assessment)};
    \node[left, customfont] at (-1.5, 6.5) {$x_{b_{2}\_flow}$ (Assessment Flow)};
    \node[left, customfont] at (-1.5, 6) {$x_{b_{15}}$ (T1: Triage)};
    \node[left, customfont] at (-1.5, 5.5) {$x_{b_{3}}$ (XOR Critical)};
    \node[left, customfont] at (-1.5, 5) {$x_{b_{3}\_dec}$ (XOR Decision)};
    \node[left, customfont] at (-1.5, 4.5) {$x_{b_{4}}$ (Critical Path)};
    \node[left, customfont] at (-1.5, 4) {$x_{b_{4}\_flow}$ (Critical Flow)};
    \node[left, customfont] at (-1.5, 3.5) {$x_{b_{5}}$ (Parallel Block)};
    \node[left, customfont] at (-1.5, 3) {$x_{b_{11}}$ (T2: Tests)};
    \node[left, customfont] at (-1.5, 2.5) {$x_{b_{12}}$ (T3: Imaging)};
    \node[left, customfont] at (-1.5, 2) {$x_{b_{6}}$ (T4: Treatment)};
    \node[left, customfont] at (-1.5, 1.5) {$x_{b_{16}}$ (T7: Discharge)};

    \draw[thick, blue] (0,8) -- (11,8);
    \draw[thick, blue] (0,7.9) -- (0,8.1);
    \draw[thick, blue] (11,7.9) -- (11,8.1);
    \node[above, customfont] at (5.5,8) {$\top$ (Active)};

    \draw[thick, blue!70] (0,7.5) -- (9,7.5);
    \draw[thick, blue!70] (0,7.4) -- (0,7.6);
    \draw[thick, blue!70] (9,7.4) -- (9,7.6);
    \node[above, customfont] at (4.5,7.5) {$\top_{before}$ (Assessment Phase)};
    
    \draw[thick, blue!70] (9,7.5) -- (11,7.5);
    \draw[thick, blue!70] (9,7.4) -- (9,7.6);
    \draw[thick, blue!70] (11,7.4) -- (11,7.6);
    \node[above, customfont] at (10,7.5) {$\top_{after}$ (Discharge Phase)};

    \draw[thick, blue] (0,7) -- (9,7);
    \draw[thick, blue] (0,6.9) -- (0,7.1);
    \draw[thick, blue] (9,6.9) -- (9,7.1);
    \node[above, customfont] at (4.5,7) {$\top$ (Active)};

    \draw[thick, blue!70] (0,6.5) -- (2,6.5);
    \draw[thick, blue!70] (0,6.4) -- (0,6.6);
    \draw[thick, blue!70] (2,6.4) -- (2,6.6);
    \node[above, customfont] at (1,6.5) {$\top_{before}$ (Triage)};
    
    \draw[thick, blue!70] (2,6.5) -- (9,6.5);
    \draw[thick, blue!70] (2,6.4) -- (2,6.6);
    \draw[thick, blue!70] (9,6.4) -- (9,6.6);
    \node[above, customfont] at (5.5,6.5) {$\top_{after}$ (Classification)};

    \draw[thick, gray] (0,6) -- (2,6);
    \draw[thick, gray] (0,5.9) -- (0,6.1);
    \draw[thick, gray] (2,5.9) -- (2,6.1);
    \node[above, customfont] at (1,6) {$\top$ (Executing)};

    \draw[thick, red] (2,5.5) -- (9,5.5);
    \draw[thick, red] (2,5.4) -- (2,5.6);
    \draw[thick, red] (9,5.4) -- (9,5.6);
    \node[above, customfont] at (5.5,5.5) {$\top$ (Active)};

    \draw[thick, red!70] (2,5) -- (9,5);
    \draw[thick, red!70] (2,4.9) -- (2,5.1);
    \draw[thick, red!70] (9,4.9) -- (9,5.1);
    \node[above, customfont] at (5.5,5) {$\top_{high}$ (Critical Selected)};

    \draw[thick, blue] (2,4.5) -- (9,4.5);
    \draw[thick, blue] (2,4.4) -- (2,4.6);
    \draw[thick, blue] (9,4.4) -- (9,4.6);
    \node[above, customfont] at (5.5,4.5) {$\top$ (Active)};

    \draw[thick, blue!70] (2,4) -- (6,4);
    \draw[thick, blue!70] (2,3.9) -- (2,4.1);
    \draw[thick, blue!70] (6,3.9) -- (6,4.1);
    \node[above, customfont] at (3,4) {$\top_{before}$ (Emergency Procedures)};
    
    \draw[thick, blue!70] (6,4) -- (9,4);
    \draw[thick, blue!70] (6,3.9) -- (6,4.1);
    \draw[thick, blue!70] (9,3.9) -- (9,4.1);
    \node[above, customfont] at (7.5,4) {$\top_{after}$ (Treatment)};

    \draw[thick, green] (2,3.5) -- (6,3.5);
    \draw[thick, green] (2,3.4) -- (2,3.6);
    \draw[thick, green] (6,3.4) -- (6,3.6);
    \node[above, customfont] at (4,3.5) {$\top$ (Parallel Active)};

    \draw[thick, gray] (2,3) -- (6,3);
    \draw[thick, gray] (2,2.9) -- (2,3.1);
    \draw[thick, gray] (6,2.9) -- (6,3.1);
    \node[above, customfont] at (4,3) {$\top$ (Testing)};

    \draw[thick, gray] (2,2.5) -- (6,2.5);
    \draw[thick, gray] (2,2.4) -- (2,2.6);
    \draw[thick, gray] (6,2.4) -- (6,2.6);
    \node[above, customfont] at (4,2.5) {$\top$ (Imaging)};

    \draw[thick, gray] (6,2) -- (9,2);
    \draw[thick, gray] (6,1.9) -- (6,2.1);
    \draw[thick, gray] (9,1.9) -- (9,2.1);
    \node[above, customfont] at (7.5,2) {$\top$ (Treatment)};

    \draw[thick, gray] (9,1.5) -- (11,1.5);
    \draw[thick, gray] (9,1.4) -- (9,1.6);
    \draw[thick, gray] (11,1.4) -- (11,1.6);
    \node[above, customfont] at (10,1.5) {$\top$ (Discharge)};

    \draw[dashed, gray, thick] (0,1) -- (0,8.5);
    \node[below, customsmall] at (0,0.8) {Triage→Classification};
    
    \draw[dashed, gray, thick] (2,1) -- (2,8.5);
    \node[below, customsmall] at (2,0.6) {Critical Path Start};
    
    \draw[dashed, gray, thick] (6,1) -- (6,8.5);
    \node[below, customsmall] at (6.5,0.6) {Parallel→Treatment};

    \draw[dashed, gray, thick] (9,1) -- (9,8.5);
    \node[below, customsmall] at (9,0.8) {Treatment→Discharge};

    \draw[->, thick, gray] (2,6.2) -- (2,5.3);
    \draw[->, thick, gray] (2,5.2) -- (2,4.3);
    \draw[->, thick, gray] (6,4.2) -- (6,1.8);
    \draw[->, thick, gray] (9,4.2) -- (9,1.3);

\end{tikzpicture}

\caption{Timeline execution for the critical path in the Emergency Department BPMN process, including all state variables and flow control timelines. The diagram shows both the main block activations and their corresponding flow control variables that enforce sequential execution semantics. Flow variables with subscript ``flow'' implement the before/after phases for FLOW blocks, while the decision variable for the XOR block tracks branch selection. Tasks T2 (Emergency Tests) and T3 (Emergency Imaging) execute in parallel between time points 2 and 6, synchronized by the parallel block constraint, both completing simultaneously before proceeding to emergency treatment (T4).}
\label{fig:critical-path-timeline}
\end{figure}

To illustrate the behavior of the synchronization rules across different block types, we present two concrete timeline executions that demonstrate alternative paths through the Emergency Department process depicted in Figure~\ref{fig:example}. The first execution, depicted in Figure~\ref{fig:non-critical-loop-timeline}, represents a plan satisfying the non-critical path, where patients receive standard care followed by iterative monitoring until their condition stabilizes. The second execution, shown in Figure~\ref{fig:critical-path-timeline}, represents a plan satisfying the critical path of the BPMN diagram, where patients classified as critical undergo emergency procedures including concurrent testing and imaging followed by emergency treatment.  These timeline figures will be referenced throughout the following block definitions to demonstrate how each type of synchronization rule enforces the intended BPMN semantics.
\\
\\
\noindent For TASK blocks $b$:
$$\mathcal{R}_{\text{forward}}(b) =  \mathcal{R}_{backward}(b) = \emptyset$$

\noindent Task blocks generate no synchronization rules as they represent atomic activities that do not decompose into sub-blocks. 
The positive duration constraint is automatically enforced by the qualitative timeline-based planning framework.
\\
\\ 
\noindent For FLOW blocks $b$ with children $b_{\text{before}}$ and $b_{\text{after}}$, we define:

$$\mathcal{R}_{\text{forward}}(b) = \left\{ \begin{array}{ll}
\text{(Ff.1)} & a_0[x_b = \top] \implies \exists a_1[x_{b\_\text{flow}} = \top_{\text{before}}]. \\ 
& \qquad (\text{start}(a_0) = \text{start}(a_1) \wedge \text{end}(a_1) \leq \text{end}(a_0)), \\[0.5em]
\text{(Ff.2)} & a_0[x_b = \top] \implies \exists a_1[x_{b\_\text{flow}} = \top_{\text{after}}]. \\ 
& \qquad (\text{start}(a_1) \leq \text{start}(a_0) \wedge \text{end}(a_0) = \text{end}(a_1)), \\[0.5em]
\text{(Ff.3)} & a_0[x_{b\_\text{flow}} = \top_{\text{before}}] \implies \exists a_1[x_b = \top]. \\ 
& \qquad (\text{start}(a_1) = \text{start}(a_0) \wedge \text{end}(a_0) \leq \text{end}(a_1)), \\[0.5em]
\text{(Ff.4)} & a_0[x_b = \bot] \implies \exists a_1[x_{b\_\text{flow}} = \bot]. \\ 
& \qquad (\text{start}(a_0) = \text{start}(a_1) \wedge \text{end}(a_0) = \text{end}(a_1)), \\[0.5em]
\text{(Ff.5)} & a_0[x_{b\_\text{flow}} = \bot] \implies \exists a_1[x_b = \bot]. \\ 
& \qquad (\text{start}(a_0) = \text{start}(a_1) \wedge \text{end}(a_0) = \text{end}(a_1)), \\[0.5em]
\text{(Ff.6)} & a_0[x_{b\_\text{flow}} = \top_{\text{before}}] \implies \exists a_1[x_{b_{\text{before}}} = \top]. \\ 
& \qquad (\text{start}(a_0) = \text{start}(a_1) \wedge \text{end}(a_0) = \text{end}(a_1)), \\[0.5em]
\text{(Ff.7)} & a_0[x_{b\_\text{flow}} = \top_{\text{after}}] \implies \exists a_1[x_{b_{\text{after}}} = \top]. \\ 
& \qquad (\text{start}(a_0) = \text{start}(a_1) \wedge \text{end}(a_0) = \text{end}(a_1))
\end{array} \right\}$$

$$\mathcal{R}_{\text{backward}}(b) = \left\{ \begin{array}{ll}
\text{(Fb.1)} & a_0[x_{b_{\text{before}}} = \top] \implies \exists a_1[x_{b\_\text{flow}} = \top_{\text{before}}]. \\ 
& \qquad (\text{start}(a_0) = \text{start}(a_1) \wedge \text{end}(a_0) = \text{end}(a_1)), \\[0.5em]
\text{(Fb.2)} & a_0[x_{b_{\text{after}}} = \top] \implies \exists a_1[x_{b\_\text{flow}} = \top_{\text{after}}]. \\ 
& \qquad (\text{start}(a_0) = \text{start}(a_1) \wedge \text{end}(a_0) = \text{end}(a_1))
\end{array} \right\}$$

\noindent Rules Ff.1-3, when paired with the transition function of $x_{b\_\text{flow}}$, effectively impose that block $b$ represents the sequential concatenation of $b_{\text{before}}$ and $b_{\text{after}}$. This constraint emerges through several key mechanisms. The if-and-only-if correspondence forced by rule Ff.6 paired with Fb.1, and rule Ff.7 paired with Fb.2, establishes a bidirectional relationship between the flow control phases and the actual child block executions. Additionally, rule Ff.3 imposes that any $b\_\text{flow}$ interval with value $\top_{\text{before}}$ must be a prefix of some $b$ interval, which, combined with the transition function constraint that $\top_{\text{before}}$ can only transition to $\top_{\text{after}}$, ensures that we can have exactly one $\top_{\text{before}}$ phase followed by just $\top_{\text{before}}$ phase for $x_{b_{flow}}$ inside any $\top$ interval of $b$. The remaining rules (Ff.4-5) provide the necessary constraints for proper disabled state management, ensuring consistency between the parent block and its control flow variable when neither is active.

This sequential execution behavior enforced by FLOW blocks can be observed in the timeline executions depicted in Figures~\ref{fig:critical-path-timeline} and~\ref{fig:non-critical-loop-timeline}. In both execution scenarios, the flow control variables (such as $x_{b_1\_\text{flow}}$, $x_{b_2\_\text{flow}}$, $x_{b_4\_\text{flow}}$, and $x_{b_7\_\text{flow}}$) demonstrate the strict sequential progression from $\top_{\text{before}}$ to $\top_{\text{after}}$ phases, with no overlap between consecutive phases and automatic mutual exclusion ensuring that only one phase can be active at any given time within each flow block execution interval.

The Allen relations appearing in the FLOW block encoding establish its eagerness according to Table~\ref{tab:allen-eagerness}. Rule Ff.1 implements the reflexive relation $a_1 \starts a_0$ with $a_0$ as trigger (row 11), rule Ff.2 implements the reflexive relation $a_1 \ends a_0$ with $a_0$ as trigger (row 8), and rule Ff.3 implements the reflexive relation $a_0 \starts a_1$ with $a_0$ as trigger (row 10). All remaining rules (Ff.4, Ff.5, Ff.6, Ff.7, Fb.1, Fb.2) implement equality relations $a_0$ equals $a_1$ with the respective trigger token (row 19). According to Table~\ref{tab:allen-eagerness}, all these configurations yield ``no'' in the ``Overall ambiguous'' column, confirming that the entire FLOW block encoding maintains eagerness properties essential for deterministic automaton construction. As noted earlier, the reflexive nature of these relations does not affect their eagerness status.
\\ 
\\ 
\noindent For PARALLEL blocks $b$ with children $b^1$ and $b^2$, we define:

$$\mathcal{R}_{\text{forward}}(b) = \left\{ \begin{array}{ll}
\text{(Pf.1)} & a_0[x_b = \top] \implies \exists a_1[x_{b^1} = \top]. \\ 
& \qquad (\text{start}(a_0) = \text{start}(a_1) \wedge \text{end}(a_0) = \text{end}(a_1)), \\[0.5em]
\text{(Pf.2)} & a_0[x_b = \top] \implies \exists a_1[x_{b^2} = \top]. \\ 
& \qquad (\text{start}(a_0) = \text{start}(a_1) \wedge \text{end}(a_0) = \text{end}(a_1))
\end{array} \right\}$$

$$\mathcal{R}_{\text{backward}}(b) = \left\{ \begin{array}{ll}
\text{(Pb.1)} & a_0[x_{b^1} = \top] \implies \exists a_1[x_b = \top]. \\ 
& \qquad (\text{start}(a_0) = \text{start}(a_1) \wedge \text{end}(a_0) = \text{end}(a_1)), \\[0.5em]
\text{(Pb.2)} & a_0[x_{b^2} = \top] \implies \exists a_1[x_b = \top]. \\ 
& \qquad (\text{start}(a_0) = \text{start}(a_1) \wedge \text{end}(a_0) = \text{end}(a_1))
\end{array} \right\}$$

\noindent These rules establish the bidirectional equivalence where block $b$ is enabled if and only if both children $b^1$ and $b^2$ are enabled on exactly the same interval. Note that explicit disabled state constraints are unnecessary since the enabled and disabled states are mutually exclusive by definition of block state variables.

This parallel execution behavior can be observed in Figure~\ref{fig:critical-path-timeline}, where the parallel block $x_{b_5}$ coordinates the simultaneous execution of emergency tests ($x_{b_{11}}$) and emergency imaging ($x_{b_{12}}$) between time points 2 and 6. Both child tasks execute concurrently with identical start and end times, synchronized by the parallel block constraint that enforces their exact temporal alignment. The PARALLEL block encoding introduces no new types of Allen relations beyond those discussed for FLOW blocks, as all rules implement equality relations between tokens.
\\ \\
\noindent For LOOP blocks $b$ with body child $b_{\text{body}}$, we define:

$$\mathcal{R}_{\text{forward}}(b) = \left\{ \begin{array}{ll}
\text{(Lf.1)} & a_0[x_b = \top] \implies \exists a_1[x_{b_{\text{body}}} = \top]. \\ 
& \qquad (\text{start}(a_0) = \text{start}(a_1) \wedge \text{end}(a_1) \leq \text{end}(a_0)), \\[0.5em]
\text{(Lf.2)} & a_0[x_b = \top] \implies \exists a_1[x_{b_{\text{body}}} = \top]. \\ 
& \qquad (\text{start}(a_1) \leq \text{start}(a_0) \wedge \text{end}(a_0) = \text{end}(a_1)), \\[0.5em]
\text{(Lf.3)} & a_0[x_b = \bot] \implies \exists a_1[x_{b_{\text{body}}} = \bot]. \\ 
& \qquad (\text{start}(a_0) = \text{start}(a_1) \wedge \text{end}(a_0) = \text{end}(a_1))
\end{array} \right\}$$

$$\mathcal{R}_{\text{backward}}(b) = \left\{ \begin{array}{ll}
\text{(Lb.1)} & a_0[x_{b_{\text{body}}} = \bot] \implies \exists a_1[x_b = \bot]. \\ 
& \qquad (\text{start}(a_0) = \text{start}(a_1) \wedge \text{end}(a_0) = \text{end}(a_1))
\end{array} \right\}$$

\noindent The bidirectional correspondence between disabled states (Lf.3 and Lb.1) ensures that the loop and its body are disabled on exactly the same intervals, meaning that when one is enabled the other must be too without gaps due to mutual exclusion of enabled and disabled states on block state variables. With only the bidirectional disabled correspondence, we would allow for chains of body executions inside loop intervals and chains of loop intervals inside body executions. To avoid this erroneous latter case, we explicitly require that any enabled loop admits both a prefix body execution (Lf.1) and a suffix body execution (Lf.2), thereby establishing the proper containment ordering where body executions are nested inside loop intervals rather than vice versa.

This iterative execution behavior can be observed in Figure~\ref{fig:non-critical-loop-timeline}, where the loop block $x_{b_9}$ enables multiple consecutive iterations of the monitoring task ($x_{b_{13}}$). The figure demonstrates two monitoring iterations: the first from time 4-7 (length 3) and the second from time 7-9 (length 2), with the loop block spanning exactly the concatenation of both iterations from time 4-9. The LOOP block encoding introduces no new types of Allen relations beyond those discussed for FLOW blocks, utilizing the same starts, ends, and equality relations to establish proper containment and sequencing constraints.
\\ \\

\noindent For XOR blocks $b$ with children $b^{\text{high}}$ and $b^{\text{low}}$, we define:

$$\mathcal{R}_{\text{forward}}(b) = \left\{ \begin{array}{ll}
\text{(Xf.1)} & a_0[x_b = \bot] \implies \exists a_1[x_{b\_\text{decision}} = \bot]. \\ 
& \qquad (\text{start}(a_0) = \text{start}(a_1) \wedge \text{end}(a_0) = \text{end}(a_1)), \\[0.5em]
\text{(Xf.2)} & a_0[x_{b\_\text{decision}} = \bot] \implies \exists a_1[x_b = \bot]. \\ 
& \qquad (\text{start}(a_0) = \text{start}(a_1) \wedge \text{end}(a_0) = \text{end}(a_1)), \\[0.5em]
\text{(Xf.3)} & a_0[x_{b\_\text{decision}} = \top_{\text{high}}] \implies \exists a_1[x_b = \top]. \\ 
& \qquad (\text{start}(a_0) = \text{start}(a_1) \wedge \text{end}(a_0) = \text{end}(a_1)), \\[0.5em]
\text{(Xf.4)} & a_0[x_{b\_\text{decision}} = \top_{\text{low}}] \implies \exists a_1[x_b = \top]. \\ 
& \qquad (\text{start}(a_0) = \text{start}(a_1) \wedge \text{end}(a_0) = \text{end}(a_1)), \\[0.5em]
\text{(Xf.5)} & a_0[x_{b\_\text{decision}} = \top_{\text{high}}] \implies \exists a_1[x_{b^{\text{high}}} = \top]. \\ 
& \qquad (\text{start}(a_0) = \text{start}(a_1) \wedge \text{end}(a_0) = \text{end}(a_1)), \\[0.5em]
\text{(Xf.6)} & a_0[x_{b\_\text{decision}} = \top_{\text{low}}] \implies \exists a_1[x_{b^{\text{low}}} = \top]. \\ 
& \qquad (\text{start}(a_0) = \text{start}(a_1) \wedge \text{end}(a_0) = \text{end}(a_1))
\end{array} \right\}$$

$$\mathcal{R}_{\text{backward}}(b) = \left\{ \begin{array}{ll}
\text{(Xb.1)} & a_0[x_{b^{\text{high}}} = \top] \implies \exists a_1[x_{b\_\text{decision}} = \top_{\text{high}}]. \\ 
& \qquad (\text{start}(a_0) = \text{start}(a_1) \wedge \text{end}(a_0) = \text{end}(a_1)), \\[0.5em]
\text{(Xb.2)} & a_0[x_{b^{\text{low}}} = \top] \implies \exists a_1[x_{b\_\text{decision}} = \top_{\text{low}}]. \\ 
& \qquad (\text{start}(a_0) = \text{start}(a_1) \wedge \text{end}(a_0) = \text{end}(a_1))
\end{array} \right\}$$

\noindent The forward rules establish the decision-driven behavior where the decision variable and main block are synchronized in their disabled states (Xf.1-2), and where decision values trigger both the main block (Xf.3-4) and the corresponding child branches (Xf.5-6). The backward rules ensure that child block activations trigger their respective decision values (Xb.1-2). The values on the decision variable are mutually exclusive by definition since they reside on the same state variable, which automatically ensures that child blocks are disabled when the decision is disabled.

This exclusive choice behavior can be observed in both Figures~\ref{fig:critical-path-timeline} and~\ref{fig:non-critical-loop-timeline}, where the XOR block $x_{b_3}$ implements the critical classification decision. In the critical path scenario, the decision variable $x_{b_3\_\text{dec}}$ takes value $\top_{\text{high}}$ to select the critical path ($x_{b_4}$) with emergency procedures, while in the non-critical scenario it takes value $\top_{\text{low}}$ to select the non-critical path ($x_{b_7}$) with standard care and monitoring. The XOR block encoding introduces no new types of Allen relations beyond those discussed for FLOW blocks, as all rules implement equality relations that ensure perfect synchronization between the decision variable, main block, and selected child branch.

This concludes our encoding of BPMN constructs into the eager fragment. Returning to the comparison with \cite{DBLP:conf/time/CombiOS19}, we note that while limited to the interval relations discussed in Section~\ref{sec:allen}, we can also express constraints in the style of \cite{DBLP:conf/time/CombiOS19} through eager rules by exploiting the compositionality of timelines. This means that plans may always be enriched through additional constraints implemented via synchronization rules that involve the newly introduced timelines and the existing ones, as we will show in the following. As an oversimplified example, consider adding a timeline for patient condition that transitions between unstable and stable states, where once stability is reached it persists: $T_{\text{condition}}(\text{unstable}) = \{\text{stable}, \text{unstable}\}$ and $T_{\text{condition}}(\text{stable}) = \{\text{stable}\}$. We can then specify that the critical path decision always starts with an unstable patient condition through the rule:
$$a_0[x_{b_4} = \top] \implies \exists a_1[x_{\text{condition}} = \text{unstable}]. (\text{start}(a_0) = \text{start}(a_1) \wedge \text{end}(a_1) \leq \text{end}(a_0))$$
This rule is activated only in case of critical path execution. In contrast, we place no constraints on the non-critical path (patients may be unstable but only mildly so). In such situations, we contemplate the case that the monitoring loop, activated only in case of non-critical execution, may cycle through stable conditions to verify that the patient remains stable, and finally we can require that discharge always begins with a stable patient condition through the rule:
$$a_0[x_{b_{16}} = \top] \implies \exists a_1[x_{\text{condition}} = \text{stable}]. (\text{start}(a_0) = \text{start}(a_1) \wedge \text{end}(a_0) = \text{end}(a_1))$$
This second rule is eventually activated for all patients, regardless of their treatment path. Such domain-specific constraints demonstrate how the eager fragment can still express some interesting temporal constraints beyond pure control flow through compositional timeline integration.

\subsection{Scope, extensibility, and limitations of the encoding}

The BPMN encoding presented above is intended as an illustrative case study 
demonstrating the expressiveness of the eager fragment on the major control 
flow patterns. The BPMN constructs not explicitly covered can, in principle, 
be encoded within our formalism; however, several limitations and drawbacks 
should be taken into account, as they may significantly affect the size of 
the resulting encoding and, consequently, the efficiency of the subsequent 
decidability goal, whether plan existence, strategy synthesis, or others.

\paragraph{SESE decomposition.} The efficiency and succinctness of our encoding relies heavily on the SESE 
property of the input BPMN diagram. Since disjunctions are not allowed in 
eager rules, the tree structure of the SESE decomposition is essential: it 
guarantees that, upon completion of the current block, the next block to be 
executed can be identified with certainty, whether it is the next sibling 
in a flow or the next sibling of the parent block, without requiring any 
nondeterministic guessing. Well-structuredness is widely regarded as a 
desirable design principle in the BPMN community, as it promotes readability, 
reduces modeling errors, and facilitates formal 
analysis~\cite{PolyvyanyyGBD12}. Moreover, in most practical cases, 
unstructured diagrams can be transformed into equivalent SESE ones by 
duplicating components: this has been shown to be always possible for acyclic 
process models~\cite{PolyvyanyyGBD12}, and partial results exist for cyclic 
models as well~\cite{KoehlerH04,ChoiKJZ15,PrinzCH22}, although some 
pathological loop constructions are known to resist structuring. All in all, 
the SESE requirement is not restrictive in practice; however, when 
restructuring of a non-SESE diagram is needed, the transformation may 
introduce additional components, resulting in a larger parse tree and, 
consequently, a higher number of state variables. Since $|\mathcal{SV}|$ is 
the parameter that governs the complexity of the resulting planning problem, 
this duplication directly impacts the efficiency.
\paragraph{Swimlanes and message passing.}
Swimlanes and pools are naturally accommodated by our encoding: they simply 
partition the state variables into organizational groups, resulting in a 
collection of independent SESE trees that do not affect the encoding in any 
way. Similarly, message passing constructs --- such as sending and receiving 
messages, signals, or events --- can be handled by introducing a dedicated 
timeline for each communication channel. Such a timeline behaves like a task 
and is shared between the two SESE trees of the communicating pools. The flow 
constraints on each side determine when the respective pool is ready to send 
or receive, and synchronization between the two pools happens in a natural 
and elegant way, practically for free.
\paragraph{Clock events and duration constraints.}
Although our framework is qualitative, clock events and duration constraints 
can be encoded within the eager fragment, at the cost of a more involved 
construction. We briefly sketch the idea for clock events; duration 
constraints for tasks can be handled similarly.
First, a \emph{global clock} timeline $x_{\mathit{clk}}$ is introduced, with a 
single value $\mathit{tick}$ (i.e., $V_{x_{\mathit{clk}}} = \{\mathit{tick}\}$). 
Every token on every other timeline is then required to be \emph{started by} a 
clock tick, corresponding to row~11 of Table~\ref{tab:allen-eagerness} in 
Section~\ref{sec:allen}, so that any value change on any timeline is 
aligned with a fresh tick of the global clock.
For each clock event of duration $d$, two dedicated timelines are introduced. 
The first, the \emph{counter} timeline, has domain 
$\{\bot, 0, 1, \ldots, d\}$ with value transition function 
$T(\bot) = \{0\}$, $T(i) = \{i{+}1\}$ for $0 \leq i < d$, and 
$T(d) = \{\bot\}$, so that the counter cycles through exactly $d$ ticks 
before resetting. The second, the \emph{envelope} timeline, alternates 
between two values: one that spans the entire $\bot$ phase and one that 
spans the full counting phase from $0$ to $d$, providing a coarser 
view of the timer's active and inactive periods.
It should be noted that this construction encodes durations in 
\emph{unary}: a clock event of duration $d$ requires $d{+}2$ values in the 
domain of the counter timeline, so the size of the encoding grows linearly 
with the magnitude of the max durations expressed in the process.
A more compact binary encoding of clocks would not be feasible within the 
eager fragment, as it would necessarily require \emph{disjunctions} to 
express the bit-manipulation logic of binary counting, and 
\emph{equality constraints between non-trigger existential tokens} to 
synchronize the individual bit timelines, the latter corresponding to 
row~21 of Table~\ref{tab:allen-eagerness}, which is ambiguous. Both 
features are incompatible with eager rules.
Another key observation is that, due to the qualitative nature of the 
synchronization rules, individual ticks may have different actual durations; 
however, since all tokens across all timelines are forced to respect the 
same global clock, a successful plan exists in the qualitative encoding if 
and only if a successful plan exists in the intended real-time 
interpretation.
\paragraph{Gateways.}
The most notable BPMN gateway not covered by our encoding is the 
\emph{inclusive OR gateway}, which allows one or more branches to be 
activated depending on runtime conditions. A binary inclusive OR with 
branches $A$ and $B$ can in principle be encoded as a three-way XOR over 
the activation patterns $\{A\}$, $\{B\}$, and $\{A, B\}$, thus avoiding 
disjunctions. However, this encoding requires duplicating the entire 
downstream process subtree for each activation pattern, with all child 
timelines in the duplicated regions renamed to fresh state variables. 
Since $n$-ary inclusive OR gateways can always be decomposed into a tree 
of binary ones, the critical cost factor is the nesting depth: $k$ nested 
binary inclusive OR gateways result in $3^k$ copies of the innermost 
subtree, yielding an encoding that is exponential in the nesting depth.
The \emph{event-based gateway} (deferred choice), where the routing 
decision depends on which among several competing events occurs first, 
can instead be treated as an XOR gateway: the decision is resolved by 
the plan itself, which selects the winning event. In the presence of 
clock events, the encoding becomes more involved due to the need for 
the timer machinery described above, but this additional complexity does 
not affect the size of the plan in the way that inclusive OR gateways do.\\

In summary, the eager fragment can accommodate a broader range of BPMN 
constructs than those explicitly covered in our case study. Swimlanes, 
message passing, and event-based gateways are handled naturally and 
essentially for free, without increasing the complexity of the encoding. 
Clock events and duration constraints can be encoded at the cost of a 
unary representation of durations and a more involved rule structure, 
but without departing from the eager fragment. The main limitations 
arise from inclusive OR gateways, whose encoding requires an exponential 
blowup in the nesting depth, and from unstructured (non-SESE) diagrams, 
whose restructuring, when possible, may introduce additional 
components. In both cases, the resulting increase in the number of state 
variables directly impacts the complexity of the subsequent decidability 
tasks, as $|\mathcal{SV}|$ is the governing parameter of our 
construction.




\section{A maximal subset of Allen's relations} \label{sec:allen}


Allen's interval algebra is a formalism for temporal reasoning introduced in~\cite{Allen83}.
It identifies all possible relations between pairs of time intervals over a linear order and specifies a machinery to reason about them.
%
In this section, we isolate the maximal subset of Allen's relations captured by
the eager fragment of qualitative timeline-based planning. To this end, we show
how to map Allen's relations over tokens in terms of their endpoints, that is,
as conjunctions of atoms over terms $\tokstart(a), \tokstart(b), \tokend(a),
\tokend(b)$, for token names $a$ and $b$.
%
%
Then, we identify the relation encodings that can be expressed by the eager
fragment, according to \autoref{def:eager:rule}.
%
%
Let $a, b \in \toknames$ and $t_1 = t_2$ be an abbreviation for $t_1 \before t_2
\wedge t_2 \before t_1$, for all pairs of terms $t_1,t_2$.
\begin{itemize}
\item$a \ibefore b$ ($b \iafter a$) can be defined as $\tokend(a) \before* \tokstart(b)$.

\item$a \meets b$ ($b \ismet a$) can be defined as $\tokend(a) = \tokstart(b)$.

\item$a \ends b$ ($b \isended a$) can be defined as $\tokstart(b) \before* \tokstart(a) \land \tokend(a) = \tokend(b)$.

\item$a \starts b$ ($b \isstarted a$) can be defined as $\tokstart(a) = \tokstart(b) \land \tokend(a) \before* \tokend(b)$.

\item$a \overlaps b$ ($b \isoverlapped a$) can be defined as $\tokstart(a) \before* \tokstart(b) \land \tokstart(b) \before* \tokend(a) \land \tokend(a) \before* \tokend(b)$.

\item$a \during b$ ($b \icontains a$) can be defined as $\tokstart(b) \before* \tokstart(a) \land \tokend(a) \before* \tokend(b)$.

\item$a \isequal b$ can be defined as $\tokstart(a) = \tokstart(b) \land \tokend(a) = \tokend(b)$.
\end{itemize}

It is not difficult to see that, if one of the tokens, say $a$, is the trigger
token, then the encodings not complying with the definition of eager rule
(\autoref{def:eager:rule}) are the ones for Allen's relations $a \ends b$,
$a \overlaps b$, $a \isoverlapped b$, and $a \during b$ (see also
\autoref{tab:allen-eagerness}).
Thus, the maximal subset of Allen's relations that can be captured by an
instance of the eager fragment of the timeline-based planning problem consists
of relations $a \ibefore b$, $a \iafter b$, $a \meets b$, $a \ismet b$,
$a \isended b$, $a \starts b$,
$a \isstarted b$, $a \icontains b$, and $a \isequal b$.
As an example, consider the constraint $a \overlaps b$ and let $\rulebody = \{
\tokstart(a) \before* \tokstart(b), \tokstart(b) \before* \tokend(a),
\tokend(a) \before* \tokend(b) \}$ be its encoding.
Clearly, the transitive closure $\hat{\rulebody}$ of \rulebody
(cf.~\autoref{sec:fragment}) includes also $\tokstart(b) \before \tokend(a)$,
but it does not include any of $\tokstart(b) \before \tokstart(a)$, $\tokend(a)
\before \tokstart(b)$, and $\tokend(b) \before \tokend(a)$, implying that token
name $b$ is left-ambiguous
(\autoref{def:eager:rule}{\it (\ref{item:left-ambiguous})}).
Moreover, we have that $\hat{\rulebody}$ includes $\tokend(a) \before
\tokend(b)$ but, as already pointed out, it does not include $\tokend(a) \before
\tokstart(b)$, which means that $b$ is right-ambiguous as well
(\autoref{def:eager:rule}{\it (\ref{item:right-ambiguous})}).
Therefore, $b$ is ambiguous (since it is not the trigger token --
see \autoref{def:eager:rule}{\it
  (\ref{item:ambiguous})}), and thus $a \overlaps b$ is not captured by
an eager rule.
A similar argument can be used for $a \ends b$, $a \isoverlapped b$, and
$a \during b$.
If, instead, none of the tokens is a trigger token, then the only Allen's
relations that can be captured by eager rules are are $\ibefore$, $\iafter$,
$\meets$, and $\ismet$.


\autoref{tab:allen-eagerness} provide a complete picture of the status of the
Allen's interval relations with respect to the property of being expressible by
means of eager rules.
As a matter of fact, the table only includes Allen's relations $\ibefore$,
$\meets$, $\ends$, $\starts$, $\overlaps$, $\during$, and $=$.
The status of the remaining relations ($\iafter$, $\ismet$, $\isended$,
$\isstarted$, $\isoverlapped$, $\icontains$) can be easily derived, as each is
the inverse of one of the listed relations.

Every Allen relation can be expressed as a conjunction of atoms, without using
disjunction.
Therefore, being expressible by means of an eager rule amounts to being
expressible by means of an unambiguous one (\cf~\autoref{def:eager:rule}).
Moreover, according to~\autoref{def:eager:rule}, the property of being
unambiguous depends on the role (being or not a trigger token) of the token
names involved in the rule.
Therefore, every relation, which establishes constraints between two token names
$a$ and $b$, is considered with respect to three different scenarios, according
to which token name is the trigger one (column ``Trigger token''):
\begin{enumerate*}[label={\it (\roman*)}]
\item $a$ is the trigger token,
\item $b$ is the trigger token,
\item neither $a$ nor $b$ is the trigger token.
\end{enumerate*}
For example, the first of the three lines for Allen relation $\ends$ (value
``$a$'' in the column ``Trigger token'')
%
%
depicts the status of the corresponding synchronization rule, which constrains
token name $a$ to be a strict suffix of token name $b$, when $a$ is the trigger
token:
\begin{align*}
  a[x_a=v_a]\implies \exists b[x_b=v_b]. \ \tokstart(b) \before* \tokstart(a)
  \wedge \tokend(a) = \tokend(b).
\end{align*}
In this case, being $a$ the trigger token, it is trivially not ambiguous.
This is expressed in the table by the character ``$-$''.
The table also shows, in the columns relevant for token $b$, that $b$ is
left-ambiguous, right-ambiguous, and ambiguous.
Since at least one among $a$ and $b$ is ambiguous, the entire rule is ambiguous,
as indicated by the string ``yes'' in the last column (meaning that the rule
does not belong to the eager fragment of qualitative timeline-based planning).
The next line of the table (value ``$b$'' in the column ``Trigger token'')
considers the rule obtained from the above one, using $b$ as trigger token:
\begin{align*}
  b[x_b=v_b]\implies \exists a[x_a=v_a]. \ \tokstart(b) \before* \tokstart(a)
  \wedge \tokend(a) = \tokend(b).
\end{align*}
In this case, $b$ is trivially not ambiguous.
Token name $a$ is not ambiguous as well, since it is not left-ambiguous (even
though it is right-ambiguous).
Since neither $a$ nor $b$ is ambiguous, the entire rule is unambiguous, as
indicated by the string ``no'' in the last column (and thus it belongs to the
eager fragment of qualitative timeline-based planning).
Furthermore, the next line of the table (value ``none'' in the column ``Trigger
token'') considers the triggerless rule for the same Allen relation:
\begin{align*}
  \top \implies \exists a[x_a=v_a] b[x_b=v_b]. \ \tokstart(b) \before*
  \tokstart(a) \wedge \tokend(a) = \tokend(b).
\end{align*}
As shown in the table, $a$ is not ambiguous (as it is not left-ambiguous), while
$b$ is ambiguous (because it is not the trigger token, and it is both left- and
right-ambiguous); therefore, the
entire rule is ambiguous (and thus it does not belong to the eager fragment of
qualitative timeline-based planning).
As a last example, consider the following rule for Allen relation $\isended$,
constraining token name $b$ to be a strict suffix of token name $a$, when $a$ is
the trigger token:
\begin{align*}
  a[x_a=v_a]\implies \exists b[x_b=v_b]. \ \tokstart(a) \before* \tokstart(b)
  \wedge \tokend(a) = \tokend(b).
\end{align*}
Saying $a \isended b$, with the first token name being the trigger one, amounts
to saying $b \ends a$, with the second token name being the trigger one;
therefore, the rule is not ambiguous, since it corresponds to the second of the
three lines in the table for Allen relation $\ends$, the one where the second
token name is the trigger token.
More precisely, $a$ is not ambiguous since it is the trigger token, and $b$
is not ambiguous since it is not left-ambiguous (it is right-ambiguous, though).

Finally, we observe that \emph{reflexive} variants of the Allen's relations can
also be considered.
These variants are obtained by replacing $\before*$ with $\before$ in each of
the mappings (of Allen's relations in terms of conjunctions of atoms) given
above.
The status of the reflexive variants does not change, for any Allen relation.

\begin{table}[htbp]
\centering
\caption{Eagerness analysis for Allen's relations}
\label{tab:allen-eagerness}
\resizebox{\textwidth}{!}{%
\begin{tabular}{|c|l|c||c|c|c||c|c|c||c|}
\hline
\multirow{2}{*}{\textbf{\#}} & \multirow{2}{*}{\textbf{Allen Relation}} & \multirow{2}{*}{\textbf{\parbox[t]{20mm}{\centering Trigger \\token}}} &
\multicolumn{3}{c||}{\textbf{Token $a$}} & \multicolumn{3}{c||}{\textbf{Token $b$}} & \multirow{2}{*}{\textbf{\parbox[t]{25mm}{\centering Overall ambiguous}}} \\
\cline{4-9}
& & & \textbf{Left-amb} & \textbf{Right-amb} & \textbf{Ambiguous} & \textbf{Left-amb} & \textbf{Right-amb} & \textbf{Ambiguous} & \\
\hline
1 & $a \ibefore b$ & $a$ & - & - & - & no & no & no & no \\
2 & $a \ibefore b$ & $b$ & no & no & no & - & - & - & no \\
3 & $a \ibefore b$ & none & no & no & no & no & no & no & no \\
\hline
4 & $a \meets b$ & $a$ & - & - & - & no & no & no & no \\
5 & $a \meets b$ & $b$ & no & yes & no & - & - & - & no \\
6 & $a \meets b$ & none & no & yes & no & yes & no & no & no \\
\hline
7 & $a \ends b$ & $a$ & - & - & - & yes & yes & yes & yes \\
8 & $a \ends b$ & $b$ & no & yes & no & - & - & - & no \\
9 & $a \ends b$ & none & no & yes & no & yes & yes & yes & yes \\
\hline
10 & $a \starts b$ & $a$ & - & - & - & no & yes & no & no \\
11 & $a \starts b$ & $b$ & no & no & no & - & - & - & no \\
12 & $a \starts b$ & none & yes & no & no & yes & yes & yes & yes \\
\hline
13 & $a \overlaps b$ & $a$ & - & - & - & yes & yes & yes & yes \\
14 & $a \overlaps b$ & $b$ & yes & yes & yes & - & - & - & yes \\
15 & $a \overlaps b$ & none & yes & yes & yes & yes & yes & yes & yes \\
\hline
16 & $a \during b$ & $a$ & - & - & - & yes & yes & yes & yes \\
17 & $a \during b$ & $b$ & no & no & no & - & - & - & no \\
18 & $a \during b$ & none & no & no & no & yes & yes & yes & yes \\
\hline
19 & $a = b$ & $a$ & - & - & - & no & yes & no & no \\
20 & $a = b$ & $b$ & no & yes & no & - & - & - & no \\
21 & $a = b$ & none & yes & yes & yes & yes & yes & yes & yes \\

\hline
\end{tabular}%
}
\end{table}

\section{Conclusions}
\label{sec:conclusion}

In this paper, 
we identified a meaningful fragment of qualitative timeline-based planning (the \emph{eager} fragment)
whose solutions
can be recognized by DFAs of singly exponential size. Specifically, we
identified restrictions on the allowed synchronization rules, which we
called \emph{eager} rules, for which we showed how to build the
corresponding deterministic automaton of exponential size, that
can then be directly exploited to synthesize strategies.
%
We also showed that it is not possible to encode the solution plans for
qualitative timeline-based planning problems (when non-eager rules are allowed)
using deterministic finite automata of exponential size.
Moreover, we demonstrated the practical relevance of the eager fragment
 through a comprehensive BPMN case study that systematically translates all major control
  flow patterns into eager synchronization rules. Last but not least,
we isolated a maximal subset of Allen's relations captured by
the eager fragment.

Whether the eager fragment of qualitative timeline-based planning is maximal or
not is an open question currently under investigation.
Further research directions include
%
a parameterized complexity analysis over the number of synchronization rules and
a characterization in terms of temporal logics, like the one
in~\cite{della2017bounded}. 
Another interesting direction for future work concerns the extension of our approach to dense temporal domains. Timeline-based planning over dense time has been studied by Bozzelli et al. \cite{BozzelliMMP18a}, who showed that the problem is undecidable in full generality, even with a single state variable. Our work deliberately focuses on the qualitative fragment, which abstracts away metric temporal information and retains only the relative order of events. The DFA construction at the core of our approach fundamentally relies on this discrete, qualitative setting, where each letter of the input word encodes a single time step and the automaton tracks rule satisfaction incrementally. Lifting the eager fragment to dense time would require fundamentally different techniques, and we regard this as a challenging and rewarding avenue for future investigation.



\appendix

\section{Proof of \autoref{lem:automatonSolutions}}
\label{app:proofSolutionAutomaton}

In this section, we proof \autoref{lem:automatonSolutions} from
\autoref{sec:dfaForSolutions}.
We restate the lemma for the reader's convenience.

\lemAutomatonSolutions*

We first proof the exponential upper bound for the automata size, which is
easier.
Then, we focus on the main claim of the statement, concerning the correctness of
the proposed automata construction.

\subsection{Exponential upper bound}
\label{app:proofSolutionAutomaton:upperbound}

\begin{lemma}
  \label{lem:automatonSolutions:upperbound}
  Let $P = (\SV, S)$ be an eager qualitative timeline-based planning problem.
  The size of $\automAP$ is at most exponential in the size of $P$.
\end{lemma}
\begin{proof}
%
%
  Let $k$ be the largest number of token names in a rule of $P$.
%
%
  Thanks to the linearity condition enjoyed by states of $\automAP$
  (cf. \autoref{def:linearity}), it is not difficult to convince oneself that
  the number of different viewpoints for the same rule in a state $q \in Q_P$ to
  be at most $k$.
  Thus, each state in $Q_P$ contains at most $|S| \times k$ different
  viewpoints.
%

  Therefore, the size of $Q_P$ is at most $|\viewpointsetP|^{(|S| \times k)}$.
  Clearly, $|S| \times k$ is at most polynomial in the size of $P$.
  Since $|\viewpointsetP| \leq \sum_{\Rule \in S} |\viewpointsetR|$ and, as
  already pointed out, $|\viewpointsetR|$ is at most exponential in the size of
  $P$, we can conclude that the size of $Q_P$ is at most exponential in the size
  of $P$.
\end{proof}

\subsection{Correctness of the automata construction}
\label{app:proofSolutionAutomaton:correctness}


\begin{lemma}
  \label{lem:automatonSolutions:onlyif}
  Let $P = (\SV, S)$ be an eager qualitative timeline-based planning problem.
  Each finite word over $\Sigma_\SV$ that encodes a plan over \SV is accepted by
  \automAP if and only if it encodes a solution plan for $P$.
\end{lemma}

First, we define a notion of witness that characterizes solution plans for (not
necessarily eager) disjunction-free qualitative timeline-based planning problems
(cf. \autoref{def:planningProblem}).
Then, in Section~\ref{app:automatonSolutions:canonicalWitness} we introduce a
particular canonical witness, which characterizes solution plans for eager
qualitative timeline-based planning problems.
The thesis follows from the fact that the automata that we defined in
\autoref{sec:dfaForPlans} and \autoref{sec:dfaForSolutions} recognize such
canonical witnesses.

\subsubsection{Matching and relaxed witnesses for (not necessarily eager)
  planning problems}
\label{app:automatonSolutions:matchingAndRelaxedWitness}

Let $P = (\SV, S)$ be a (not necessarily eager) qualitative timeline-based
planning problem, $\Rule \in S$ a rule with trigger $a_0[x_0=v_0]$, and
$\lambda=\seq{\sigma_0,\ldots,\sigma_{|\lambda|-1}}$ a finite word over
$\Sigma_\SV$.
We denote by \triggerslambdaR the set of pairs of natural numbers that
correspond to tokens matching the trigger $a_0$ of \Rule, i.e., tokens for
timeline $x_0$ of value $v_0$.
Formally, for every $m,n \in \mathbb N$, with $m < n$, we have that $(m,n) \in
\triggerslambdaR$ if and only if
\begin{itemize}
\item $\mathit{start}(x_0,v_0) \in \eventsOneBasic{\sigma_m}$,
\item $\mathit{end}(x_0,v_0) \in \eventsOneBasic{\sigma_n}$, and
\item $\{ \mathit{start}(x_0,v_0), \mathit{end}(x_0,v_0) \} \cap
  \eventsOneBasic{\sigma_i} = \emptyset$ for all $i \in \mathbb N$ with $m < i <
  n$.
\end{itemize}
For the sake of a uniform treatment of triggerless and non-triggerless rules, we
stipulate that $\triggerslambdaR = \{ \top \}$ for every triggerless rule \Rule and
$\lambda \in \Sigma_\SV^*$.

\begin{definition}[trigger-level witness]
  \label{def:triggerWitness}
  Let $P = (\SV, S)$ be a (not necessarily eager) disjunction-free qualitative
  timeline-based planning problem, $\Rule \in S$, and $\lambda \in
  \Sigma_\SV^*$.
  Moreover, let \rulebody be the only clause of \Rule and $T$ be the set of
  terms (of the form $\tokstart(a)$ and $\tokend(a)$) appearing in \Rule.
  For every $\trgr \in \triggerslambdaR$, a \emph{(trigger-level) witness} of
  \trgr in $\lambda$ is a function $f : T \rightarrow \mathbb N$ such that:
  \begin{enumerate}[label={\em (\roman*)}]
  \item \label{item:fcompatible}
    $f$ is \emph{compatible} with $\lambda$, that is, if
    $\lambda=\seq{\sigma_0,\ldots,\sigma_{|\lambda|-1}}$ and $a[x=v]$ occurs in
    \Rule, then
    \begin{itemize}
    \item $\tokstart(a) \in T$ implies $\mathit{start}(x,v) \in
      \eventsOneBasic{\sigma_{f(\tokstart(a))}}$ and
    \item $\tokend(a) \in T$ implies $\mathit{end}(x,v) \in
      \eventsOneBasic{\sigma_{f(\tokend(a))}}$
    \end{itemize}
  \item \label{item:fSatC}
    $f$ \emph{satisfies} \rulebody, that is, replacing in \rulebody every
    occurrence of every term $t \in T$ with $f(t)$ makes all atoms in \rulebody
    true, and
  \item \label{item:trigger}
    if $tr = (m,n) \in \mathbb N \times \mathbb N$, then $f(\tokstart(a_0)) = m$
    and $f(\tokend(a_0)) = n$, where $a_0[x_0=v_0]$ is the trigger of \Rule
    (this condition does not apply if $tr = \top \not\in \mathbb N \times
    \mathbb N$, i.e, when \Rule is triggerless).
  \end{enumerate}
\end{definition}

\begin{definition}[matching and relaxed (trigger-level) witness]
  \label{def:matchingAndRelaxedWitness}
  Let $P = (\SV, S)$ be a (not necessarily eager) disjunction-free qualitative
  timeline-based planning problem, $\Rule \in S$,
  $\lambda=\seq{\sigma_0,\ldots,\sigma_{|\lambda|-1}} \in \Sigma_\SV^*$, $\trgr
  \in \triggerslambdaR$, and $f : T \rightarrow \mathbb N$ a witness of \trgr in
  $\lambda$, and suppose that $a[x=v]$ occurs in \Rule.

  Then, token name $a$ is \emph{matching} (in $f$ and $\lambda$) if and only if
  $\{ \tokstart(a), \tokend(a) \} \not\subseteq T$ or $\{ \mathit{start}(x,v),
  \mathit{end}(x,v) \} \cap \eventsOneBasic{\sigma_i} = \emptyset$ for all $i
  \in \mathbb N$ with $f(\tokstart(a)) < i < f(\tokend(a))$.
  Note that if $a$ is the trigger token of \Rule, then it is matching.

  If $a$ is not matching, then we define $\matchStartlambdafa = \max_i \{ i \in
  \mathbb N \mid f(\tokstart(a)) < i < f(\tokend(a)) \text{ and }
  \mathit{start}(x,v) \in \eventsOneBasic{\sigma_i} \}$.
  Observe that if $\lambda$ encodes a plan over \SV, then \matchStartFunlambdaf
  is well defined and $f(\tokstart(a)) < \matchStartlambdafa$.

  Moreover, witness $f$ is a \emph{matching witness} of \trgr in $\lambda$ if
  and only if all token names appearing in $T$ are matching.
  Instead, it is a \emph{relaxed witness} of \trgr in $\lambda$ if and only if
  all non-matching token names $a$ appearing in $T$
%
%
  violate \autoref{item:left-ambiguous-ii} of \autoref{def:eager:rule}, that is,
  for all terms $t \not\in \{ \tokstart(a), \tokend(a) \}$ we have that
  $\tokstart(a) \before t \in \hat \rulebody$ implies $\tokend(a) \before t \in
  \hat \rulebody$.
\end{definition}

We show in the following how to obtain matching trigger-level witnesses from
relaxed ones.
\begin{definition}
  \label{def:fromRelaxedToMatchingWitness}
  Let $P = (\SV, S)$ be a (not necessarily eager) disjunction-free qualitative
  timeline-based planning problem and $f : T \rightarrow \mathbb N$ a relaxed
  witness of \trgr in $\lambda$, for some finite word $\lambda$ over
  $\Sigma_\SV$ that encodes a plan over \SV and $\trgr \in \triggerslambdaR$,
  with $\Rule \in S$.
  We define the \emph{matching version} of $f$, denoted $\hat f$, as the
  function from $T$ to $\mathbb N$ such that:
  \begin{itemize}
  \item $\hat f(\tokend(a)) = f(\tokend(a))$ for all $\tokend(a) \in T$,
  \item $\hat f(\tokstart(a)) = f(\tokstart(a))$ for all $\tokstart(a) \in T$
    where $a$ is a matching token name,
  \item $\hat f(\tokstart(a)) = \matchStartlambdafa$ for all $\tokstart(a) \in
    T$ where $a$ is not a matching token name.
  \end{itemize}

\end{definition}

\begin{lemma}
  \label{lem:fromRelaxedToMatchingWitness}
  Let $P = (\SV, S)$ be a (not necessarily eager) disjunction-free qualitative
  timeline-based planning problem and $f : T \rightarrow \mathbb N$ a relaxed
  witness of \trgr in $\lambda$, for some finite word $\lambda$ over
  $\Sigma_\SV$ that encodes a plan over \SV and $\trgr \in \triggerslambdaR$,
  with $\Rule \in S$.
  Then, $\hat f$ is a matching witness of \trgr in $\lambda$.

\end{lemma}
\begin{proof}
  We only need to show that $\hat f$ is a witness of \trgr in $\lambda$.
  Being a matching one follows from the definition of function
  \matchStartFunlambdaf.
  \autoref{item:fcompatible} and \autoref{item:trigger} of
  \autoref{def:triggerWitness} are clearly fulfilled.

  Thus we focus on \autoref{item:fSatC}.
  Let \rulebody be the only clause of \Rule and $\alpha \in \rulebody$ be $t_1
  \triangleleft t_2$ for some $t_1,t_2 \in T$ and $\triangleleft \in \{ \leq,
  <\}$.
  We show that $\hat f$ \emph{satisfies} $\alpha$, that is $\hat f(t_1)
  \triangleleft \hat f(t_2)$ holds true.
  By definition of $\hat f$, it holds that $f(t) \leq \hat f(t)$ for all $t \in
  T$.
  If $f(t_1) = \hat f(t_1)$, then we have $\hat f(t_1) = f(t_1) \triangleleft
  f(t_2) \leq \hat f(t_2)$, which implies $\hat f(t_1) \triangleleft \hat
  f(t_2)$.
  Thus, let us assume $f(t_1) < \hat f(t_1)$.
  Since $\hat f(\tokend(a)) = f(\tokend(a))$ for all token names $a$, it must be
  $t_1 = \tokstart(a)$ for some token name $a$; moreover, since $\hat
  f(\tokstart(a)) = f(\tokstart(a))$ for all $\tokstart(a) \in T$ where $a$ is a
  matching token name, we have that $a$ is not matching.
  If $t_2 = t_1$, then $\triangleleft = \leq$ (indeed, if it was $\triangleleft
  = <$, then we would have a contradiction with the hypothesis that $f$
  satisfies $t_1 < t_2$), and $\hat f(t_1) \leq \hat f(t_2)$ trivially holds.
  If $t_2 = \tokend(a)$, then, by definition of $\hat f$, it holds that $\hat
  f(t_1) = \hat f(\tokstart(a)) < \hat f(\tokend(a)) = \hat f(t_2)$, which
  implies $\hat f(t_1) < \hat f(t_2)$, and thus $\hat f(t_1) \triangleleft \hat
  f(t_2)$.
  Finally, let us consider the last case, with $t_2 \not\in \{ \tokstart(a),
  \tokend(a) \}$.
  Since $f$ is a relaxed witness and $a$ is not matching, for all terms $t
  \not\in \{ \tokstart(a), \tokend(a) \}$ we have that $\tokstart(a) \before t
  \in \hat \rulebody$ implies $\tokend(a) \before t \in \hat \rulebody$, which
  means that $\tokend(a) \before t_2 \in \hat \rulebody$.
  Because $f$ satisfies $\rulebody$, we have that $f(\tokend(a)) \leq f(t_2)$,
  and thus it holds that $\hat f(\tokstart(a)) < \hat f(\tokend(a)) =
  f(\tokend(a)) \leq f(t_2) \leq \hat f(t_2)$, which, in turn, implies $\hat
  f(\tokstart(a)) < \hat f(t_2)$.
  Hence, $\hat f(t_1) \triangleleft \hat f(t_2)$ holds true.
\end{proof}

Next, we lift the notion of witness in order to characterize solution plans for
(not necessarily eager) disjunction-free qualitative timeline-based planning
problem.

\begin{definition}[rule-level witness]
  Let $P = (\SV, S)$ be a (not necessarily eager) disjunction-free qualitative
  timeline-based planning problem, $\Rule \in S$, and $\lambda \in
  \Sigma_\SV^*$.
  A \emph{(rule-level) witness} of \Rule in $\lambda$ is a sequence $w$ of
  functions $\langle f^{\trgr} \rangle_{\trgr \in \triggerslambdaR}$, where
  $f^{\trgr}$ is a (trigger-level) witness of \trgr in $\lambda$ for all $\trgr
  \in \triggerslambdaR$.
  Moreover, $w$ is matching (resp., relaxed) if so are all its components.
\end{definition}

\begin{definition}[problem-level witness]
  Let $P = (\SV, S)$ be a (not necessarily eager) disjunction-free qualitative
  timeline-based planning problem and $\lambda \in \Sigma_\SV^*$.
  A \emph{(problem-level) witness} of $S$ in $\lambda$ is a sequence $\mathfrak
  w$ of functions $\langle w^{\Rule} \rangle_{\Rule \in S}$, where $w^{\Rule}$
  is a (rule-level) witness of \Rule in $\lambda$ for all $\Rule \in S$.
  Moreover, $\mathfrak w$ is matching (resp., relaxed) if so are all its
  components.
\end{definition}

A (trigger-/rule-/problem-level) matching witness is also a
(trigger-/rule-/problem-level) relaxed one, but the converse does not hold in
general.
Moreover, it is obvious that problem-level matching witnesses characterize
solution plans, also for non-eager instances of the planning problem.
\begin{remark}
  \label{rem:solutionIffWitness}
  Let $P = (\SV, S)$ be a (not necessarily eager) disjunction-free qualitative
  timeline-based planning problem and $\lambda$ a finite word over $\Sigma_\SV$
  that encodes a plan over \SV.
  Then, $\lambda$ encodes a solution plan for $P$ if and only if there is a
  matching witness of $S$ in $\lambda$.
\end{remark}

It is less obvious that relaxed witnesses characterize solution plans as well.
\begin{lemma}
  \label{lem:solutionIffRelaxedWitness}
  Let $P = (\SV, S)$ be a (not necessarily eager) disjunction-free qualitative
  timeline-based planning problem and $\lambda$ a finite word over $\Sigma_\SV$
  that encodes a plan over \SV.
  Then, $\lambda$ encodes a solution plan for $P$ if and only if there is a
  relaxed witness of $S$ in $\lambda$.
\end{lemma}
\begin{proof}
  Thanks to \autoref{rem:solutionIffWitness}, it is enough to show that if a
  relaxed witness of $S$ in $\lambda$ exists, then a matching one exists as
  well.
  Let $\mathfrak w = \langle w^\Rule \rangle_{\Rule \in S}$ be a (problem-level)
  relaxed witness of $S$ in $\lambda$, where $w^\Rule$ is a (rule-level) relaxed
  witness of \Rule in $\lambda$, for all $\Rule \in S$.
  Then, the sequence $\hat{w^\Rule}$, obtained from $w^\Rule$ by replacing each
  (trigger-level) witness $f$ with its matching version $\hat f$, is a
  (rule-level) matching witness of \Rule in $\lambda$, for all $\Rule \in S$.
  Finally, $\hat{\mathfrak w} = \langle \hat{w^\Rule} \rangle_{\Rule \in S}$ is
  a matching witness of $S$ in $\lambda$.
\end{proof}

\subsubsection{Canonical eager witnesses for eager planning problems}
\label{app:automatonSolutions:canonicalWitness}

For an eager synchronization rule \Rule, whose unique clause is \rulebody, and a
term $t$ occurring in it, we define sets \predeceqRt of \emph{predecessors} and
\predecstrRt of \emph{strict predecessors} of $t$:

\bigskip

{\centering

  $
  \begin{array}{l@{\hspace{1mm}}l}
    \predecstrRt
    & = \{ s \mid s \before* t \in \hat \rulebody \} \\[3mm]

    \predeceqRt
    & = \{ s \mid s \before t \in \hat \rulebody \text{ and } t \before
      s \not\in \hat \rulebody, \}
  \end{array}
  $

}

\bigskip

\noindent Clearly, if $s \before t \in \hat \rulebody$, then $\predecstr{R}{s}
\subseteq \predecstrRt$ and $\predeceq{R}{s} \subseteq \predeceqRt$; therefore,
$t_1 \equiv t_2 \in \hat \rulebody$ implies $\predecstr{R}{s} = \predecstrRt$
and $\predeceq{R}{s} = \predeceqRt$.

Now, let $P = (\SV, S)$ be an eager qualitative timeline-based planning problem,
$\Rule \in S$, and $\lambda=\seq{\sigma_0,\ldots,\sigma_{|\lambda|-1}} \in
\Sigma_\SV^*$.
Moreover, let \rulebody be the only clause of \Rule and $T$ be the set of terms
(of the form $\tokstart(a)$ and $\tokend(a)$) appearing in \Rule.
For every $\trgr \in \triggerslambdaR$, we define a function $f^\trgr_{\lambda}
: T \rightarrow \mathbb N$, and then we show that (since $P$ is eager) if there
is a relaxed witness of \trgr in $\lambda$, then $f^\trgr_{\lambda}$ is one such
witness.
To this end we need to define an auxiliary functions, namely $\lowFuntrlambda :
T \rightarrow \mathbb N$.
%
%
Functions $f^\trgr_{\lambda}$ and \lowFuntrlambda are defined by mutual
induction as follows.
Let $a_0$ be the trigger token of \Rule, unless \Rule is triggerless, and $t \in
T \setminus \{ \tokstart(a_0), \tokend(a_0) \}$.
\begin{itemize}
\item If \Rule is not triggerless (thus $\trgr \neq \top$)), let $\trgr = (m,n)
  \in \mathbb N \times \mathbb N$; then $f^\trgr_{\lambda}(\tokstart(a_0)) = m$
  and $f^\trgr_{\lambda}(\tokend(a_0)) = n$.
\item If $\predeceqRt \cup \predecstrRt = \emptyset$, then
  \begin{itemize}
  \item $\lowtrlambdat = 0$ and
  \item $f^{\trgr}_{\lambda}(t) = \min\{ x \in \mathbb N \mid
    \termToEvent{[t]_\equiv} \subseteq \eventsOnesigmax \}$.
  \end{itemize}
\item Otherwise (i.e., $\predeceqRt \cup \predecstrRt \neq \emptyset$)
  \begin{itemize}
  \item $\lowtrlambdat = \max( \{ f^\trgr_{\lambda}(s) \mid s \in \predeceqRt \}
    \cup \{ f^\trgr_{\lambda}(s)+1 \mid s \in \predecstrRt \})$ and
  \item $f^{\trgr}_{\lambda}(t) = \min\{ x \in \mathbb N \mid x \geq
    \lowtrlambdat \text{ and } \termToEvent{[t]_\equiv} \subseteq
    \eventsOnesigmax \}$.
  \end{itemize}
\end{itemize}

It is immediate to see that when $f^\trgr_\lambda$ is well defined (in
particular, it is total), it satisfies all atoms $t_1 \before t_2$ and $t_1
\before* t_2$ in \rulebody, as long as $t_2$ does not involve a trigger token,
and thus, in order to check whether it is a witness of \trgr in $\lambda$
(cf. \autoref{def:triggerWitness}), it suffices to verify atoms $t_1 \before
t_2$ and $t_1 \before* t_2$ in \rulebody, where $t_2$ involves a trigger token.
The other conditions of \autoref{def:triggerWitness} are clearly satisfied.

\begin{lemma}
  \label{lem:canonicWellDefined}
  Let $P = (\SV, S)$ be an eager qualitative timeline-based planning problem,
  $\Rule \in S$, $\lambda=\seq{\sigma_0,\ldots,\sigma_{|\lambda|-1}} \in
  \Sigma_\SV^*$, and $\trgr \in \triggerslambdaR$.
  If a witness $f : T \rightarrow \mathbb N$ of \trgr in $\lambda$ exists, then
  $f^{\trgr}_{\lambda} : T \rightarrow \mathbb N$ is well defined and such that
  $f^{\trgr}_{\lambda}(t) \leq f(t)$ for all $t \in T$.
\end{lemma}
\begin{proof}
  Let $t : T \rightarrow \mathbb N$ be a witness of \trgr in $\lambda$ and $a_0$
  be the trigger token of \Rule, unless \Rule is triggerless.
  If $\tokstart(a_0), \tokend(a_0) \in T$ (i.e., \Rule is not triggerless and
  thus $\trgr = (m,n)$ for some natural numbers $m$ and $n$), then
  $f^{\trgr}_{\lambda}(\tokstart(a_0)) = m = f(\tokstart(a_0))$ and
  $f^{\trgr}_{\lambda}(\tokend(a_0)) = n = f(\tokend(a_0))$.

  Let $t \in T \setminus \{ \tokstart(a_0), \tokend(a_0) \}$.
  Since $f$ is a witness of \trgr in $\lambda$, we have that
  \begin{equation}
    \label{eq:fImpliesEvents}
    f(t) = i \text{ implies } \termToEvent{[t]_\equiv} \subseteq
    \eventsOneBasic{\sigma_i}.
  \end{equation}
  We show by induction that
  \begin{enumerate*}[label={\em(\roman*)}]
  \item \label{it:lowLessThanf} \lowtrlambdat id defined and $\lowtrlambdat \leq
    f(t)$, and
  \item \label{it:canonicLessThanf} $f^{\trgr}_{\lambda}(t)$ is defined and
    $f^{\trgr}_{\lambda}(t) \leq f(t)$.
  \end{enumerate*}
  If $\predeceqRt \cup \predecstrRt = \emptyset$, then \ref{it:lowLessThanf}
  holds trivially, while \ref{it:canonicLessThanf} follows from
  \autoref{eq:fImpliesEvents}.
  If, instead, $\predeceqRt \cup \predecstrRt \neq \emptyset$, then
  \ref{it:canonicLessThanf} immediately follows from \ref{it:lowLessThanf} and
  \autoref{eq:fImpliesEvents}, while we proceed as follows to prove
  \ref{it:lowLessThanf}.
  By definition of \lowFuntrlambda, we have that $\lowtrlambdat =
  f^\trgr_{\lambda}(s)$ for some $s$ such that $s \before t \in \hat \rulebody$
  or $\lowtrlambdat = f^\trgr_{\lambda}(s)+1$ for some $s$ such that $s \before*
  t \in \hat \rulebody$.
  Observe that $f^\trgr_{\lambda}(s)$ is defined by induction, and thus
  \lowtrlambdat is defined, too.
  By induction, it also holds that $f^\trgr_{\lambda}(s) \leq f(s)$.
  In the former case ($\lowtrlambdat = f^\trgr_{\lambda}(s)$ and $s \before t
  \in \hat \rulebody$), we have that $f(s) \leq f(t)$, and thus it holds:

  {\centering

    $\lowtrlambdat = f^\trgr_{\lambda}(s) \leq f(s) \leq f(t)$;

  }

  \noindent hence \ref{it:lowLessThanf}.
  In the latter case ($\lowtrlambdat = f^\trgr_{\lambda}(s)+1$ and $s \before* t
  \in \hat \rulebody$), we have that $f(s) < f(t)$, and thus it holds:

  {\centering

    $\lowtrlambdat = f^\trgr_{\lambda}(s) + 1 \leq f(s) + 1 < f(t) + 1$;

  }

  \noindent hence \ref{it:lowLessThanf}.
\end{proof}
Thanks to \autoref{lem:canonicWellDefined}, if a witness of \trgr in $\lambda$
exists, then $f^\trgr_{\lambda}$ is a witness of \trgr in $\lambda$, too.
If, in addition, it is a relaxed one
(cf. \autoref{def:matchingAndRelaxedWitness}), then we call it the
\emph{canonical eager witness} of \trgr in $\lambda$.

\begin{lemma}
  \label{lem:canonicalEagerWitness}
  Let $P = (\SV, S)$ be an eager qualitative timeline-based planning problem,
  $\Rule \in S$, $\lambda$ a finite word over $\Sigma_\SV$ that encodes a plan
  over \SV, and $\trgr \in \triggerslambdaR$.
  Then, there is a relaxed witness of \trgr in $\lambda$ if and only if
  $f^\trgr_{\lambda} : T \rightarrow \mathbb N$ is the canonical eager witness of
  \trgr in $\lambda$.
\end{lemma}
\begin{proof}
  The ``if'' direction of the claim trivially holds, since the canonical eager
  witness of \trgr in $\lambda$ is, by definition, a relaxed one.
  Therefore, it suffices to prove that if a relaxed witness of \trgr in
  $\lambda$ exists, then $f^\trgr_{\lambda}$ is the canonical eager witness of
  \trgr in $\lambda$.
  As already observed, since a relaxed witness of \trgr in $\lambda$ exists,
  $f^\trgr_{\lambda}$ is indeed a witness of \trgr in $\lambda$.
  In the following, we show that all non-matching token names $a$ appearing in
  $T$ violate \autoref{item:left-ambiguous-ii} of \autoref{def:eager:rule}.
  Towards a contradiction, assume that there is a token name $a$ that is not
  matching in $f^\trgr_{\lambda}$ and $\lambda$ for which
  \autoref{item:left-ambiguous-ii} of \autoref{def:eager:rule} holds true.
  Notice that $a$ is matching in $f$ and $\lambda$, or $f$ would not be a
  relaxed witness.
  Therefore it must be $f^\trgr_{\lambda}(\tokstart(a)) < f(\tokstart(a))$:
  indeed, if it was $f^\trgr_{\lambda}(\tokstart(a)) = f(\tokstart(a))$, then it
  would be $f^\trgr_{\lambda}(\tokend(a)) > f(\tokend(a))$ (otherwise $a$ is
  matching in $f^\trgr_{\lambda}$ and $\lambda$), which is in contradiction with
  \autoref{lem:canonicWellDefined}.
  We show that $a$ satisfies also \autoref{item:left-ambiguous-i} and
  \autoref{item:right-ambiguous} of \autoref{def:eager:rule}.
  Since $a$ is not a matching token name, it is not a trigger token either
  (cf. \autoref{def:matchingAndRelaxedWitness}), and thus it is ambiguous, which
  means that \Rule, and thus $P$, are not eager (cf. \autoref{def:eager:rule}),
  which contradicts the hypothesis of $P$ being eager.

  First, we focus on \autoref{item:left-ambiguous-i} of
  \autoref{def:eager:rule}.
  If \Rule is triggerless, then the condition is vacuously verified.
  Otherwise, let $a_0$ be the trigger token of \Rule and let $\trgr = (m,n)$ for
  $m, n \in \mathbb N$.
  If $\tokstart(a) \equiv \tokstart(a_0) \in \hat \rulebody$, then
  $f^\trgr_{\lambda}(\tokstart(a)) = f^\trgr_{\lambda}(\tokstart(a_0)) = m =
  f(\tokstart(a_0)) = f(\tokstart(a))$, which is in contradiction with
  $f^\trgr_{\lambda}(\tokstart(a)) < f(\tokstart(a))$.
  If $\tokstart(a) \equiv \tokend(a_0) \in \hat \rulebody$, then
  $f^\trgr_{\lambda}(\tokstart(a)) = f^\trgr_{\lambda}(\tokend(a_0)) = n =
  f(\tokend(a_0)) = f(\tokstart(a))$, which is in contradiction with
  $f^\trgr_{\lambda}(\tokstart(a)) < f(\tokstart(a))$.
  Therefore, it must be $\tokstart(a) \equiv \tokstart(a_0) \notin \hat
  \rulebody$ and $\tokstart(a) \equiv \tokend(a_0) \notin \hat \rulebody$, which
  means that \autoref{item:left-ambiguous-i} of \autoref{def:eager:rule} is
  satisfied.

  We turn now to \autoref{item:right-ambiguous} of \autoref{def:eager:rule}.
  Let $\lambda=\seq{\sigma_0,\ldots,\sigma_{|\lambda|-1}}$ and suppose $a[x=v]$
  occurs in \Rule.
  Since $\lambda$ encodes a plan and $a$ is not a matching token name in
  $f^\trgr_{\lambda}$ and $\lambda$, there is a natural number $i$ such that
  $f^\trgr_{\lambda}(\tokstart(a)) < i < f^\trgr_{\lambda}(\tokend(a))$ and
  $\mathit{end}(x,v) \in \eventsOneBasic{\sigma_i}$.
  Let $\hat i$ be the smallest among such numbers, that is, $\hat i = \min \{ i
  \in \mathbb N \mid f^\trgr_{\lambda}(\tokstart(a)) < i <
  f^\trgr_{\lambda}(\tokend(a))$ and $\mathit{end}(x,v) \in
  \eventsOneBasic{\sigma_i}\}$.
  There are two possible reason for $f^\trgr_{\lambda}(\tokend(a)) > \hat i$:
  \begin{itemize}
  \item $\hat i < \lowtrlambda{\tokend(a)}$ or
  \item $\termToEvent{[\tokend(a)]_\equiv} \not\subseteq
    \eventsOneBasic{\sigma_{\hat i}}$.
  \end{itemize}
  In the former case, by the definition of \lowFuntrlambda, we have that there
  exists a term $t$ such that $t \before \tokend(a) \in \hat \rulebody$,
  $\tokend(a) \before t \notin \hat \rulebody$, and $f^\trgr_{\lambda}(t) \geq
  \hat i$.
  By $\tokend(a) \before t \notin \hat \rulebody$, it holds that $\tokend(a)
  \equiv t \notin \hat \rulebody$, and thus $t \neq \tokend(a)$.
  It also holds that $t \before \tokstart(a) \notin \hat \rulebody$, or
  $f^\trgr_{\lambda}$ would not satisfy \rulebody, because of
  $f^\trgr_{\lambda}(t) \geq \hat i > f^\trgr_{\lambda}(\tokstart(a))$.
  Therefore, $t \notin \{ \tokstart(a), \tokend(a) \}$, and
  \autoref{item:right-ambiguous} of \autoref{def:eager:rule} is verified.
  In the latter case (i.e., $\termToEvent{[\tokend(a)]_\equiv} \not\subseteq
  \eventsOneBasic{\sigma_{\hat i}}$), since $\mathit{end}(x,v) \in
  \eventsOneBasic{\sigma_{\hat i}}$, there is $t \neq \tokend(a)$ such that $t
  \equiv \tokend(a) \in \hat \rulebody$ and $\termToEvent{t} \notin
  \eventsOneBasic{\sigma_{\hat i}}$.
  Clearly, $t \neq \tokstart(a)$ and $t \before \tokstart(a) \notin \hat
  \rulebody$ (because $\tokstart(a) \before* \tokend(a) \in \hat \rulebody$).
  Therefore, $t \notin \{ \tokstart(a), \tokend(a) \}$, and
  \autoref{item:right-ambiguous} of \autoref{def:eager:rule} is verified.
\end{proof}

\begin{corollary}
  \label{cor:solutionIffCanonicalWitness}
  Let $P = (\SV, S)$ be an eager qualitative timeline-based planning problem and
  $\lambda$ a finite word over $\Sigma_\SV$ that encodes a plan over \SV.
  Then, $\lambda$ encodes a solution plan for $P$ if and only if
  $f^{\trgr}_{\lambda}$ is the canonical eager witness of \trgr in $\lambda$ for
  all $\Rule \in S$ and $\trgr \in \triggerslambdaR$.
\end{corollary}

\subsubsection{Correspondence between automata and witnesses}
\label{app:automatonSolutions:automataIffWitnesses}

In this section, we establish a correspondence between automata and witnesses.
The claim then follows.
Let $P = (\SV, S)$ be an eager qualitative timeline-based planning problem,
$\lambda=\seq{\sigma_0,\ldots,\sigma_{n-1}}$ a finite word over $\Sigma_\SV$
that encodes a plan over \SV, and $\pi = q_0 \trans{\sigma_0} q_1
\trans{\sigma_1} \ldots \trans{\sigma_{n-1}} q_n$ the run of \automAP on
$\lambda$.
We associate to $\pi$ a set \funsetpi of function and we show that \automAP
accepts $\lambda$ if and only if \funsetpi is the set of canonical eager
witnesses of \trgr in $\lambda$ for all $\Rule \in S$ and $\trgr \in \Rule$.

First, for every $i \in \{ 1, \ldots, n \}$, with $q_i \neq \sinkAP$, and every
$\viewpointV = (G,K) \in q_i$, we define the \emph{pre-image} of \viewpointV in
$q_{i}$, denoted \preimage{q_i}{\viewpointV}, as the set $\{ \viewpointVprime
\in q_{i-1} \mid \evolviewVprime{\sigma_{i-1}} = \viewpointV \}$.
Notice that $\preimage{q_i}{\viewpointV} \neq \emptyset$.
Then, \funsetqiV is a set of functions from terms (of the forms $\tokstart(a)$
and $\tokend(a)$) to natural numbers defined inductively as follows.
\begin{itemize}
\item $\funsetqZeroV = \{ \emptyset \}$ for all $\viewpointV \in q_0$ (the
  empty function is associated to \funsetqZeroV).
\item $\funsetqiV = \{ f[K \setminus K' \mapsto i-1] \mid f \in
  \viewpointVprime, \viewpointVprime = (G,K') \in \preimageqiV \}$, where, for a
  function $f$, a set of sets of terms $K$, with $\kappa \cap \domf = \emptyset$
  for all $\kappa \in K$, and a natural number $n$, we denote by $f[K \mapsto
  n]$ the function obtained from $f$ by assigning value $n$ to all terms $t \in
  \kappa$ for some set of terms $\kappa \in K$.
  Formally, $f[K \mapsto n](t) = f(t)$ if $t \in \dom{f}$, $f[K \mapsto n](t) =
  n$ if $t \in \bigcup_{\kappa \in K} \kappa$, $f[K \mapsto n](t)$ is undefined
  otherwise.
\end{itemize}
Notice that, for every $f \in \funsetqiV$, it holds that $\domf =
\bigcup_{[t]_{\equiv} \in K}[t]_{\equiv}$, and that $f(t) < i$, for all $t \in
\domf$.
Consequently, it is not difficult to see that $f$ satisfies all atoms $t_1
\before t_2 \in \rulebody$ (resp., $t_1 \before* t_2 \in \rulebody$), where
$t_1,t_2 \in \domf$ and \rulebody is the unique clause of \rulefunV.
Moreover, if \rulefunV has trigger $a_0[x_0=v_0]$ and \viewpointV is enabled
(i.e., $[\tokstart(a_0)]_{\equiv} \in K$), then $f$ identifies univocally a
trigger in \triggerslambdaR, denoted by \triggerFunlambdaf and defined as
$\triggerFunlambdaf = (f(\tokstart(a_0)), \matchEndlambdafaZero) \in \mathbb N
\times \mathbb N$, where, \matchEndlambdafaZero is defined analogously to
\matchStartlambdafaZero, that is, $\matchEndlambdafaZero = \min_i \{ i \in
\mathbb N \mid f(\tokstart(a_0)) < i \text{ and } \mathit{end}(x_0,v_0) \in
\eventsOneBasic{\sigma_i} \}$.
If \rulefunV is triggerless, then we set $\triggerFunlambdaf = \top$.

\begin{lemma}
  \label{lem:automataAcceptWitnesses}
  Let $P = (\SV, S)$ be an eager qualitative timeline-based planning problem,
  $\lambda=\seq{\sigma_0,\ldots,\sigma_{n-1}}$ a finite word over $\Sigma_\SV$
  that encodes a plan over \SV, and $\pi = q_0 \trans{\sigma_0} q_1
  \trans{\sigma_1} \ldots \trans{\sigma_{n-1}} q_n$ the run of \automAP on
  $\lambda$.
  For all $i \in \{ 0, \ldots, n \}$, with $q_i \neq \sinkAP$, all $\viewpointV
  \in q_i$, with \viewpointV enabled, and all $f \in \funsetqiV$, it holds that
  $f(t) = j$ if and only if $f^\trgr_{\lambda}(t) = j$ for all $t \in T$ and all
  natural numbers $j < i$,
%
%
  where $\trgr=\triggerFunlambdaf$ and $T$ is the
  set of terms (of the form $\tokstart(a)$ and $\tokend(a)$) appearing in
  \rulefunV.

\end{lemma}
\begin{proof}
  We prove, by induction on $t$, the following equivalent statement.
  For all $t \in T$
%
%
  and all $j \in \{ 0, \ldots,
  i \}$, it holds that $f(t) < j$ if and only if $f^\trgr_{\lambda}(t) < j$.
  Assume $f(t') < j$ if and only if $f^\trgr_{\lambda}(t') < j$ holds true for
  all $j \in \{ 0, \ldots, i \}$ and all terms $t'$ such that $t' \before t \in
  \hat \rulebody$ and $t \before* t' \notin \hat \rulebody$ (inductive
  hypothesis).

  \medskip

  \noindent{\em ``only if'' direction.} Assume $f(t) < j$.
  We show that $f^\trgr_{\lambda}(t) < j$ holds.
  From $f(t) < j$, there are a natural number $h < j$, viewpoints $\viewpointVh
  = (G, K_{h}) \in q_h$ and $ \viewpointVhPlusOne = (G, K_{h+1}) \in q_{h+1}$,
  and functions $f_h \in \funsetqhVh$ and $f_{h+1} \in
  \funsetqhPlusOneVhPlusOne$ such that $\evolviewBasic{\viewpointVh}{\sigma_{h}}
  = \viewpointVhPlusOne$, $f_{h+1} = f_{h}[K_{h+1} \setminus K_h \mapsto h]$,
  and $t \in \bigcup_{[s]_{\equiv} \in
    K_{h+1} \setminus K_h} [s]_{\equiv}$ (i.e., $t \in [s]_{\equiv}$ for some
  $[s]_{\equiv} \in K_{h+1} \setminus K_h$).
  To prove that $f^\trgr_{\lambda}(t) < j$ holds, we show the following:
  \begin{enumerate}[label={\em(\roman*)}]
  \item \label{it:onlyIf:leq}
    $h \geq f^\trgr_{\lambda}(t')$ for all $t'$ with $t' \before t \in \hat
    \rulebody$ and $t \before t' \notin \hat \rulebody$.
    By the downward closure of $K_{h+1}$, we have that $t' \in
    \bigcup_{[s]_\equiv \in K_{h+1}} [s]_{\equiv}$.
    Thus, $f(t') \leq h$.
    By inductive hypothesis, $f(t') \leq h$ implies $f^\trgr_\lambda(t') \leq h$.
  \item \label{it:onlyIf:lt}
    $h > f^\trgr_{\lambda}(t')$ for all $t'$ with $t' \before* t \in \hat
    \rulebody$.
    By the definition of \nextviewpoint{\viewpointVh} in
    \autoref{sec:viewpoints}, we have $t' \in \bigcup_{[s]_\equiv \in
      K_{h}} [s]_{\equiv}$.
    Thus, $f(t') < h$.
    By inductive hypothesis, $f(t') < h$ implies $f^\trgr_\lambda(t') < h$.
  \item \label{it:onlyIf:events}
    $\termToEvent{[t]_\equiv} \subseteq \eventsOnesigmah$.
    It immediately follows from $\termToEvent{K_{h+1} \setminus K_h} \subseteq
    \eventsOnesigmah$ (cf. \autoref{sec:viewpoints}) and $[t]_{\equiv} \in
    K_{h+1} \setminus K_h$
  \end{enumerate}
  From \ref{it:onlyIf:leq}, \ref{it:onlyIf:lt}, and \ref{it:onlyIf:events}
  above, it follows that $f^\trgr(t) \leq h < j$.

  \medskip

  \noindent{\em ``if'' direction.} Assume $f^\trgr_\lambda(t) = h < j$.
  We show that $f(t) < j$ holds.
  The assumption $f^\trgr_\lambda(t) = h$ implies that
    \begin{enumerate*}[label={\em(\roman*)}]
  \item \label{it:if:leq}
    $h \geq f^\trgr_{\lambda}(t')$ for all $t'$ with $t' \before t \in \hat
    \rulebody$ and $t \before t' \notin \hat \rulebody$,
  \item \label{it:if:lt}
    $h > f^\trgr_{\lambda}(t')$ for all $t'$ with $t' \before* t \in \hat
    \rulebody$, and
  \item \label{it:if:events}
    $\termToEvent{[t]_\equiv} \subseteq \eventsOnesigmah$.
  \end{enumerate*}
  From \ref{it:onlyIf:leq} and \ref{it:onlyIf:lt} above, it follows that there
  are viewpoints $\viewpointVh = (G, K_{h}) \in q_h$ and $ \viewpointVhPlusOne =
  (G, K_{h+1}) \in q_{h+1}$
%
%
  such that
%
%
%
  $t' \in \bigcup_{[s]_{\equiv} \in K_{h+1} } [s]_{\equiv}$ for all $t'$ with
  $t' \before t \in \hat \rulebody$ and $t \before t' \notin \hat \rulebody$,
  (due to \ref{it:if:leq}), $t' \in \bigcup_{[s]_{\equiv} \in K_{h} }
  [s]_{\equiv}$ for all $t'$ with $t' \before* t \in \hat \rulebody$, and (due
  to \ref{it:if:lt}).
  Therefore, $[t]_{\equiv} \in \nextviewpoint{\viewpointVh}$.
  From \ref{it:if:events} above, it follows that $[t]_{\equiv} \in
  \consumed{\viewpointVh}{\sigma_{h}} = K_{h+1}$.
  Hence, $f(t) = h < j$.
\end{proof}

The following corollary immediately follows (recall that the last state of an
accepting run is not the rejecting sink state \sinkAP and it is such that all
enabled viewpoints in it are final).

\begin{corollary} \label{cor:automataAcceptWitnesses}
  Let $P = (\SV, S)$ be an eager qualitative timeline-based planning problem,
  $\lambda=\seq{\sigma_0,\ldots,\sigma_{n-1}}$ a finite word over $\Sigma_\SV$
  that encodes a plan over \SV, and $\pi = q_0 \trans{\sigma_0} q_1
  \trans{\sigma_1} \ldots \trans{\sigma_{n-1}} q_n$ the run of \automAP on
  $\lambda$.
  If $\pi$ is accepting, then for all $\viewpointV \in q_n$, with \viewpointV
  enabled, and all $f \in \funsetqnV$, we have that $f=f^\trgr_\lambda$, where
  $\trgr=\triggerFunlambdaf$.
\end{corollary}

The next lemma establishes the correctness of out automata construction.

\begin{lemma}
  \label{lem:automataCorrectness}
  Let $P = (\SV, S)$ be an eager qualitative timeline-based planning problem,
  $\lambda=\seq{\sigma_0,\ldots,\sigma_{n-1}}$ a finite word over $\Sigma_\SV$
  that encodes a plan over \SV, and $\pi = q_0 \trans{\sigma_0} q_1
  \trans{\sigma_1} \ldots \trans{\sigma_{n-1}} q_n$ the run of \automAP on
  $\lambda$.

  \begin{enumerate}
  \item If $\pi$ is accepting, then $\bigcup_{\viewpointV \in q_n, \viewpointV
      \text{enabled}} \funsetqnV = \{ f^\trgr_\lambda \mid \Rule \in S, \trgr
    \in \triggerslambdaR \}$.
  \item If $\pi$ is accepting, then for every $f \in \bigcup_{\viewpointV \in
      q_n, \viewpointV \text{enabled}} \funsetqnV$, it holds that $f$ is the
    canonical eager witness of \triggerFunlambdaf in $\lambda$.
  \item If $\pi$ is rejecting, then there are $\Rule \in S$ and $\trgr \in
    \triggerslambdaR$ such that $f^{\trgr}_{\lambda}$ is not the canonical eager
    witness of \trgr in $\lambda$.
  \end{enumerate}
\end{lemma}
\begin{proof}
  \begin{enumerate}
  \item The left-to-right inclusion immediately follows from
    \autoref{cor:automataAcceptWitnesses}.
    In order to proof the converse inclusion, we proceed as follows.
    Let $\Rule \in S$ and $\trgr \in \triggerslambdaR$.
    We show that there are $\viewpointVn \in q_n$, with \viewpointVn enabled,
    and $f \in \funsetqnVn$ such that $f = f^\trgr_\lambda$.
    If \Rule is triggerless, then the claim follows immediately from
    \autoref{cor:automataAcceptWitnesses} (recall that viewpoints for
    triggerless rules are always enabled).
    If, instead, \Rule is not triggerless, with trigger $a_0[x_0=v_0]$, then
    $\trgr = (h,k)$, for some $h,k \in \mathbb N \times \mathbb N$, with
    $h<n-1$, and $\mathit{start}(x_0,v_0) \in \eventsOnesigmah$ (i.e.,
    $\sigma_h$ triggers \Rule).
    Note that, from $\pi$ being accepting, it follows that $q_i \neq \sinkAP$
    for all $i \in \{ 0, 1, \ldots, n \}$ and that $q_i$ is compatible with
    $\sigma_i$ for all $i \in \{ 0, 1, \ldots, n-1 \}$.
    Then, since $q_h$ is compatible with $\sigma_h$ and $\sigma_h$ triggers
    \Rule, there is $\viewpointVh = (G,K_h) \in q_h$ for \Rule such that
    $[\tokstart(a_0)]_\equiv \in \consumedVhsigmah \setminus K_{h}$.
    Therefore, there is $\viewpointVhPlusOne = (G,K_{h+1}) \in q_{h+1}$ such
    that $[\tokstart(a_0)]_\equiv \in K_{h+1} \setminus K_h$, which means that
    there are $\viewpointVn \in q_n$ and $f \in \funsetqnVn$ such that
    \viewpointVn is enabled and $f(\tokstart(a_0)) = h$.
    By \autoref{cor:automataAcceptWitnesses}, $f=f^{\trgr'}_\lambda$, where
    $\trgr' = \trgr$ (because $\trgr' = \triggerFunlambdaf = (f(\tokstart(a_0)),
    \matchEndlambdafaZero) = (h,k) = \trgr$).

  \item Let $\viewpointV \in q_n$, with \viewpointV enabled, and $f \in
    \funsetqnV$.
    By the previous item, $f = f^\trgr_\lambda$, with $\trgr =
    \triggerFunlambdaf$.
    Moreover, by definition, $f$ is a witness of \triggerFunlambdaf in
    $\lambda$.
    We show that it is a relaxed one, i.e., for all non-matching token names $a$
    appearing in \domf and all terms $t \not\in \{ \tokstart(a), \tokend(a) \}$
    it holds that $\tokstart(a) \before t \in \hat \rulebody$ implies
    $\tokend(a) \before t \in \hat \rulebody$, where \rulebody is the unique
    clause of \rulefunV (cf. \autoref{def:matchingAndRelaxedWitness}).
    Let $a$ be a non-matching token name in $f$ and $\lambda$, with $a[x=v]$
    appearing in \rulefunV, and let $m = \matchEndlambdafa < f(\tokend(a))$.
    Then, $\tokstart(a), \tokend(a) \in \domf$ and $\mathit{end}(x,v) \in
    \eventsOneBasic{\sigma_m}$.
%
%
    Since $\pi$ is accepting and given that $f(\tokend(a)) > m$, there are a
    viewpoint $\viewpointV \in q_m$ and a function $f' \in \funsetqmV$ such that
    $f'(\tokstart(a)) = f(\tokstart(a))$ and $\tokend(a) \notin \waitingV$ (or
    it would be $f(\tokend(a))=m$, leading to a contradiction).
    Therefore, by the definition of \waitingV, it must be that $\tokstart(a)
    \before t \in \hat \rulebody$ implies $\tokend(a) \before t \in \hat
    \rulebody$ for all terms $t \not\in \{ \tokstart(a), \tokend(a) \}$, which
    concludes the proof.

  \item Since $\pi$ is rejecting, there are two possibilities:

    \begin{enumerate}[label={\it \Alph*.},ref={\it \Alph*}]
    \item \label{item:not-sink}
      $q_n \neq \sinkAP$ and there is $\viewpointV \in q_n$ that is enabled but
      not final;
    \item \label{item:sink}
      there is $i \in \{ 0, 1, n-1 \}$ such that $q_i \neq \sinkAP$ and $q_{i+1}
      = \sinkAP$.
    \end{enumerate}

    In the former case, there is $f \in \funsetqnV$ that is not defined for some
    term $t$ appearing in \rulefunV.
    By \autoref{lem:automataAcceptWitnesses}, we have that $f =
    f^\trgr_{\lambda}$, where $\trgr=\triggerFunlambdaf$.
    Therefore, $f^\trgr_{\lambda}$ is a partial function, and thus it is not the
    canonical eager witness of \trgr in $\lambda$ (as a matter of fact
    $f^\trgr_{\lambda}$ is not even a witness \trgr in $\lambda$).

    In the latter case, we further distinguish two cases:
    \begin{enumerate}[label={\it \ref{item:sink}.\roman*}, align=left,
      labelindent=1mm]
    \item \label{item:not-compatible}
      there is a viewpoint $\viewpointV = (G,K) \in q_i$ that is not compatible
      with $\sigma_i$;
    \item \label{item:no-enable}
      there is a rule $\Rule \in S$ such that \Rule is not triggerless,
      $\sigma_i$ triggers \Rule, and there is no viewpoint $\viewpointV \in q_i$
      for \Rule such that $\sigma_i$ enables \viewpointV.
    \end{enumerate}
    We deal with these two cases separately.
    \begin{itemize}[align=left, labelindent=1mm,
      widest=\ref{item:no-enable}\hspace{3mm}]

    \item[\ref{item:not-compatible}]
      {\setlength{\emergencystretch}{2em}
      \viewpointV being not compatible with
      $\sigma_i$ implies that $\termToEvent{\waitingV} \cap \eventsOnesigmai
      \subseteq \termToEvent{\consumedVsigmai \setminus K}$.
      Thus, there is a token name $a$ such that $a[x=v]$ appears in \rulefunV,
      $\tokend(a) \in \waitingV$, $\mathit{end}(x,v) \in \eventsOnesigmai$, and
      $[\tokend(a)]_\equiv \notin \consumedVsigmai \setminus K$.
      \par}

      Let $f \in \funsetqiV$.
      By $\tokend(a) \in \waitingV$, we know that $[\tokstart(a)]_\equiv \in K$,
      meaning that \viewpointV is enabled.
      Therefore, $f$ identifies univocally the trigger \triggerFunlambdaf in
      \triggerslambdaR.
      Moreover, by \autoref{lem:automataAcceptWitnesses}, we have $f(t) = j$ if
      and only if $f^\trgr_\lambda(t) = j$, for all terms $t$ appearing in
      \rulefunV and all $j \in \mathbb N$, with $j < i$.
      We show that $f^\trgr_\lambda$ is not the canonical eager witness of \trgr
      in $\lambda$.

      {\setlength{\emergencystretch}{2em}
      From $\tokend(a) \in \waitingV$, it follows $[\tokend(a)]_\equiv \notin
      K$; thus, since $[\tokend(a)]_\equiv \notin \consumedVsigmai \setminus K$,
      it holds that $[\tokend(a)]_\equiv \notin \consumedVsigmai$.
      By the definition of \consumedVsigmai, there are two possibilities, both
      leading to $f^\trgr_\lambda(\tokend(a)) > i$.
      \par}
      \begin{itemize}
      \item If $[\tokend(a)]_\equiv \notin \nextV$, then there is $t$ such that
        $t \before* \tokend(a) \in \hat \rulebody$ and $[t]_\equiv \not \in K$.
        In this case, we have $f(t) \geq i$, that implies $f^\trgr_\lambda(t)
        \geq i$.
        If $f^\trgr_\lambda$ was a witness of \trgr in $\lambda$, then we would
        have $f^\trgr_\lambda(\tokend(a)) > f^\trgr_\lambda(t) \geq i$.

      \item Otherwise, there must be $t$ such that $t \before \tokend(a) \in
        \hat \rulebody$, $[t]_\equiv \not \in \consumedVsigmai$, and
        $\termToEventt \notin \eventsOnesigmai$.
        Therefore, also in this case, we have $f^\trgr_\lambda(\tokend(a)) > i$.
      \end{itemize}

      From $f^\trgr_\lambda(\tokend(a)) > i$, it follows that $a$ is not
      matching in $f^\trgr_\lambda$ and $\lambda$.
      Therefore, $a$ is not the
      trigger token of \rulefunV.
      By $\tokend(a) \in \waitingV$, it must then be the case that there is a
      term $t \not\in \{ \tokstart(a), \tokend(a) \}$ such that $\tokstart(a)
      \before t \in \hat \rulebody$ and $\tokend(a) \before t \not\in \hat
      \rulebody$.
      This, together with $a$ being not matching, implies that $f^\trgr_\lambda$
      is not a relaxed witness, and thus it is not the canonical eager witness
      of \trgr in $\lambda$.


    \item[\ref{item:no-enable}]
      Assume there is a rule $\Rule \in S$ such that \Rule is not triggerless,
      with trigger $a_0[x_0=v_0]$, $\sigma_i$ triggers \Rule, and there is no
      viewpoint $\viewpointV \in q_i$ for \Rule such that $\sigma_i$ enables
      \viewpointV.
      Then, from the fact that $\sigma_i$ triggers \Rule, it follows
      $\termToEvent{\tokstart(a_0)} \in \eventsOnesigmai$, and thus $(i,j) \in
      \triggerslambdaR$, for some $j \in \mathbb N$.
      Assume, towards a contradiction, that $f^{(i,j)}_\lambda$ is the canonical
      eager witness of $(i,j)$ in $\lambda$.
      Therefore, $f^{(i,j)}_\lambda(\tokstart(a)) = i$.
      However, let $\viewpointV \in q_i$ be a viewpoint for \Rule and let $f \in
      \funsetqiV$.
      Since $\sigma_i$ does not enable \viewpointV, we distinguish two
      scenarios.
      \begin{itemize}
      \item If there is $t$ such that $t \before* \tokstart(a) \in \hat
        \rulebody$ and $[t]_\equiv \not \in K$, then $f(t) \geq i$, that implies
        $f^{(i,j)}_\lambda(t) \geq i$.
        If $f^{(i,j)}_\lambda$ was a witness of $(i,j)$ in $\lambda$, then we
        would have $f^{(i,j)}_\lambda(\tokstart(a)) > f^{(i,j)}_\lambda(t) \geq
        i$, which is in contradiction with $f^{(i,j)}_\lambda(\tokstart(a)) =
        i$.

      \item Otherwise, there must be $t$ such that $t \before \tokstart(a) \in
        \hat \rulebody$, $[t]_\equiv \not \in \consumedVsigmai$, and
        $\termToEventt \notin \eventsOnesigmai$.
        Therefore, also in this case, we have $f^\trgr_\lambda(\tokstart(a)) >
        i$, which is in contradiction with $f^{(i,j)}_\lambda(\tokstart(a)) =
        i$. \qedhere
      \end{itemize}
    \end{itemize}
  \end{enumerate}
\end{proof}

\begin{corollary}\label{cor:automatonCorrectness}
  Let $P = (\SV, S)$ be an eager qualitative timeline-based planning problem and
  $\lambda$ a finite word over $\Sigma_\SV$ that encodes a plan over \SV.
  \automAP accepts $\lambda$ if and only if $f^{\trgr}_{\lambda}$ is the
  canonical eager witness of \trgr in $\lambda$ for all $\Rule \in S$ and $\trgr
  \in \triggerslambdaR$.
\end{corollary}

\autoref{lem:automatonSolutions:onlyif} now follows from
\autoref{cor:solutionIffCanonicalWitness} and
\autoref{cor:automatonCorrectness}, thus completing the proof of
\autoref{lem:automatonSolutions}.


\bibliographystyle{alphaurl}
\bibliography{biblio}

@book{hopcroft2006introduction,
  author    = {John E. Hopcroft and Rajeev Motwani and Jeffrey D. Ullman},
  title     = {Introduction to Automata Theory, Languages, and Computation},
  edition   = {3rd},
  publisher = {Addison-Wesley},
  year      = {2006},
  address   = {Boston},
  isbn      = {9780321455369}
}

@article{PolyvyanyyGBD12,
  author    = {Artem Polyvyanyy and Luciano Garc{\'\i}a-Ba{\~n}uelos and Marlon Dumas},
  title     = {Structuring acyclic process models},
  journal   = {Information Systems},
  volume    = {37},
  number    = {6},
  pages     = {518--538},
  year      = {2012},
  doi       = {10.1016/j.is.2011.10.005}
}

@inproceedings{KoehlerH04,
  author    = {Jana Koehler and Rainer Hauser},
  title     = {Untangling Unstructured Cyclic Flows -- {A} Solution Based on Continuations},
  booktitle = {On the Move to Meaningful Internet Systems 2004: CoopIS, DOA, and ODBASE},
  series    = {Lecture Notes in Computer Science},
  volume    = {3290},
  pages     = {121--138},
  publisher = {Springer},
  year      = {2004},
  doi       = {10.1007/978-3-540-30468-5\_10}
}

@article{ChoiKJZ15,
  author    = {Yongsun Choi and Petcharat Kongsuwan and Cheol Min Joo and J. Leon Zhao},
  title     = {Stepwise structural verification of cyclic workflow models with acyclic decomposition and reduction of loops},
  journal   = {Data \& Knowledge Engineering},
  volume    = {95},
  pages     = {39--65},
  year      = {2015},
  doi       = {10.1016/j.datak.2014.11.003}
}

@inproceedings{PrinzCH22,
  author    = {Thomas M. Prinz and Yongsun Choi and N. Long Ha},
  title     = {Understanding and Decomposing Control-Flow Loops in Business Process Models},
  booktitle = {Business Process Management},
  series    = {Lecture Notes in Computer Science},
  volume    = {13420},
  pages     = {291--308},
  publisher = {Springer},
  year      = {2022},
  doi       = {10.1007/978-3-031-16103-2\_21}
}

@inproceedings{DBLP:conf/time/CombiOS19,
  author       = {Carlo Combi and
                  Barbara Oliboni and
                  Pietro Sala},
  editor       = {Johann Gamper and
                  Sophie Pinchinat and
                  Guido Sciavicco},
  title        = {Customizing {BPMN} Diagrams Using Timelines},
  booktitle    = {26th International Symposium on Temporal Representation and Reasoning,
                  {TIME} 2019, October 16-19, 2019, M{\'{a}}laga, Spain},
  series       = {LIPIcs},
  volume       = {147},
  pages        = {5:1--5:17},
  publisher    = {Schloss Dagstuhl - Leibniz-Zentrum f{\"{u}}r Informatik},
  year         = {2019},
  url          = {https://doi.org/10.4230/LIPIcs.TIME.2019.5},
  doi          = {10.4230/LIPICS.TIME.2019.5},
  bibsource    = {dblp computer science bibliography, https://dblp.org}
}

@article{GiganteMOCR20,
   author = {Nicola Gigante and Angelo Montanari and Andrea Orlandini and Marta Cialdea Mayer and Mark Reynolds},
   doi = {10.1016/j.tcs.2020.02.011},
   issn = {03043975},
   journal = {Theoretical Computer Science},
   month = {5},
   pages = {247-269},
   publisher = {Elsevier B.V.},
   title = {On timeline-based games and their complexity},
   volume = {815},
   year = {2020}
}

@inproceedings{AcamporaGGMP2022,
  author       = {Renato Acampora and
                  Luca Geatti and
                  Nicola Gigante and
                  Angelo Montanari and
                  Valentino Picotti},
  editor       = {Pierre Ganty and
                  Della Monica, Dario},
  title        = {Controller Synthesis for Timeline-based Games},
  booktitle    = {Proceedings of the 13th International Symposium on Games, Automata,
                  Logics and Formal Verification, GandALF 2022, Madrid, Spain, September
                  21-23, 2022},
  series       = {{EPTCS}},
  volume       = {370},
  pages        = {131--146},
  year         = {2022},
  doi          = {10.4204/EPTCS.370.9},
  bibsource    = {dblp computer science bibliography, https://dblp.org}
}

@inproceedings{DellaMonicaGTM20,
  author       = {Della Monica, Dario and
                  Nicola Gigante and
                  La Torre, Salvatore and
                  Angelo Montanari},
  editor       = {Emilio Mu{\~{n}}oz{-}Velasco and
                  Ana Ozaki and
                  Martin Theobald},
  title        = {Complexity of Qualitative Timeline-Based Planning},
  booktitle    = {27th International Symposium on Temporal Representation and Reasoning,
                  {TIME} 2020, September 23-25, 2020, Bozen-Bolzano, Italy},
  series       = {LIPIcs},
  volume       = {178},
  pages        = {16:1--16:13},
  publisher    = {Schloss Dagstuhl - Leibniz-Zentrum f{\"{u}}r Informatik},
  year         = {2020},
  doi          = {10.4230/LIPICS.TIME.2020.16},
  bibsource    = {dblp computer science bibliography, https://dblp.org}
}

@inproceedings{GiganteMCO17,
  author    = {Nicola Gigante and
               Angelo Montanari and
               Marta {Cialdea Mayer} and
               Andrea Orlandini},
  editor    = {Laura Barbulescu and
               Jeremy Frank and
               Mausam and
               Stephen F. Smith},
  title     = {Complexity of Timeline-Based Planning},
  booktitle = {Proceedings of the 27th International Conference on Automated
               Planning and Scheduling},
  pages     = {116--124},
  publisher = {{AAAI} Press},
  doi       = {10.1609/icaps.v27i1.13830},
  year      = {2017}
}

@inproceedings{DellaMonicaGMS18,
  author    = {Della Monica, Dario and Nicola Gigante and
               Angelo Montanari and Pietro Sala},
  editor    = {Michael Thielscher and Francesca Toni and Frank Wolter},
  title     = {A Novel Automata-Theoretic Approach to Timeline-Based Planning},
  booktitle = {Proceedings
               of the 16th International Conference on Principles of Knowledge Representation and Reasoning},
  pages     = {541--550},
  publisher = {{AAAI} Press},
  year      = {2018}
}

@article{FoxL03,
  author       = {Maria Fox and
                  Derek Long},
  title        = {{PDDL2.1:} An Extension to {PDDL} for Expressing Temporal Planning
                  Domains},
  journal      = {J. Artif. Intell. Res.},
  volume       = {20},
  pages        = {61--124},
  year         = {2003},
  doi          = {10.1613/jair.1129}
}

@incollection{Muscettola94,
  author      = {Nicola Muscettola},
  title       = {{{HSTS}}: {I}ntegrating {P}lanning and {S}cheduling},
  editor      = {Monte Zweben and Mark S. Fox},
  booktitle   = {Intelligent Scheduling},
  publisher   = {Morgan Kaufmann},
  year        = {1994},
  pages       = {169--212},
  chapter     = {6}
}

@inproceedings{ChienSTCRCDLMFTHDSUBBGGDBDI04,
  author    = {Steve A. Chien and Rob Sherwood and Daniel Tran and
               Benjamin Cichy and Gregg Rabideau and Rebecca Casta{\~{n}}o and
               Ashley Davies and Rachel Lee and Dan Mandl and Stuart Frye and
               Bruce Trout and Jerry Hengemihle and Jeff D'Agostino and
               Seth Shulman and Stephen G. Ungar and Thomas Brakke and
               Darrell Boyer and Jim Van Gaasbeck and Ronald Greeley and
               Thomas Doggett and Victor R. Baker and James M. Dohm and
               Felipe Ip},
  title     = {The {EO-1} Autonomous Science Agent},
  booktitle = {3rd International Joint Conference on Autonomous Agents and
                Multiagent Systems},
  pages     = {420--427},
  publisher = {{IEEE} Computer Society},
  year      = {2004},
  doi       = {10.1109/AAMAS.2004.10022}
}

@misc{ChienRKSEMESFBST00,
  author = {Steve A. Chien and Gregg Rabideau and Russell L. Knight and Rob Sherwood and Barbara E.
  Engelhardt and D. Mutz and Tara Estlin and B. Smith and Forest Fisher and T.
  Barrett and G. Stebbins and Daniel Tran},
  title = {ASPEN - Automating Space Mission Operations using Automated Planning
           and Scheduling},
  booktitle = {Proceedings of the International Conference on Space Operations},
  year = {2000}
}

@misc{FratiniCORD11,
  author = {Fratini, Simone and Cesta, Amedeo and Orlandini, Andrea and Rasconi, Riccardo and De Benedictis, Riccardo},
  booktitle = {ASTRA 2011},
  publisher = {ESA},
  title = {APSI-based Deliberation in Goal Oriented Autonomous Controllers},
  volume = {11},
  year = {2011}
}

@article{CialdeaMayerOU16,
  author    = {Marta {Cialdea Mayer} and Andrea Orlandini and
               Alessandro Umbrico},
  title     = {Planning and execution with flexible timelines: a formal
               account},
  journal   = {Acta Informatica},
  volume    = {53},
  number    = {6-8},
  pages     = {649--680},
  year      = {2016},
  doi       = {10.1007/s00236-015-0252-z}
}

@inproceedings{UmbricoCMO17,
  author    = {Alessandro Umbrico and Amedeo Cesta and Marta {Cialdea Mayer} and
               Andrea Orlandini},
  editor    = {Floriana Esposito and Roberto Basili and Stefano Ferilli and
               Francesca A. Lisi},
  title     = {PLATINUm: {A} New Framework for Planning and Acting},
  booktitle = {Proceedings of the 16th International
               Conference of the Italian Association for Artificial Intelligence},
  series    = {LNCS},
  volume    = {10640},
  pages     = {498--512},
  publisher = {Springer},
  year      = {2017},
  doi       = {10.1007/978-3-319-70169-1\_37}
}

@inproceedings{BozzelliMMP18a,
  author    = {Laura Bozzelli and Alberto Molinari and
               Angelo Montanari and Adriano Peron},
  editor    = {Michael Thielscher and
               Francesca Toni and
               Frank Wolter},
  title     = {Decidability and Complexity of Timeline-Based Planning over Dense
               Temporal Domains},
  booktitle = {Proceedings of the 16th International Conference on Principles
               of Knowledge Representation and Reasoning},
  pages     = {627--628},
  publisher = {{AAAI} Press},
  year      = {2018},
  url       = {https://aaai.org/ocs/index.php/KR/KR18/paper/view/17995}
}

@inproceedings{BozzelliMMP18b,
  author    = {Laura Bozzelli and Alberto Molinari and
               Angelo Montanari and Adriano Peron},
  editor    = {Andrea Orlandini and Martin Zimmermann},
  title     = {Complexity of Timeline-Based Planning over Dense Temporal
               Domains: Exploring the Middle Ground},
  booktitle = {Proceedings of the 9th International Symposium on Games,
               Automata, Logics, and Formal Verification},
  series    = {{EPTCS}},
  volume    = {277},
  pages     = {191--205},
  year      = {2018},
  doi       = {10.4204/EPTCS.277.14}
}

@inproceedings{UmbricoCO23,
  author       = {Alessandro Umbrico and
                  Amedeo Cesta and
                  Andrea Orlandini},
  title        = {Human-Aware Goal-Oriented Autonomy through ROS-Integrated Timeline-based
                  Planning and Execution},
  booktitle    = {32nd {IEEE} International Conference on Robot and Human Interactive
                  Communication},
  pages        = {1164--1169},
  publisher    = {{IEEE}},
  year         = {2023},
  doi          = {10.1109/RO-MAN57019.2023.10309516}
}

@article{Allen83,
  author       = {James F. Allen},
  title        = {Maintaining Knowledge about Temporal Intervals},
  journal      = {Commun. {ACM}},
  volume       = {26},
  number       = {11},
  pages        = {832--843},
  year         = {1983},
  doi          = {10.1145/182.358434}
}

@inproceedings{PnueliR89,
  author       = {Amir Pnueli and
                  Roni Rosner},
  editor       = {Giorgio Ausiello and
                  Mariangiola Dezani{-}Ciancaglini and
                  Simona Ronchi Della Rocca},
  title        = {On the Synthesis of an Asynchronous Reactive Module},
  booktitle    = {16th International Colloquium on Automata, Languages and Programming},
  series       = {Lecture Notes in Computer Science},
  volume       = {372},
  pages        = {652--671},
  publisher    = {Springer},
  year         = {1989},
  doi          = {10.1007/BFB0035790}
}

@inproceedings{GiganteMMO16,
  author       = {Nicola Gigante and
                  Angelo Montanari and
                  Marta Cialdea Mayer and
                  Andrea Orlandini},
  editor       = {Curtis E. Dyreson and
                  Michael R. Hansen and
                  Luke Hunsberger},
  title        = {Timelines Are Expressive Enough to Capture Action-Based Temporal Planning},
  booktitle    = {23rd International Symposium on Temporal Representation and Reasoning,},
  pages        = {100--109},
  publisher    = {{IEEE} Computer Society},
  year         = {2016},
  doi          = {10.1109/TIME.2016.18}
}

@inproceedings{della2017bounded,
  title={Bounded timed propositional temporal logic with past captures timeline-based planning with bounded constraints},
  author={Della Monica, Dario and Gigante, Nicola and Montanari, Angelo and Sala, Pietro and Sciavicco, Guido and others},
  booktitle={IJCAI},
  pages={1008--1014},
  year={2017},
  doi={10.24963/IJCAI.2017/140},
  organization={International Joint Conferences on Artificial Intelligence}
}

@inproceedings{gandalf24,
  author    = {Renato Acampora and Della Monica, Dario
               and Luca Geatti and Nicola Gigante
	       and Angelo Montanari},
  title     = {Synthesis of Timeline-Based Planning Strategies
               Avoiding Determinization},
  booktitle = {Proc. of the 15th International Symposium on Games,
               Automata, Logics and Formal Verification (GandALF)},
  year      = {2024}
}

\end{document}